\documentclass[journal,compsoc,onecolumn]{IEEEtran}

\ifCLASSINFOpdf
\else
\fi
\let\labelindent\relax
\usepackage{array}
\usepackage{color}
\usepackage[normalem]{ulem}
\usepackage{amsmath, amssymb, mathtools, amsthm}
\usepackage{cancel}

\usepackage[colorlinks=true, urlcolor=black,citecolor=black,linkcolor=black]{hyperref} 
\newtheorem{theorem}{\textbf{Theorem}}%[section]
\newtheorem{lemma}[theorem]{Lemma}

\newtheorem{remark}{Remark}%[section]

\usepackage{cleveref}
\usepackage{booktabs}
\usepackage{upgreek}
\usepackage{utfsym}
\usepackage{bbding}
\usepackage{amssymb}

\newtheorem*{lemma*}{Theorem}
\newtheorem*{proposition*}{Proposition}

\usepackage{mathrsfs}
 \usepackage{algorithm}
 \usepackage{algpseudocode}
\usepackage{fontawesome}
\usepackage{orcidlink}
\usepackage[square,numbers,sort&compress]{natbib}
\usepackage{multirow}
\usepackage{subfigure} %插入多图时用子图显示的宏包

\usepackage{nomencl}
\makenomenclature

\definecolor{darkgreen}{rgb}{0.0, 0.5, 0.0}

\usepackage{xr-hyper}

\crefname{equation}{}{}
\Crefname{equation}{}{}
\Crefname{algorithm}{Algorithm}{Algorithms}
\Crefname{figure}{Fig.}{Figs.}
\crefname{proposition}{proposition}{propositions}
\Crefname{proposition}{Proposition}{Propositions}
\DeclareMathAlphabet{\mathcal}{OMS}{cmsy}{m}{n}
\usepackage{enumitem}

\usepackage{threeparttable}

\begin{document}
\title{
% \huge{Enhancing Supervised Nonlinear Probabilistic Latent Variable Models for Inferential Sensors: An Optimization Perspective}
% Finite Mixture Model Under Industrial Soft Sensor Modeling: A Functional Gradient Descent Approach
% Finite Mixture Model for Industrial Data-Driven Modeling in the Era of Deep Learning: 
% A Functional Gradient Approach
% Mixing Last Layer Deep Neural Network for Soft Sensors: A Proximal Gradient Approach
% \huge{
Slack More, Predict Better: Proximal Relaxation for Probabilistic Latent Variable Model-based \\Soft Sensors
%Reversing the Conditional Arrow: Proximal Gradient Descent for Multi-modal Soft Sensor Modeling}
% }
}
\author{Zehua Zou~\orcidlink{0009-0004-1933-5035}, Yiran Ma~\orcidlink{0009-0004-6749-5121}, Yulong Zhang~\orcidlink{0000-0002-4038-1616}, Zhengnan Li~\orcidlink{0009-0008-0592-214X}, Zeyu Yang~\orcidlink{0000-0003-0253-9053}~\IEEEmembership{Member, IEEE}, Jinhao Xie~\orcidlink{0000-0001-6870-263X},\\Xiaoyu Jiang~\orcidlink{0000-0003-4170-5579},~\IEEEmembership{Member, IEEE} and Zhichao Chen~\orcidlink{0000-0001-5785-0741}
% \thanks{This paragraph of the first footnote will contain the date on 
% which you submitted your paper for review. It will also contain support 
% information, including sponsor and financial support acknowledgment. For 
% example, ``This work was supported in part by the U.S. Department of 
% Commerce under Grant BS123456.'' }
\thanks{This work was supported by the National Natural Science Foundation of China (NSFC) under Grant Number 62403425, and was also supported by  the China Postdoctoral Science Foundation under Grant Number 2025M781449. Paper no. TII-26-1425. (Corresponding author: Xiaoyu Jiang and Zhichao Chen)}
\thanks{The Copyright of this work has been transferred to the IEEE.}
\thanks{Zehua Zou and Xiaoyu Jiang are with the Hangzhou International Innovation Institute, Beihang University, Hangzhou 311115, China (Email: 19076144g@connect.polyu.hk, jiangxiaoyu@buaa.edu.cn).}
\thanks{Yiran Ma is with the State Key Laboratory of Industrial Control Technology, College of Control science and Engineering, Zhejiang University, Hangzhou 310027, Zhejiang, China (E-mail: mayiran@zju.edu.cn)}
\thanks{Yulong Zhang is with the College of Electrical Engineering and Automation, Fuzhou University, Fuzhou 350108, China (E-mail: zhangyl@fzu.edu.cn).}
\thanks{Zhengnan Li is with the School of Data Science, The Chinese University of Hong Kong, Shenzhen 518172, China (E-mail: 225040119@link.cuhk.edu.cn).}
\thanks{Zeyu Yang is with the Zhejiang Key Laboratory of Industrial Solid Waste Thermal Hydrolysis Technology and Intelligent Equipment, Huzhou Key Laboratory of Intelligent Sensing and Optimal Control for Industrial Systems, School of Engineering, Huzhou University, Huzhou 313000, China (E-mail: yangzeyu@zjhu.edu.cn).
}
\thanks{Jinhao Xie is with MOE of the Key Laboratory of Bioinorganic and Synthetic Chemistry, the Key Lab of Low‐Carbon Chem \& Energy Conservation of Guangdong Province, School of Chemistry, Sun Yat‐Sen University, Guangzhou 510275, China. (E-mail: xiejh25@mail2.sysu.edu.cn).}
\thanks{Zhichao Chen is with the National Key Lab of General AI, School of Intelligence Science and Technology, Peking University, Beijing 100871, China (E-mail: 
12032042@zju.edu.cn).}
}

\markboth{
 Manuscript Submitted to IEEE Transactions Industrial Informatics
% Manuscript
}
{
% Semi-implict 
% Enhancing Supervised Nonlinear Probabilistic Latent Variable Models for Inferential Sensors: An Optimization Perspective
% Zhichao Chen \MakeLowercase{\textit{\textit{et al}.}}: Rethinking the Parameter Learning of Nonlinear Dynamical Probabilistic Latent Variable Models
\vspace{-0.5cm}
}

\maketitle
\vspace{-1em}
\begin{abstract}
Nonlinear Probabilistic Latent Variable Models (NPLVMs) are a cornerstone of soft sensor modeling due to their capacity for uncertainty delineation. However, conventional NPLVMs are trained using amortized variational inference, where neural networks parameterize the variational posterior. While facilitating model implementation, this parameterization converts the distributional optimization problem within an infinite-dimensional function space to parameter optimization within a finite-dimensional parameter space, which introduces an approximation error gap, thereby degrading soft sensor modeling accuracy. To alleviate this issue, we introduce KProxNPLVM, a novel NPLVM that pivots to relaxing the objective itself and improving the NPLVM's performance. Specifically, we first prove the approximation error induced by the conventional approach. Based on this, we design the Wasserstein distance as the proximal operator to relax the learning objective, yielding a new variational inference strategy derived from solving this relaxed optimization problem. Based on this foundation, we provide a rigorous derivation of KProxNPLVM's optimization implementation, prove the convergence of our algorithm can finally sidestep the approximation error, and propose the KProxNPLVM by summarizing the abovementioned content. Finally, extensive experiments on synthetic and real-world industrial datasets are conducted to demonstrate the efficacy of the proposed KProxNPLVM. 

\end{abstract}
\vspace{-0.3cm}

\section{Introduction}\label{sec:introduction}

\IEEEPARstart{P}{robablistic} latent variable models (PLVMs)~\cite{bishop2006pattern,murphy2012machine}, which explicitly model the data uncertainty, are crucial for industrial soft sensor modeling, ensuring product quality maintenance, reducing energy consumption, and increasing economic income. For instance, predicting the quality variable content in a distillation column allows for precise determination of the reflux ratio, ensuring efficient use of hot utility and reducing carbon emissions~\cite{11122530}. Similarly, predicting the reactant gas composition in a reactor enables optimal control of the reactor, thereby reducing the operational costs of the industrial plant~\cite{10778061}. \emph{Consequently, the effective development of PLVMs that accurately predict quality variables is of great importance in soft sensor modeling.}

However, PLVMs introduce latent variables, so the learning procedure must additionally infer their distributions~\cite{bishop2006pattern}, which distinguishes PLVMs from non-PLVMs. Accordingly, under the maximum log-likelihood principle, Kullback-Leibler (KL) divergence is adopted as the objective to evaluate and guide latent-variable inference. In practice, an approximation distribution, termed the variational distribution, is introduced to approximate the conditional distribution of latent variables given the observed data. Subsequently, model parameters are learned using the resulting optimized variational distribution~\cite{8588399,blei2017variational}.
% However, PLVMs introduce latent variables, necessitating latent variable distribution inference within the learning procedure, which distinguishes them from non-PLVMs. To this end, based on the maximum log-likelihood principle, the Kullback-Leibler (KL) divergence is selected as the objective to evaluate and guide the latent variable distribution inference, and an approximation distribution termed `variational distribution' is introduced to estimate the conditional distribution of the latent variables given the observed data. Subsequently the parameter learning is performed using this optimized approximation distribution~\cite{8588399}.
% the Kullback-Leibler (KL) divergence is often selected as the objective function for latent variable distribution inference, and an approximation distribution termed `variational distribution' is introduced to estimate the conditional distribution of the latent variables given the observed data, and subsequent parameter learning is performed using this optimized approximation distribution~\cite{8588399}.

Notably, KL divergence training requires effectively ``inverting'' the generative function that maps latent variables to observed data. Earlier works attempted to use linear structures to model this generative function, for example, probabilistic principal component analysis~\cite{tipping1999probabilistic} and Gaussian mixture models~\cite{ghahramani1999variational}. Various soft sensor modeling techniques have been developed, such as probabilistic principal component regression~\cite{10507044,10982263} and Gaussian mixture regression~\cite{9874385}. However, these linear structures often fail to capture the nonlinear dynamics of complex industrial processes. Consequently, researchers have explored using neural networks~\cite{9638604} to model the generative function (which is called generative network in this context), leading to the development of nonlinear probabilistic latent variable models (NPLVMs). Note that deep generative models have demonstrated broad application prospects in industrial scenarios~\cite{zhang2022diversifying,zhang2025domain,11202603}. In the context of soft sensor modeling, NPLVMs similarly leverage generative networks to achieve competitive performance. For example, Shen et al.~\cite{shen2020nonlinear} introduced supervised NPLVMs to model the primary reformer. Further studies have incorporated neural network modules, such as recurrent neural networks and Transformers, to model the spatiotemporal properties of industrial data. For instance, Yao et al.~\cite{9797056} integrated RNNs into NPLVMs to model temporal dependencies. In contrast, Sui et al.~\cite{10765125} utilized Koopman analysis in conjunction with RNNs for potassium chloride flotation process modeling. Chen et al.~\cite{chen2023adaptive} introduced the Transformer architecture for sulfur recovery unit modeling. It should be noted that neural network modules may not always satisfy the invertibility requirement. Therefore, NPLVMs are typically trained using amortized variational inference (AVI)~\cite{kingma2013auto,ganguly2022amortized}. In AVI, an inference network parameterizes the variational distribution, and the inference and generative networks are trained jointly using standard deep-learning backends~\cite{TorchNips}. Among the aforementioned works, we observe that, with the rapid development of neural networks, numerous approaches have been proposed to enhance the expressiveness of the variational family. Nevertheless, a major limitation termed ``approximation accuracy'' remains. Specifically, the true posterior can be viewed as an element of an effectively infinite-dimensional function space, whereas in practice it must be represented by a finite-dimensional, parameterized variational distribution. This inherent restriction can limit how closely the variational posterior matches the true posterior and, in turn, may reduce the modeling accuracy of soft sensors.
% Among the abovementioned works, we observe that with the development of neural networks, various approaches have been introduced to improve the capability of variational distribution family. There remain major limitation termed approximation accuracy. Specifically, the posterior distribution is finding a function within infinite-dimensional space, while parameterizing it within finite space remain limited due to this parameterization, potentially reducing the soft sensor modeling accuracy.

The key to alleviating this issue is to sidestep the direct optimization of the KL divergence and find alternative ways to infer the distribution of the latent variable model. Notably, there has been rapid progress in \emph{optimization over probability measures} that uses the Wasserstein distance as a proximal term~\cite{parikh2014proximal}. This line of work provides a principled way to construct variational schemes whose iterates admit clear interpretations at the probability distribution level, and it has been successfully used to design numerical solvers for measure-valued dynamics. In particular, Wasserstein proximal-gradient methods have been developed to compute strong solutions of continuity-equation systems and related PDEs, as exemplified by references~\cite{8890903,caluya2021wasserstein,mokrov2021large}. Beyond algorithmic development, these Wasserstein-proximal formulations also enable rigorous convergence analyses for structured composite objectives, as exemplified by references~\cite{salim2020wasserstein,JiaoJiaoparGradientFlow,chen2023density}. Despite these encouraging developments, existing Wasserstein-proximal frameworks are largely developed in a problem-agnostic manner, with emphasis on general measure-valued dynamics. For PLVMs, important gaps still remain. First, when the objective is the KL divergence to a posterior, it is not straightforward to turn Wasserstein-proximal updates into practical implementations. Second, existing convergence analyses are typically derived under general assumptions, and thus provide limited model-specific guarantees for KL-based latent-variable inference.
%Notably, there has been great progress for the optimization over distribution techniques, which lift the optimization space to finite-dimensional space to infinite dimensional space, as represented by the Wasserstein distance induced proximal gradient descent, has been gradually applied for designing algorithm and analyze conducting the convergence analysis. For example, references~\cite{8890903,caluya2021wasserstein,mokrov2021large} have introduced Wasserstein distance induced proximal gradient descent scheme to analyze and develop strong solution solver of the continuity equation. Meanwhile, related theoretical studies have been proposed based on Wasserstein distance induced proximal gradient descent scheme, for example, convergence rate 

To fill these technical gaps, in this paper, we introduce the Wasserstein distance as a proximal operator to relax the optimization of the KL divergence, facilitating a convergence-guaranteed algorithm for latent variable inference, which we term the kernelized Proximal Gradient Descent-based (KProx) algorithm. Based on this, we further develop the training process for the inference network and the generative network accordingly, and propose a novel training algorithm for NPLVM training. Finally, we conduct various experiments to validate the efficacy of the proposed approach.

\noindent\textbf{Contributions:} The contributions are summarized as follows:
\begin{enumerate}[leftmargin=*]
    \item{This paper theoretically characterizes the approximation gap in latent-variable inference induced by parameterizing the variational distribution with a finite parameter space, introducing a Wasserstein distance-based proximal operator to alleviate this issue.
    % This paper first proves that the accuracy of latent variable inference approximation is limited by parameterizing the variational distribution with a finite parameter space, introducing a Wasserstein distance-based proximal operator to alleviate this issue.
    }
    \item{This paper derives a computationally implementable procedure to realize the Wasserstein distance-based proximal operator latent variable inference strategy and proves its asymptotic local convergence under mild assumptions.}
    \item{This paper advances a novel algorithm specifically designed for NPLVMs in the context of soft sensor modeling.}
    % \item{Various experiments are conducted to demonstrate the efficacy of the proposed approach.}
\end{enumerate}

\noindent\textbf{Organization:} The organization of this manuscript can be summarized as follows: To facilitate the reading of this manuscript, the preliminaries are given in~\Cref{sec:Preliminaries}. On this basis, the detailed analysis of the abovementioned problems, and the corresponding solution strategy to the abovementioned problems are given in~\Cref{sec:proposedApproach}. After that, the experimental results are given in~\Cref{sec:experResults}. Finally, the conclusions, limitations, and future research directions are listed in~\Cref{sec:Conclusions}.

% \section{Literature Review \& Technical. Gaps}\label{sec:Preliminaries}

\section{Preliminaries}\label{sec:Preliminaries}
\subsection{Amortized Variational Inference}
% \indent Variational inference is a commonly used algorithm to learn the parameters of PLVMs \cite{blei2017variational,8588399}. 
%Denote the latent variable as $z$, the dataset as $\mathcal{D}$, the PLVMs attempt to model the data generative distribution $\mathcal{P}(\mathcal{D}\vert z)$ by a neural network parameterized by  $\theta$ denoted as $p_\theta(\mathcal{D}|z)$. 
Let the latent variable be $z$ and the dataset be $D$. NPLVMs model the generative distribution $\mathcal{P}(\mathcal{D}\vert z)$ using a neural network parameterized by $\theta$, denoted $p_\theta(\mathcal{D}\vert z)$. Based on this, the basic idea of variational inference is to approximate the posterior distribution of the model distribution 
$\mathcal{P}(z|\mathcal{D})$ by variational distribution $\mathcal{Q}(\theta)$ through minimizing the Kullback-Leibler divergence (KL divergence) as follows~\cite{8588399}:
\begin{equation} \label{VBApproximate}
 \mathcal{Q}^*(z) = \mathop{\arg\min}_{\mathcal{Q}(z)}{\underbrace{{\mathbb{D}}_{\text{KL}}\left[\mathcal{Q}(z)\Vert \mathcal{P}(z \vert\mathcal{D})\right]}_{\coloneqq \int{\mathcal{Q}(z)\log{\frac{\mathcal{Q}(z)}{\mathcal{P}(z\vert\mathcal{D})}}\mathrm{d} z}} }. %= {}.%\coloneqq .
\end{equation}
Eq.~\Cref{VBApproximate} can be further reformulated as follows:
\begin{equation}\label{KLDivergence}
\begin{aligned}
& {{\mathbb{D}}_{\text{KL}}}\left[q(z)\Vert \mathcal{P}(z\vert\mathcal{D})\right]
%  =& \int {q(\Theta )\log \frac{{q(\Theta )p(\mathcal{D})}}{{p(\mathcal{D}\left| \Theta  \right.)p(\Theta )}}\mathrm{d}\Theta } \\
% & = \int {q{\rm{(}}\Theta {\rm{)}}\log \frac{{q{\rm{(}}\Theta {\rm{)}}}}{{p{\rm{(}}\Theta |\mathcal{D}{\rm{)}}}}\mathrm{d}\Theta } \\
 =\underbrace{ {\mathbb{D}_{\text{KL}}}(q(z)||\mathcal{P}(z ))-\int{q(z )\log p(\mathcal{D}|z)\mathrm{d}z }}_{\coloneqq -\mathcal{L}_{\text{ELBO}}} +\log p(\mathcal{D}) ,
\end{aligned}
\end{equation}
where $\mathcal{L}_{\text{ELBO}}$ is the \uline{E}vidence \uline{L}ower \uline{BO}und. It can be concluded that lifting ELBO $\mathcal{L}_{\text{ELBO}}$ is equivalent to reducing the KL divergence term $\mathbb{D}_{\rm{KL}}[q(\theta)\Vert \mathcal{P}(\theta\vert \mathcal{D})]$, thereby realizing the variational inference.

On this basis, the AVI attempts to parameterize the $\mathcal{Q}(z)$ via a neural network $q_\varphi(z\vert \mathcal{D})$ with parameter $\varphi$ as:
\begin{equation}
   \mathcal{Q}(z) \coloneqq q_\varphi(z\vert\mathcal{D}),
\end{equation}
and learns the parameter of $\varphi$ and $\theta$ jointly as follows:
\begin{equation}\label{eq:aviELBOEquation}
    \mathop{\arg\max}_{\varphi,\theta}~~\mathbb{E}_{q_\varphi(z|\mathcal{D})}[\log{p_\theta(\mathcal{D}|z)}] - \mathbb{D}_{\rm{KL}}[q_\varphi(z\vert\mathcal{D})\Vert \mathcal{P}(z)].
\end{equation}
% Notably, the AVI converts the optimization of $\mathcal{Q}(z)$ in an infinite dimensional function space to the optimization of $\varphi$ in a finite dimensional parameter space, thereby facilitate the model learning procedure.
Notably, AVI converts the optimization of $\mathcal{Q}(z)$ in an infinite-dimensional function space into optimizing parameters $\varphi$ in a finite-dimensional parameter space, thereby facilitating the model learning procedure.

\iffalse
% ~\Cref{KLDivergence} can be rewritten as \eqref{ELBOReform}:\\
\begin{equation}\label{ELBOReform}
  {{\mathbb{D}}_{\text{KL}}}(q{\rm{(}}\Theta {\rm{)}}\left\| {p{\rm{(}}\Theta \left| \mathcal{D} \right.{\rm{)}}} \right.) = \log p{\rm{(}}\mathcal{D}) - \text{ELBO}
\end{equation}
As the log-likelihood term $log p(\mathcal{D})$ is constant and the KL divergence term is nonnegative, \eqref{VBApproximate} can be reformulated as \eqref{VariationalApproxReform}:
\begin{equation} \label{VariationalApproxReform}
\begin{aligned}
  {q^*}{\rm{(}}\Theta {\rm{) = arg min }}\log p{\rm{(}}D) - \text{ELBO} = {\rm{arg max }}\text{ELBO}
\end{aligned}
\end{equation}
Therefore, the optimal variational distribution can be obtained via maximizing the ELBO defined in \eqref{ELBODef}. It should be pointed out that, the ELBO term still contains KL divergence, which is hard to compute most of the time. To alleviate this issue, conventional model-based variational inference~\cite{8588399} attempts to introduce the exponential family as the hypothesis space of variational distribution $q(\Theta)$.
\fi

\subsection{Proximal Operators and Optimization}\label{subsec:proximalOptimizationMethod}

The proximal operator is a key concept in optimization theory, defined for a proper, lower semicontinuous convex function $g: \mathbb{R}^\mathrm{D} \rightarrow \mathbb{R} \cup \{+\infty\}$ and parameter $\varepsilon > 0$ as:
\begin{equation}
    \text{prox}_{\varepsilon f}(x) = \mathop{\arg\min}_{y} g(y) + \frac{1}{2\varepsilon}\|y-x\|_2^2 .
\end{equation}
This operator balances minimizing $g$ while remaining close to point $x$, with $\varepsilon$ controlling this trade-off. 

A key insight is that the standard gradient descent update can be interpreted through the lens of the proximal operator. For a differentiable function $f$, the update step $x_{k+1} = x_k - \varepsilon \nabla f(x_k)$ is precisely equivalent to applying the proximal operator to a first-order Taylor approximation of $f$ around $x_k$:
\begin{equation}
    x_{k+1} = \mathop{\arg\min}_{y \in \mathbb{R}^n} \underbrace{f(x_k) + \nabla f(x_k)^\top(y-x_k)}_{\coloneqq g(y)} + \frac{1}{2\varepsilon}\|y-x_k\|_2^2 .
\end{equation}
This perspective naturally extends to solving composite optimization problems of the form $F(x) = g(x) + h(x)$, where $g$ is smooth and $h$ is convex but possibly non-smooth. The resulting algorithm, known as proximal gradient descent~\cite{parikh2014proximal}, thus elegantly decouples the smooth and non-smooth components of the objective. \emph{In this scheme, $h(x)$ is called proximal term.} By updating the differentiable term via a standard gradient step and treating the non-differentiable term through a proximal mapping (for example induced by the Wasserstein-distance~\cite{salim2020wasserstein,JiaoJiaoparGradientFlow} or KL-divergence~\cite{skretafeynman,haoWangCausalBalancing,10795195}), the proximal gradient descent offers a powerful and versatile framework for solving complex optimization problems in which conventional gradient-based methods may fail.

\subsection{Wasserstein Space and Wasserstein Distance}\label{subsec:WassDistanceAndWassSpace}
The Wasserstein space $\mathscr{P}_2(\mathrm{D})$ is the finite second-order moment space, whose definition can be given as follows:
\begin{equation}\label{eq:wassDefinition}
\begin{aligned}
   \mathscr{P}_2(\mathrm{D})  
   \coloneqq \{\mathcal{Q}:\mathbb{R}^{\mathrm{D}_{\rm{LV}}}\to\mathbb{R}^+ |\int\mathcal{Q}(z)\mathrm{d}z=1,\mathbb{E}_{\mathcal{Q}(z)}[z^\top z]<\infty\}.
\end{aligned}
\end{equation}
Based on this, the 2-Wasserstein distance ($\mathcal{W}_2$) measures the discrepancy between two probability distributions~\cite{villani2009optimal,wang2026DistDF}, $\mathcal{Q}(z)$ and $\mathcal{Q}_{\boldsymbol{T}}(z)$, defined on a space $\mathcal{Z}$ (e.g., $\mathbb{R}^{\mathrm{D}_{\rm{LV}}}$). Specifically, let $\Pi(\mathcal{Q}(z), \mathcal{Q}_{\boldsymbol{T}}(z))$ be the set of all joint probability distributions $\pi(z, z')$ on $\mathcal{Z} \times \mathcal{Z}$ whose marginals are $\mathcal{Q}(z)$ and $\mathcal{Q}_{\boldsymbol{T}}(z)$, respectively. That is, for any $\pi \in \Pi$, we get $\int \pi(z, z') \mathrm{d}z' = \mathcal{Q}(z)$ and $\int \pi(z, z') \mathrm{d}z = \mathcal{Q}_{\boldsymbol{T}}(z')$. Each such joint distribution $\pi$ is called a ``coupling'' or a ``transport plan," as it describes how mass from $\mathcal{Q}(z)$ is moved to $\mathcal{Q}_{\boldsymbol{T}}(z')$. The squared 2-Wasserstein distance is then defined as the minimum transportation cost over all possible plans:
\begin{equation}\label{eq:kangWassersteinDistance}
    \mathcal{W}_2^2(\mathcal{Q}(z), \mathcal{Q}_{\boldsymbol{T}}(z)) = \inf_{\pi \in \Pi} \int \|z - z'\|_2^2 \, \mathrm{d}\pi(z, z') \ge 0.
\end{equation}
Notably, the Wasserstein distance is capable of measuring discrepancy \emph{even when the supports of the two distributions do not overlap.}
%Notably, Wasserstein distance can compare the discrepancy even the support of two distributions are not overlapped.

In addition, $ \mathcal{W}_2^2$ can be defined by finding an optimal transport map $\boldsymbol{T}: \mathbb{R}^{\mathrm{D}_{\rm{LV}}} \to \mathbb{R}^{\mathrm{D}_{\rm{LV}}}$ that minimizes the average cost of transporting mass from $\mathcal{Q}(z)$ to $\mathcal{Q}_{\boldsymbol{T}}(z)$ as follows:
\begin{equation}
    \mathcal{W}_2^2(\mathcal{Q}(z), \mathcal{Q}_{\boldsymbol{T}}(z)) = \inf_{\boldsymbol{T}:{\boldsymbol{T}}_{\#}\mathcal{Q}(z) = \mathcal{Q}_{\boldsymbol{T}}(z)} \int \|z - \boldsymbol{T}(z)\|_2^2 \, \mathrm{d}\mathcal{Q}(z),
\end{equation}
%where $\boldsymbol{T}_{\#}$ denotes the pushforward measure, and $\boldsymbol{T}(z) = z + \varepsilon v(z)$, with $\varepsilon$ is an Infinitesimal and $v(z)$ is called velocity field. 
where $\boldsymbol{T}_{\#}$ denotes the pushforward measure, $\mathcal{Q}_{\boldsymbol{T}}$(z) is the PDF after applying the transportation map $\boldsymbol{T}$ to $\mathcal{Q}$(z), and $\boldsymbol{T}(z) = z + \varepsilon v(z)$, with $\varepsilon$ representing an infinitesimal quantity and $v(z)$ referred to as the velocity field. Let $\{\mathcal{Q}_t\}_{t=1}^{\mathrm{T}}$ be a path in the space of probability measures $\mathscr{P}_2(\mathbb{R}^{\mathrm{D}_{\rm{LV}}})$, and $\mathcal{Q}_t$ indicates the PDF at time index $t$. 

A curve $\{\mathcal{Q}_t(z)\}$ in $\mathscr{P}_2(\mathrm{D})$ is governed by the continuity equation~\cite{santambrogio2017euclidean}: 
%Such a path is governed by the continuity equation, which describes the transport of mass over time:
\begin{equation}\label{eq:continuityEquationPDE}
    \frac{\partial \mathcal{Q}_t(z)}{\partial t} = -\nabla\cdot [v(z)\mathcal{Q}_t(z)]~~\text{or}~~\frac{\mathrm{d}\mathcal{Q}_t(z)}{\mathrm{d}t} = -\mathcal{Q}_t(z)\nabla\cdot v(z),
\end{equation}
where $\frac{\partial \mathcal{Q}_t(z)}{\partial t}$ is the Lagrangian derivative, $\frac{\mathrm{d}\mathcal{Q}_t(z)}{\mathrm{d}t}$ is the Eulerian derivative, $v: \mathbb{R}^{\mathrm{D}_{\rm{LV}}}\to \mathbb{R}^{\mathrm{D}_{\rm{LV}}}$ is called perturbation direction that defines the instantaneous velocity of the probability particles at location $z$ and time $t$.

% The resulting algorithm, known as proximal gradient descent, combines a standard gradient step on $f$ with a proximal step on $h$, forming the basis for our proposed method.
%

%Notably, gradient descent emerges as a special case of proximal operations. For a differentiable function $f$, the gradient descent update $x_{k+1} = x_k - \alpha \nabla f(x_k)$ can be reformulated as:
% \begin{equation}
% \begin{aligned}
%     x_{k+1}  = \mathop{\arg\min}_{y} f(x_k) + [\nabla_{x} f(x_k)]^\top(y-x_k)+ \frac{1}{2\alpha}\|y-x_k\|_2^2 
% \end{aligned}
% \end{equation}
% This reveals gradient descent as $x_{k+1} = \text{prox}_{\alpha g}(x_k)$ where $g(y)$ is the first-order approximation of $f$ at $x_k$. 

\section{Methodology}\label{sec:proposedApproach}
\subsection{Motivation Analysis}\label{subsec:MotivationAnalysis}
While the learning objective in~\Cref{eq:aviELBOEquation} is foundational to modern NPLVMs, its practical application imposes significant constraints on the variational distribution $q_\varphi(z|\mathcal{D})$, hindering effective learning. Specifically, two conditions must be met:
\begin{enumerate}[leftmargin=*]
    \item{The support of $q_\varphi(z\vert\mathcal{D})$ must be a subset of the support of the prior $\mathcal{P}(z)$ to ensure the KL divergence is well-defined.}
\item{The KL divergence term, $\mathbb{D}_{\rm{KL}}[q_\varphi(z\vert\mathcal{D})\Vert \mathcal{P}(z)]$, must be tractable. This is essential for computing the gradients needed to optimize $\varphi$.}
\end{enumerate}
To ensure tractability, $q_\varphi(z\vert\mathcal{D})$ should be restricted to a simple family (e.g., Gaussian), for which the KL divergence has a convenient closed-form solution. This limits the expressiveness of the variational posterior, regardless of the complexity of its underlying neural network. More specifically, we establish the following lemma characterizing the approximation error in this setting, using the KL divergence as the error metric, based on Theorem 7.2 in~\cite{amari2016information}:
\begin{lemma}\label{thm:klDivergenceBound}
Let $\mathcal{P}(z| \theta)$ and $\mathcal{Q}_\varphi(z)=\mathcal{Q}(z|\varphi)$ be two probability distributions belonging to the exponential family, defined as: $\mathcal{Q}(z|\varphi) = h(z) \exp(\upeta^\top(\varphi) \mathscr{T}(z) - \mathcal{A}(\varphi))$ where $\upeta(\varphi)$ is the natural parameter, $\mathscr{T}(z)$ is the sufficient statistic, $\mathcal{A}(\varphi)$ is the log-partition function, and $h(z)$ is the base measure. With the following assumptions:
\begin{enumerate}[leftmargin=*]
\item[(A1)]{\textbf{Same regular exponential family.}
$\mathcal{P}(z\vert\theta)$ and $\mathcal{Q}(z\vert\varphi)$ belong to the same regular exponential family (same $h(z)$ and $\mathscr{T}(z)$), and $ \mathcal{A}(\varphi)$ is twice continuously differentiable in the neighborhood.
}
\item[(A2)]{\textbf{$M$-Lipschitz condition.}
With Fisher information matrix $\mathscr{I}(\eta)\coloneqq  \nabla^2 \mathcal{A}(\eta)$, we have:
$
\|\mathscr{I}(x)-\mathscr{I}(y)\| \le M\|x-y\|.
$
}
\item[(A3)]{\textbf{Locally strong convexity.}
With $w\coloneqq x-y$ and a positive constant $\kappa$, we have:
$
w^\top \mathscr{I}(x)w \ge \kappa \Vert w \Vert^2.
$
}
\end{enumerate}
Then, for $\theta$ and $\varphi$ sufficiently close, the KL divergence between $\mathcal{Q}(z|\varphi)$ and $\mathcal{P}(z| \theta)$ is approximately lower bounded as follows: % long as $\upeta(\varphi)$ and $\upeta(\theta)$ are close enough and the ${M}$-Lipschitz Fisher information matrix is local strong convex:
\begin{equation}\label{eq:klIneqaulityResult}
    \mathbb{D}_{\rm{KL}}[\mathcal{Q}_\varphi(z) || \mathcal{P}(z| \theta)] %\gtrsim 
  \ge  \frac{1}{2} [\upeta(\varphi) - \upeta(\theta)]^\top \mathscr{I}(\theta_1) [\upeta(\varphi) - \upeta(\theta)].
\end{equation}
% where $\mathscr{I}(\theta_1)$ is the Fisher information matrix with respect to the natural parameter $\upeta$, evaluated at $\theta_1$, and is given by $\mathscr{I}(\varphi) \coloneqq   \nabla^2 \mathcal{A}(\varphi) $.
\end{lemma}
The abovementioned assumptions typically hold in industrial process data–driven modeling when we restrict attention to a normal operating region where the data-generating mechanism is well approximated by a single regular exponential-family model (e.g., Gaussian for continuous sensor readings, Bernoulli for binary quality indicators) with a common base measure and sufficient statistics, and where the natural parameters remain in a bounded neighborhood of the nominal condition (no extreme extrapolation, and no probabilities near the boundary). Notably, to the best of our knowledge, most NPLVMs~\cite{shen2020nonlinear,9625835,kong2022latent} that employ unimodal Gaussian distributions or Gaussian mixture models implicitly make this assumption. As such, the proof of this lemma is as follows: % On such a local, bounded set, the log-partition function $ \mathcal{A}(\eta) $ is smooth, the Fisher information $ \mathscr{I}(\eta)=\nabla_\eta^2 \mathcal{A}(\eta) $ varies smoothly (hence is locally $M$-Lipschitz), and the process provides sufficient excitation/identifiability so that $ \mathscr{I}(\eta) $ is uniformly positive definite, implying local strong convexity.

\begin{proof}
For $f(x)$, if the hessian matrix $\nabla^2 f(x)$ is positive semidefinite and ${M}$-Lipschitz and locally strong convex, we get the following inequality as long as $y$ and $x$ are close enough (detailed derivation about this condition is given in the supplementary material):
\begin{equation}\label{eq:HessianInequality}
    f(y) \ge f(x) + [\nabla f(x)]^\top[y-x] + \frac{1}{2} [y-x]^\top \nabla^2f(x)[y-x].
\end{equation}
On this basis, using the definition of KL divergence, we get:
\begin{equation}
    \begin{aligned}
        &\mathbb{D}_{\rm{KL}}[\mathcal{Q}_\varphi(z) || \mathcal{P}(z| \theta)] =\int \mathcal{Q}_\varphi(z) \log \frac{h(z) \exp(\upeta(\varphi)^\top \mathscr{T}(z) - \mathcal{A}(\varphi))}{h(z) \exp(\upeta(\theta)^\top \mathscr{T}(z) - \mathcal{A}(\theta))} \mathrm{d}z 
        % &= \int \mathcal{Q}_\varphi(z) \left[ (\upeta(\varphi) - \upeta(\theta))^\top \mathscr{T}(z) - (\mathcal{A}(\varphi) - \mathcal{A}(\theta)) \right] dz \\
        = (\upeta(\varphi) - \upeta(\theta))^\top \mathbb{E}_{\mathcal{Q}_\varphi(z)}[\mathscr{T}(z)] - (\mathcal{A}(\varphi) - \mathcal{A}(\theta)).
        \end{aligned}
    \end{equation}
Applying the Taylor's expansion to $\mathcal{A}(\theta)$, the following equation can be obtained:
\begin{equation}
    \mathcal{A}(\theta) \approx \mathcal{A}(\varphi) + \nabla \mathcal{A}(\varphi)^\top (\theta - \varphi) + \frac{1}{2} (\theta - \varphi)^\top \nabla^2 \mathcal{A}(\varphi) (\theta - \varphi).
\end{equation}
Based on this, we get:
\begin{equation}
        \begin{aligned}
        &\mathbb{D}_{\rm{KL}}[\mathcal{Q}_\varphi(z) || \mathcal{P}(z| \theta)] \\
        \approx& \cancel{(\upeta(\varphi) - \upeta(\theta))^\top \mathbb{E}_{\mathcal{Q}_\varphi(z)}[\mathscr{T}(z)]} -\cancel{\mathbb{E}_{\mathcal{Q}_\varphi(z)}[\mathscr{T}(z)]^\top (\upeta(\theta) - \upeta(\varphi))}  + \frac{1}{2} (\upeta(\theta) - \upeta(\varphi))^\top \nabla^2 \mathcal{A}(\varphi) (\upeta(\theta) - \upeta(\varphi)) \\
        % &= \frac{1}{2} (\upeta(\varphi) - \upeta(\theta))^\top \nabla^2 \mathcal{A}(\varphi) (\upeta(\varphi) - \upeta(\theta))\\
        =& \frac{1}{2} [\upeta(\varphi) - \upeta(\theta)]^\top \mathscr{I}(\varphi) [\upeta(\varphi) - \upeta(\theta)].
        \end{aligned}
\end{equation}
Using the fact that the Fisher information matrix $\mathscr{I}(\theta)$ is a positive semidefinite matrix, we get~\Cref{eq:klIneqaulityResult} based on~\Cref{eq:HessianInequality}.
\end{proof}
From the abovementioned lemma, it can be observed that \emph{when parameterizing the variational distribution within a finite-dimensional parameter $\varphi$, the approximation error will be lower bounded by the selection of the distribution family, thereby resulting in a large approximation error when misselecting the parameter family.} To better support this lemma, we introduce the following toy problem: We set $\mathcal{P}(z\vert\mathcal{D})\propto \tfrac{1}{2}\mathcal{N}(-2,1) + \tfrac{1}{2}\mathcal{N}(2,1)$ and the optimization problem is set as $\mathop{\arg\min}_{\mathcal{Q}(z)}\mathbb{D}_{\rm{KL}}[\mathcal{Q}(z)\Vert \mathcal{P}(z\vert\mathcal{D})]$. As shown in~\Cref{fig:gaussApproximation}, when $\mathcal{Q}(z)$ is constrained to a predefined Gaussian family, the variational distribution fails to accurately approximate the true posterior if its family does not match that of $\mathcal{P}(z\vert\mathcal{D})$. In contrast, as demonstrated in~\Cref{fig:malaApproximation}, employing the KProx algorithm allows it to gradually approach $p(z\vert\mathcal{D})$, achieving significantly improved approximation accuracy. Thus, the key to alleviating the approximation error problem is to find another objective function that bypasses the direct computation of $\mathbb{D}_{\rm{KL}}[\mathcal{Q}(z)\Vert \mathcal{P}(z\vert\mathcal{D})]$, which is one of the key contributions of this manuscript.

\begin{figure}[!h]
  \vspace{-0.4cm}
  \centering
  % figures\cut_sampled_results.pdf
  % figures\cut_matlab_graph\cut_sampled_results_50.pdf
  \subfigure[$\mathcal{Q}(z)$ within Gaussian family.]{\includegraphics[width=0.30\linewidth]{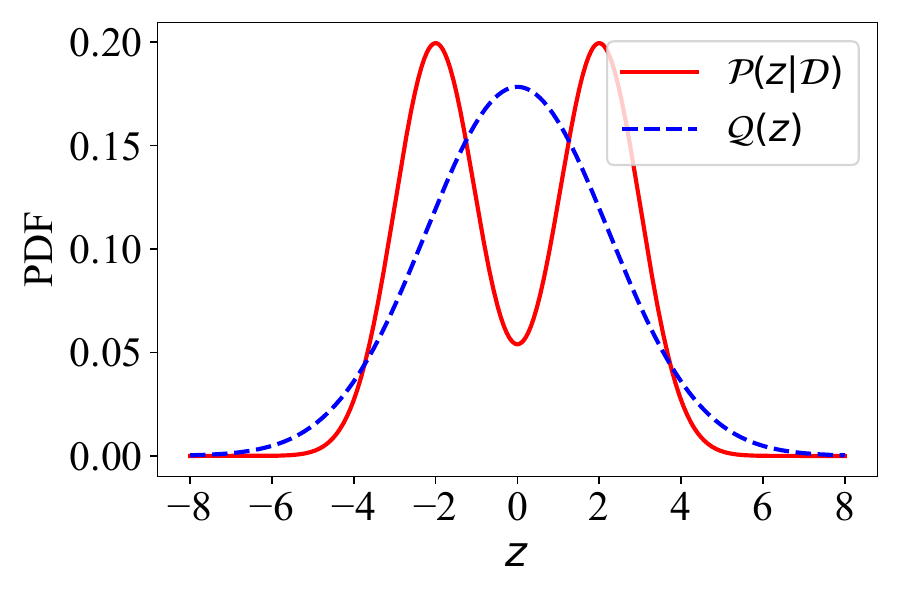}\label{fig:gaussApproximation}}
  \subfigure[$\mathcal{Q}(z)$ by KProx.]{\includegraphics[width=0.30\linewidth]{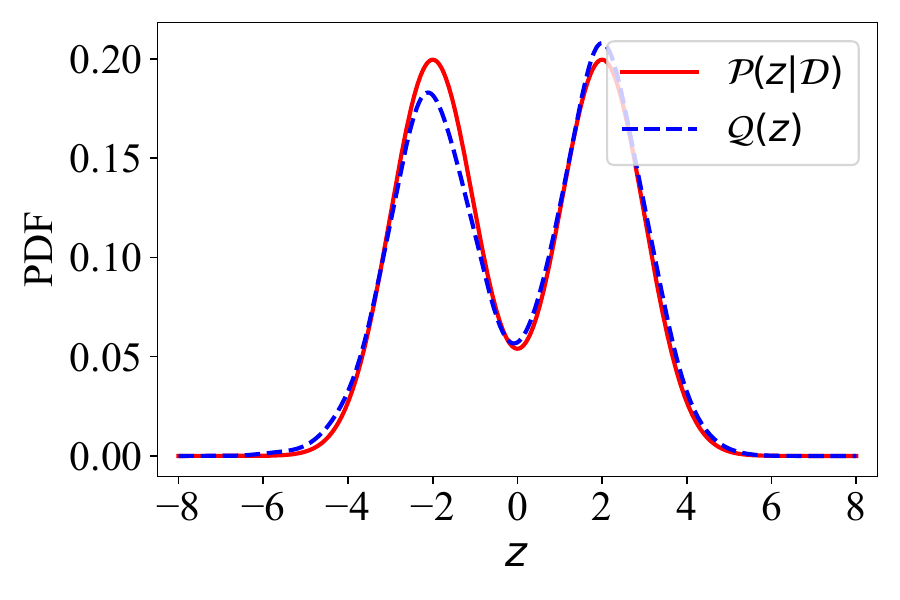}\label{fig:malaApproximation}}
  
  % \subfigure[$w=2$, $\kappa_\mathrm{p}=0.5$.]{\includegraphics[width=0.24\linewidth]{figures/transport/transport_0.5_2.pdf}}
  % \subfigure[Transport cost of different $w$.]{\includegraphics[width=0.24\linewidth]{figures/transport/match_loss.pdf}}

  \caption{Posterior distribution approximation comparison.
  }\label{fig:samplingResults}
  \vspace{-0.3cm}
\end{figure}

% From the abovementioned theorem, it can be observed that \emph{when parameterizing the variational distribution within finite dimensional parameter $\varphi$, the approximation error will be lower bounded by the selection of distribution family, thereby resulting in a large approximation error when misselecting the parameter family.} Thus, the key to alleviating this problem is to find another objective function that bypasses the direct computation of $\mathbb{D}_{\rm{KL}}[\mathcal{Q}(z)\Vert \mathcal{P}(z\vert\mathcal{D})]$, which is one of the key contributions for this manuscript.
% Consequently, the model may fail to capture the true posterior's geometry, motivating our search for a more flexible learning framework.
% Based on the abovementioned consideration, the $q_\varphi(z\vert\mathcal{D})$ should be chosen easily such that the computational of $\mathbb{D}_{\rm{KL}}[q_\varphi(z\vert\mathcal{D})\Vert \mathcal{P}(z)]$ is well-defined and tractable. As a result, the model expressiveness can be greatly restrcited no matter how the $q_\varphi(z\vert\mathcal{D})$ is complex.
%Specifically, to ease the computation of $\mathcal{L}_{\text{ELBO}}$ such that the gradient with-respect-to parameter $\varphi$ is easy-to-obtain, the approximation $q_\varphi(z\vert\mathcal{D})$ should satisfy the condition that $q_\varphi(z\vert\mathcal{D})\gg \mathcal{P}(z)$ and satisfy some specific relationship, otherwise the KL divergence term is intractable. As a result, the approximation distribution is greatly eliminated, thereby hinder the difficulty. 

\subsection{Wasserstein Distance as Proximal Operator}\label{subsec:WassProximalOperator}

From Lemma \ref{thm:klDivergenceBound}, we observe that directly minimizing the KL divergence can be limited by the accuracy of its approximation. Motivated by this observation, we adopt a practical alternative: rather than optimizing the KL divergence itself, we progressively optimize a tractable upper bound on the KL term, which in turn drives down the original KL objective~\cite{weiguo2026ProximalSampler}. Based on this, we consider progressively reducing the $\mathcal{Q}(z)$ within the Wasserstein space. Based on this intuition, we consider using the Wasserstein distance as the proximal term and formulate the problem we aim to solve as follows:
% The vital step to realize the variational inference is reducing the KL divergence between $\mathcal{Q}(z)$ and the posterior distribution $\mathcal{P}(z|\mathcal{D})$. Based on this, we consider progressively reducing the $\mathcal{Q}(z)$ within the Wasserstein space. Based on this intuition, we consider using the Wasserstein distance as the \textcolor{blue}{proximal term} and formulate the problem we aim to solve as follows:
\begin{equation}\label{eq:klDivProximalOperation}
\begin{aligned}
   & \mathop{\arg\min}_{\mathcal{Q}_{\boldsymbol{T}}}~~\mathbb{D}_{\rm{KL}}[\mathcal{Q}_{\boldsymbol{T}}(z)\Vert \mathcal{P}(z|\mathcal{D})] + \frac{1}{2\varepsilon} \mathcal{W}_2^2(\mathcal{Q}_{\boldsymbol{T}}(z),\mathcal{Q}(z))\\
\Rightarrow 
& \mathop{\arg\min}_{\mathcal{Q}_{\boldsymbol{T}}}~~\mathbb{D}_{\rm{KL}}[\mathcal{Q}_{\boldsymbol{T}}(z)\Vert \mathcal{P}(z|\mathcal{D})] + \frac{1}{2\varepsilon} \mathcal{W}_2^2(\mathcal{Q}_{\boldsymbol{T}}(z),\mathcal{Q}(z)) - \mathbb{D}_{\rm{KL}}[\mathcal{Q}(z)\Vert \mathcal{P}(z|\mathcal{D})],
\end{aligned}
\end{equation}
where the last line is based on the fact that $\mathbb{D}_{\rm{KL}}[\mathcal{Q}(z)\Vert \mathcal{P}(z|\mathcal{D})]$ will not affect the optimization result. Based on this, we consider decomposing the transportation map $\boldsymbol{T}(z)$ according to~\Cref{subsec:WassDistanceAndWassSpace} as follows:
\begin{equation*}%\label{eq:transportMapResult}
    \boldsymbol{T}(z) = z + \varepsilon v(z),
\end{equation*}
where $v(z)$ is called velocity field, and $\varepsilon$ is the Infinitesimal quantity. As such, the continuity equation given by~\Cref{eq:continuityEquationPDE} can be expanded as follows:
\begin{equation}
    \mathcal{Q}_{\boldsymbol{T}}(z) = \mathcal{Q}(z) -\varepsilon\nabla\cdot{[ \mathcal{Q}(z) v(z)]} + \mathbf{H.O.T.}(\varepsilon^2),
\end{equation}
where $\mathbf{H.O.T.}$ is the abbreviation of higher-order term. Consequently,~\Cref{eq:klDivProximalOperation} can be expanded as follows (detailed derivations are given in supplementary material):
 \begin{align}\label{eq:klDivExpansionResult}
       &  \mathbb{D}_{\rm{KL}}[\mathcal{Q}_{\boldsymbol{T}}(z)\Vert \mathcal{P}(z|\mathcal{D})] + \frac{1}{2\varepsilon} \mathcal{W}_2^2(\mathcal{Q}_{\boldsymbol{T}}(z),\mathcal{Q}(z))  -  \mathbb{D}_{\rm{KL}}[\mathcal{Q}(z)\Vert \mathcal{P}(z|\mathcal{D})] \notag\\
     \le  &  \cancel{\mathbb{D}_{\rm{KL}}[\mathcal{Q}(z)\Vert \mathcal{P}(z|\mathcal{D})]}-  \cancel{\mathbb{D}_{\rm{KL}}[\mathcal{Q}(z)\Vert \mathcal{P}(z|\mathcal{D})]}  + \dfrac{\varepsilon}{2}\mathbb{E}_ \mathcal{Q}(z)[\Vert{v(z)}\Vert^2_2 ]
  + \varepsilon\int [{\mathcal{Q}(z) v^\top(z)}\nabla\dfrac{\delta \mathbb{D}_{\rm{KL}}[\mathcal{Q}(z)\Vert \mathcal{P}(z|\mathcal{D})]}{\delta \mathcal{Q}(z)}]\mathrm{d}z\notag\\
     \le & \dfrac{\varepsilon}{2}\mathbb{E}_{\mathcal{Q}(z)}[\Vert \nabla\dfrac{\delta \mathbb{D}_{\rm{KL}}[\mathcal{Q}(z)\Vert \mathcal{P}(z|\mathcal{D})]}{\delta \mathcal{Q}(z)}]
     \Vert^2_2 ] +\dfrac{\varepsilon}{2}\mathbb{E}_{\mathcal{Q}(z)}[\Vert{v(z)}\Vert^2_2 ]  + \varepsilon\int [{\mathcal{Q}(z) v^\top(z)} \nabla\dfrac{\delta \mathbb{D}_{\rm{KL}}[\mathcal{Q}(z)\Vert \mathcal{P}(z|\mathcal{D})]}{\delta \mathcal{Q}(z)}]\mathrm{d}z\notag\\
      = & \frac{\varepsilon}{2}\mathbb{E}_{\mathcal{Q}(z)}\{\Vert v(z) + \nabla\frac{\delta \mathbb{D}_{\rm{KL}}[\mathcal{Q}(z)\Vert\mathcal{P}(z\vert\mathcal{D})]}{\delta \mathcal{Q}(z)} \Vert_2^2 \}.
\end{align}
Thus, the regularized optimal solution to the proximal operator induced problem defined by~\Cref{eq:klDivProximalOperation} can be given as follows: 
\begin{equation}\label{eq:progressivelyReducingKLDivergence}
\begin{aligned}
    \boldsymbol{T}(z) %=& z-\varepsilon \nabla\frac{\delta \mathbb{D}_{\rm{KL}}[\mathcal{Q}(z)\Vert\mathcal{P}(z\vert\mathcal{D})]}{\delta \mathcal{Q}(z)}
   \overset{\text{(i)}}{=}  z+\varepsilon [\nabla\log{\mathcal{P}(z\vert\mathcal{D})} - \nabla\log{\mathcal{Q}(z)}],
\end{aligned}
\end{equation}
where `(i)' is based on the first variation of KL divergence with-respect-to $\mathcal{Q}(z)$:
\begin{equation}
  \nabla  \frac{\delta \mathbb{D}_{\rm{KL}}[\mathcal{Q}(z)\Vert\mathcal{P}(z\vert\mathcal{D})]}{\delta \mathcal{Q}(z)} =  \nabla\log{\mathcal{Q}(z)} - \nabla\log{\mathcal{P}(z\vert\mathcal{D})} .
\end{equation}

Based on the abovementioned results, we begin by initializing a set of particles $\{z_{0,i}\}_{i=1}^{\ell}$, drawn i.i.d. from an initial distribution $\mathcal{Q}_0(z)$. By recursively applying the transformation from Eq.~\eqref{eq:progressivelyReducingKLDivergence} till time $\mathrm{T}$, we generate a sequence of particles:
\begin{equation}\label{eq:realizationOfTransportationMap}
  z_{t+1, i} = \boldsymbol{T}(z_{t,i}) = z_{t,i} + \varepsilon [\nabla\log{\mathcal{P}(z_{t,i}|\mathcal{D})} - \nabla\log{\mathcal{Q}_t(z_{t,i})}],
\end{equation}
where $\mathcal{Q}_t$ is the empirical distribution of the particles $\{z_{t,i}\}_{i=1}^{\ell}$ at step $t$. This iterative process progressively transforms the particle distribution, decreasing the KL divergence until, at $t=\mathrm{T}$, the resulting density $\mathcal{Q}_{\mathrm{T}}(z)$ is sufficiently close to the target distribution $\mathcal{P}(z\vert \mathcal{D})$.
\iffalse
So far, we can initialize a group of $\{z_{0,i}\}|_{i=0}^{\ell}\in\mathbb{R}^{\mathrm{D}_{\rm{LV}}}\overset{\mathrm{i.i.d.}}{\sim}\mathcal{Q}_0(z)$, and apply~\Cref{eq:progressivelyReducingKLDivergence} recursively to $z_0$ as follows:
\begin{equation}\label{eq:realizationOfTransportationMap}
    z_{t+1} =\boldsymbol{T}(z_{t}) = z_t + \varepsilon [\nabla\log{\mathcal{P}(z_t|\mathcal{D})} - \nabla\log{\mathcal{Q}(z_t)}],
\end{equation}
we can finally reduce the KL divergence represented between $\mathcal{Q}_t(z)$ that represented by $\{z_{t,i}\in\mathbb{R}^{\mathrm{D}_{\rm{LV}}}\}|_{i=0}^{\ell}$ closed enough to the target distribution $\mathcal{P}(z\vert\mathcal{D})$.
\fi
\subsection{Proximal Gradient Recursion within RKHS}\label{subsec:proximalImplementationRKHS}
Even though~\Cref{subsec:WassProximalOperator} has proposed the transportation map that reduces the KL divergence greedily, the implementation remains challenging. Specifically,~\Cref{eq:realizationOfTransportationMap} requires the estimation of $\nabla{\log{\mathcal{Q}_t(z)}}$, which is intractable when we represent $\mathcal{Q}_t(z)$ by a group of $\mathrm{D}_{\rm{LV}}$-dimensional particles $\{z_{t,i}\in\mathbb{R}^{\mathrm{D}_{\rm{LV}}}\}|_{i=0}^{\ell}$. To approximate the intractable term $\nabla \log \mathcal{Q}_t(z)$, we introduce a ``test function'' $h(z)$ that minimizes the weighted squared error:
\begin{equation}
\mathop{\arg\min}_{h} ~~ \frac{1}{2}  \int{\mathcal{Q}_t(z) \Vert \nabla\log{\mathcal{Q}_t(z)} - h(z) \Vert_2^2 \mathrm{d}z},
%  \frac{1}{2}  \int{\mathcal{Q}(z) \Vert \nabla\log{\mathcal{Q}_t(z)} - h(z) \Vert_2^2 \mathrm{d}z} = 0,
\end{equation}
where $\mathcal{Q}_t(z)$ serves as a weighting function. This choice of $h(z)$ allows us to focus on accurately approximating $\nabla \log \mathcal{Q}_t(z)$ in regions where $\mathcal{Q}_t(z)$ has high probability mass. % By using this approximation, we can rewrite Equation (X) as follows:
%On this basis, to approximate the intractable term $\nabla\log{\mathcal{Q}_t(z)}$, we introduce the approximation function $h(z)$ as follows:
% \begin{equation}
%     \mathop{\arg\min}_{h(z)}~~\int{\mathcal{Q}_t(z)\Vert h(z) - \nabla\log{\mathcal{Q}_t(z)} \Vert_2^2\mathrm{d}z} .
% \end{equation}
Building on this, we impose the following assumptions, which can be satisfied by appropriately choosing the function class for $h(z)$ and specifying $\mathcal{Q}(z)$ within the Wasserstein space:
\begin{enumerate}[leftmargin=*]
    \item{Test function $h(z)$ is compactly supported on $\mathbb{R}^{\mathrm{D}_{\rm{LV}}}$: there exists a radius $\mathscr{R} > 0$ such that $h(z) = 0$ for all $\|z\|  > \mathscr{R} $;}
\item{PDF $\mathcal{Q}(z)$ is bounded:  $\lim_{\|z\| \to \infty}\mathcal{Q}(z) = 0$.}
\end{enumerate}
We have the following objective for finding the optimal $h(z)$:
\begin{equation}\label{eq:unconstraintOptimizationProblem}
\begin{aligned}
       \mathop{\arg\min}_{h}~~\frac{1}{2}\int{\mathcal{Q}_t(z)\Vert h(z) - \nabla\log{\mathcal{Q}_t(z)} \Vert_2^2\mathrm{d}z} 
  %   = & \mathop{\arg\min}_{h(z)}~~\int{\mathcal{Q}_t(z) [\frac{1}{2}h^\top(z)h(z) - h^\top(z)\nabla\log{\mathcal{Q}_t(z)}] \mathrm{d}z} \\
       % \\
       \overset{\text{(i)}}{=}  \mathop{\arg\min}_{h}~~\int{\mathcal{Q}_t(z) [\frac{1}{2}h^\top(z)h(z) + \nabla \cdot h(z)] \mathrm{d}z} ,
\end{aligned}
\end{equation}
where `(i)' is based on the integration-by-parts. Specifically, it can be observed that:
\begin{equation}
\begin{aligned}
    \int [\nabla\cdot h(z)] \mathcal{Q}_t(z)\mathrm{d}z +  \int [h(z) ]^\top[\nabla\mathcal{Q}_t(z)]\mathrm{d}z  
   =  \int \nabla\cdot[\mathcal{Q}_t(z)h(z)]\mathrm{d}z.
\end{aligned}
\end{equation}
Using the Gauss divergence theorem, we get:
\begin{equation}\label{eq:divTheoremResult}
    \int{\nabla\cdot{h(z)\mathcal{Q}(z)}\mathrm{d}z}=\oint_{\partial z}{h(z)\mathcal{Q}(z)\cdot\vec{n}(z)\mathrm{d}S(z)}=0,
\end{equation}
where $\vec{n}(z)$ and $\mathrm{d}S(z)$ represent the outer normal vector and surface element, respectively, and the last equality is based on the assumptions 1) and 2) of $h(z)$ and $\mathcal{Q}(z)$, which gives the following result $\lim_{\Vert z \Vert\to\infty}h(z)\mathcal{Q}(z)=0$.

It is worth noting that directly minimizing $h(z)$ in \Cref{eq:unconstraintOptimizationProblem} over the Wasserstein space is generally ill-posed, since multiple candidate functions can satisfy the same objective. Nevertheless, an RKHS-based parameterization provides a practical and well-defined solution~\cite{haoWangCausalBalancing,lidebiased}, albeit at the cost of reduced expressiveness in high-dimensional settings~\cite{scholkopf2002learning,dongparticle,wang2024gad}. To obtain a well-posed and computationally tractable formulation, we introduce the RKHS regularization on the test function. Specifically, let the approximation function $h(z)$ be confined to the $\mathrm{D}_{\rm{LV}}$-dimensional RKHS $\mathcal{H}^{\mathrm{D}_{\rm{LV}}}$, i.e., $h(z) \in \mathcal{H}^{\mathrm{D}_{\rm{LV}}}$, where the corresponding kernel function $\mathcal{K}: \mathbb{R}^{\mathrm{D}_{\rm{LV}}} \to \mathbb{R}^{\mathrm{D}_{\rm{LV}}}$ satisfies the boundary condition $\lim_{\| z \| \to \infty} \mathcal{K}(z', z) = 0$. As such, we replace the $\mathcal{Q}_t(z)$-weighted term $\frac{1}{2}\mathbb{E}_{\mathcal{Q}_t(z)}[\|h(z)\|_2^2]$
by the penalty term $\frac{1}{2}\|h\|_{\mathcal{H}^{\mathrm{D}_{\rm LV}}}^2$, which controls the smoothness of $h$. Consequently, we have the following objective for $h(z)$:
% we can first reformulate the learning objective as follows using the fact that RKHS norm is independent of the data distribution $\mathcal{Q}_t(z)$:
\begin{equation}\label{eq:SimilarityobjKernelFunction}
   \mathop{\arg\min}_{h}~~ \frac{1}{2} \Vert h(z)\Vert_{\mathcal{H}^{\mathrm{D}_{\rm{LV}}}}^2 + \mathbb{E}_{\mathcal{Q}_t(z)}{[\nabla \cdot h(z)]}.
\end{equation}

Notably, for kernel function $\mathcal{K}(z',z)$, we can define its feature map $\xi:\mathbb{R}^{\mathrm{D}_{\rm{LV}}}:\to \mathcal{H}$, and decompose the kernel function as $\mathcal{K}(z',z)=\left\langle \xi(z'),\xi(z) \right\rangle _{\mathcal{H}^{\mathrm{D}}}$. Based on this, we can apply the following spectral decomposition (where $\Lambda_i$ and $\Xi_i:\mathbb{R}^{\mathrm{D}_{\rm{LV}}}\to\mathbb{R}$ are the eigen value and orthonormal basis, respectively) $\mathcal{K}(z',z) = \sum_{i=1}^{\infty}{\Lambda_i\Xi_i(z')\Xi_i(z)}$, to test function $h(z)\in\mathcal{H}^{\mathrm{D}_{\rm{LV}}}$ as $h(z) = \sum_{i=1}^{\infty}{\widehat{h}_i\sqrt{\Lambda_i}\Xi_i(z)}$,
% \begin{equation}\label{eq:kernelDecompositionEquation}
%     h(z) = \sum_{i=1}^{\infty}{\widehat{h}_i\sqrt{\Lambda_i}\Xi_i(z)},
% \end{equation}
where feature importance weight $\widehat{h}_i \in\mathbb{R}^{\mathrm{D}_{\rm{LV}}}$, and $\sum_{i=1}^{\infty}{\Vert \widehat{h}_i \Vert_2^2}\le\infty$. Consequently,~\Cref{eq:SimilarityobjKernelFunction} can be reformulated as follows:
\begin{equation}\label{eq:optimalPsiKernelFunction}
    \widehat{h}_i^* = - \sqrt{\Lambda_i}\mathbb{E}_{\mathcal{Q}_t(z')}\left[ \nabla_{z'}\Xi(z')  \right].
\end{equation}
The optimal test function can be given as follows:
\begin{equation}\label{eq:RKHSHeatEqautionResult}
\begin{aligned}
   h_{\text{RKHS}}(z) 
 % = & \sum_{i=1}^{\infty}{ \sqrt{\Lambda_i}\mathbb{E}_{\mathcal{Q}_t(z')}\left[ \nabla_{z'}\Xi(z') + \nabla_{z'}\log{\mathcal{P}(z'\vert x)\Xi(z')} \right]\sqrt{\Lambda_i}\Xi_i(z)}\\
= -\sum_{i=1}^{\infty}{\widehat{h}_i^*\sqrt{\Lambda_i}\Xi_i(z)} = - \mathbb{E}_{\mathcal{Q}_t(z')}[\nabla_{z'}\mathcal{K}(z',z) ].
\end{aligned}
\end{equation}

Plugging~\Cref{eq:RKHSHeatEqautionResult} into~\Cref{eq:realizationOfTransportationMap}, we obtain the following result to implement the optimization problem:
\begin{equation}\label{eq:theTransportationInRKHS}
      z_{t+1} = z_t + \varepsilon \{\nabla\log{\mathcal{P}(z_t|\mathcal{D})} +\mathbb{E}_{\mathcal{Q}_t(z')}[\nabla_{z'}\mathcal{K}(z',z) ]\},
\end{equation}
To fulfill the compact support requirement of the test function $h(z)$, we use the radial basis function (RBF) as the kernel function for model implementation:
\begin{equation}
    \mathcal{K}(z,z')  \coloneqq \exp(-\frac{\Vert z -z' \Vert_2^2}{2 }),
\end{equation}
where the score function $\nabla\log{\mathcal{P}(z|\mathcal{D})}$ can be reformulated as follows using the Bayesian rule:
\begin{equation}\label{eq:scoreFunctionDecomposition}
    \nabla\log{\mathcal{P}(z\vert\mathcal{D})} = \nabla \log{p_\theta(\mathcal{D}|z) + \nabla\log{\mathcal{P}(z)}},
\end{equation} 
the values for $z$ and $z'$ are identity, and the prime `$\prime$'  are designed to mark which variable the derivative is taken.
%In addition, 

On this basis, the following algorithm for NPLVM learning can be summarized in Algorithm~\ref{algo:kProxVIProcedure} . Since the derivation of the latent variable distribution inference algorithm is based on the \uline{k}ernelized \uline{prox}imal gradient descent-based approach, we termed our algorithm KProx algorithm.
\begin{algorithm}[htbp]% \label{algo:odeSimulationOptimalControl}
\caption{KProx Algorithm for $\mathcal{P}(z|x)$ Inference.}\label{algo:kProxVIProcedure}
\begin{algorithmic}[1]
\State \textbf{Input:} 
Latent Variable Model: $p_\theta(\mathcal{D}\vert z)$, prior distribution $\mathcal{P}(z)$, end time: $\mathrm{T}$, proximal operator coefficient: $\varepsilon$. % , learning rate for $\psi_\varphi(z)$ optimization: $\eta$, and iteration time for $\psi_\varphi(z)$ optimization: $\widehat{\mathrm{T}}$.
% Train Dataset: $\{x_{b}, y_{b}\}\vert_{b=1}^{\mathrm{N}_{\text{train}}}{\in\mathcal{D}_{\text{train}}}$, Validate Dataset: $\{x_{b}, y_{b}\}\vert_{b=1}^{\mathrm{N}_{\text{valid}}}{\in\mathcal{D}_{\text{valid}}}$, Test dataset: $\{x_{b}, y_{b}\}\vert_{b=1}^{\mathrm{N}_{\text{test}}}{\in\mathcal{D}_{\text{test}}}$.
\State $z_{i,t} \overset{\mathrm{i.i.d.}}{\sim} \mathcal{Q}_0(z)${\Comment{Initialize $\{z_{i,0}\}_{i=1}^{\ell}$}}
% \State \textbf{Parameter:} parameter for $\psi_\varphi(z)$: $\varphi$.
% \State $\varphi_0 \leftarrow \varphi$
\For{$t = 0$ \textbf{to} $\mathrm{T}-1$}
\State $\nabla\log{\mathcal{P}(z_{t+1}\vert \mathcal{D})}\leftarrow\text{Eq.}~\eqref{eq:scoreFunctionDecomposition}$\;
\State $z_{i,t+1}\leftarrow \text{Eq.}~\eqref{eq:theTransportationInRKHS}$
\EndFor

\State \textbf{Output:} $\{z_{i,\mathrm{T}}\}_{i=1}^{\ell}$\;

\end{algorithmic}
\end{algorithm}

Based on~\Cref{algo:kProxVIProcedure}, we have the following theorem to demonstrate the convergence speed of the KL divergence as the iteration time of the KProx algorithm increases:
\begin{theorem} \label{thm:ConvergenceControlEpsilon}
Suppose that $ \Vert v(z)\Vert\le \mathscr{A}  $ and $ \Vert \nabla v(z)\Vert\le \mathscr{B}  $. 
Let $\{ \mathcal{Q}_t(z)\}_{t=1}^{\mathrm{T}}$ denote the sequence of variational distributions generated by the KProx algorithm. Then, when $\varepsilon=\frac{1}{\sqrt{\mathrm{T}}} $, the following inequality holds:
\begin{equation}\label{eq:convergenceTheorem}
\begin{aligned} 
\lim_{\mathrm{T}\to\infty}\mathbb{D}_{\rm}[\mathcal{Q}_\mathrm{T}(z)\Vert\mathcal{P}(z\vert\mathcal{D})] = 0
 % & \dfrac{1}{{\mathrm{T}}}\sum_{t=1}^{\mathrm{T}}\mathbb{D}_{\rm{KL}}[\mathcal{Q}_t(z)\Vert \mathcal{P}(z|x)]
 % %   \le &  \frac{1}{\sqrt{\mathrm{T}}}\{ \mathbb{D}_{\rm{KL}}[\mathcal{Q}_{t+\varepsilon}(z)\Vert \mathcal{P}(z|x)]- \mathbb{D}_{\rm{KL}}[\mathcal{Q}_{t}(z)\Vert \mathcal{P}(z|x)]\} + \frac{2\mathcal{C}}{\sqrt{\mathrm{T}}}\\
 %    \leq
 %    \mathcal{O}(\frac{1}{\sqrt{\mathrm{T}}}).
\end{aligned}   
\end{equation}
\end{theorem} 
Due to the page limit, we mainly provide proof of sketch in the main content.
\begin{proof}
The KL divergence at $t+1$ has the following relationship with the KL divergence at $t$:
\begin{equation}\label{eq:expansionKPlus1Time}
\begin{aligned}
& \mathbb{D}_{\rm{KL}}[\mathcal{Q}_{t+1}(z)\Vert \mathcal{P}(z|\mathcal{D})]= \mathbb{D}_{\rm{KL}}[\mathcal{Q}_{t}(z)\Vert \mathcal{P}(z|\mathcal{D})] \\
 &  -\varepsilon \mathbb{E}_{\mathcal{Q}_t(z)}\{v^\top(z) \nabla\frac{\delta \mathbb{D}_{\rm{KL}}[\mathcal{Q}_t(z)\Vert \mathcal{P}(z\vert\mathcal{D})]}{\delta \mathcal{Q}_t(z)} \} + \mathcal{O}(\varepsilon^2)\\
 \le & \mathbb{D}_{\rm{KL}}[\mathcal{Q}_{t}(z)\Vert \mathcal{P}(z|\mathcal{D})] \\
& -\varepsilon \mathbb{E}_{\mathcal{Q}_t(z)}\{v^\top(z) \nabla\frac{\delta \mathbb{D}_{\rm{KL}}[\mathcal{Q}_t(z)\Vert \mathcal{P}(z\vert\mathcal{D})]}{\delta \mathcal{Q}_t(z)} \} + \mathcal{C}\varepsilon^2.
\end{aligned}
\end{equation}
where the last inequality is based on the fact that $ \Vert v(z)\Vert\le \mathscr{A}  $ and $ \Vert \nabla v(z)\Vert\le \mathscr{B}  $, which indicates that there exists a positive constant $\mathcal{C}$ that bound the $\mathcal{O}(\varepsilon^2)$.
% \begin{equation}\label{eq:inequalityResult}
% \begin{aligned}
%     &\mathbb{D}_{\rm{KL}}[\mathcal{Q}_{t}(z)\Vert \mathcal{P}(z|\mathcal{D})]  
%     \\
%     &-\varepsilon \mathbb{E}_{\mathcal{Q}_t(z)}\{v^\top(z) \nabla\frac{\delta \mathbb{D}_{\rm{KL}}[\mathcal{Q}_t(z)\Vert \mathcal{P}(z\vert\mathcal{D})]}{\delta \mathcal{Q}_t(z)} \} + \mathcal{O}(\varepsilon^2)\\
% \le & \mathbb{D}_{\rm{KL}}[\mathcal{Q}_{t}(z)\Vert \mathcal{P}(z|\mathcal{D})] \\
% & -\varepsilon \mathbb{E}_{\mathcal{Q}_t(z)}\{v^\top(z) \nabla\frac{\delta \mathbb{D}_{\rm{KL}}[\mathcal{Q}_t(z)\Vert \mathcal{P}(z\vert\mathcal{D})]}{\delta \mathcal{Q}_t(z)} \} + \mathcal{C}\varepsilon^2.
% \end{aligned}
% \end{equation}
In addition, since $h_{\text{RKHS}}(z)$ is the weak mode of convergence for $\nabla\log{\mathcal{Q}_t(z)}$, it can be observed that:
\begin{equation}\label{eq:innerProductInequality}
\begin{aligned}
  &  %\int{
 \mathbb{E}_{\mathcal{Q}_t(z)}\{[ \nabla\frac{\delta \mathbb{D}_{\rm{KL}}[\mathcal{Q}_t(z)\Vert \mathcal{P}(z\vert\mathcal{D})]}{\delta \mathcal{Q}_t(z)}]^\top [ \nabla\frac{\delta \mathbb{D}_{\rm{KL}}[\mathcal{Q}_t(z)\Vert \mathcal{P}(z\vert\mathcal{D})]}{\delta \mathcal{Q}_t(z)}]\}
  %\mathrm{d}z} 
  \\
&\simeq  \\
&  \mathbb{E}_{\mathcal{Q}_t(z)}\{[ \nabla\frac{\delta \mathbb{D}_{\rm{KL}}[\mathcal{Q}_t(z)\Vert \mathcal{P}(z\vert\mathcal{D})]}{\delta \mathcal{Q}_t(z)}]^\top [ \nabla\log{\mathcal{P}(z\vert\mathcal{D})} - h_{\text{RKHS}}(z)]\}
%\mathrm{d}z}
.
\end{aligned}
\end{equation}
% $v^\top\nabla\frac{\delta \mathbb{D}_{\rm{KL}}[\mathcal{Q}_t(z)\Vert \mathcal{P}(z\vert\mathcal{D})]}{\delta \mathcal{Q}_t(z)}$ 
% since we restrict the $h(z)$ within RKHS, we get the following inequality from the perspective of inner product similarity:
% \begin{equation}\label{eq:innerProductInequality}
% \begin{aligned}
%   &  [ \nabla\frac{\delta \mathbb{D}_{\rm{KL}}[\mathcal{Q}_t(z)\Vert \mathcal{P}(z\vert\mathcal{D})]}{\delta \mathcal{Q}_t(z)}]^\top [ \nabla\frac{\delta \mathbb{D}_{\rm{KL}}[\mathcal{Q}_t(z)\Vert \mathcal{P}(z\vert\mathcal{D})]}{\delta \mathcal{Q}_t(z)}] \\
% \ge  & [ \nabla\frac{\delta \mathbb{D}_{\rm{KL}}[\mathcal{Q}_t(z)\Vert \mathcal{P}(z\vert\mathcal{D})]}{\delta \mathcal{Q}_t(z)}]^\top [ \nabla\log{\mathcal{P}(z\vert\mathcal{D})} - h_{\text{RKHS}}(z)].
% \end{aligned}
% \end{equation}
Plugging~\Cref{eq:innerProductInequality} into~\Cref{eq:expansionKPlus1Time}, the following result can be obtained:
\begin{equation}\label{eq:inequalityResult2}
\begin{aligned}
    &\mathbb{D}_{\rm{KL}}[\mathcal{Q}_{t+1}(z)\Vert \mathcal{P}(z|\mathcal{D})] \le  \mathbb{D}_{\rm{KL}}[\mathcal{Q}_{t}(z)\Vert \mathcal{P}(z|\mathcal{D})]   
  %  &-\varepsilon \mathbb{E}_{\mathcal{Q}_t(z)}\{v^\top(z) \nabla\frac{\delta \mathbb{D}_{\rm{KL}}[\mathcal{Q}_t(z)\Vert \mathcal{P}(z\vert\mathcal{D})]}{\delta \mathcal{Q}_t(z)} \} + \mathcal{C}\varepsilon^2\\
 -\varepsilon \mathbb{E}_{\mathcal{Q}_t(z)}\{\Vert  \nabla\frac{\delta \mathbb{D}_{\rm{KL}}[\mathcal{Q}_t(z)\Vert \mathcal{P}(z\vert\mathcal{D})]}{\delta \mathcal{Q}_t(z)} \Vert_2^2 \} + \mathcal{C}\varepsilon^2.
\end{aligned}
\end{equation}
Cascading~\Cref{eq:inequalityResult2} from $t=1$ to $\mathrm{T}$, we get:
\begin{equation}
\begin{aligned}
   & \mathbb{D}_{\rm{KL}}[\mathcal{Q}_{\mathrm{T}}(z)\Vert \mathcal{P}(z|\mathcal{D})]   \le  \mathbb{D}_{\rm{KL}}[\mathcal{Q}_{0}(z)\Vert \mathcal{P}(z|\mathcal{D})]  -\varepsilon \sum_{t=1}^{\mathrm{T}}{\mathbb{E}_{\mathcal{Q}_t(z)}\{\Vert  \nabla\frac{\delta \mathbb{D}_{\rm{KL}}[\mathcal{Q}_t(z)\Vert \mathcal{P}(z\vert\mathcal{D})]}{\delta \mathcal{Q}_t(z)} \Vert_2^2 \} } + 2\mathrm{T}\mathcal{C}\varepsilon^2\\
    \Rightarrow & \frac{1}{\mathrm{T}}\sum_{t=1}^{\mathrm{T}}{\mathbb{E}_{\mathcal{Q}_t(z)}\{\Vert  \nabla\frac{\delta \mathbb{D}_{\rm{KL}}[\mathcal{Q}_t(z)\Vert \mathcal{P}(z\vert\mathcal{D})]}{\delta \mathcal{Q}_t(z)} \Vert_2^2 \} }   \le \underbrace{\frac{\mathbb{D}_{\rm{KL}}[\mathcal{Q}_0(z)\Vert \mathcal{P}(z\vert\mathcal{D})] - \mathbb{D}_{\rm{KL}}[\mathcal{Q}_{\mathrm{T}}(z)\Vert \mathcal{P}(z\vert\mathcal{D})]}{\sqrt{\mathrm{T}}} + \frac{2\mathcal{C}}{\sqrt{\mathrm{T}}}}_{\mathcal{O}(\frac{1}{\sqrt{\mathrm{T
    }}})}. %\\
    %& .
\end{aligned}
\end{equation}
This phenomenon reflects that:
\begin{equation}
    \lim_{\mathrm{T}\to\infty}  \nabla \log{\mathcal{Q}_{\mathrm{T}}(z)} -\nabla\log{\mathcal{P}(z\vert\mathcal{D})} = 0,
\end{equation}
Integrating both side with $z$, we get:
\begin{equation}
    \mathcal{Q}_{\mathrm{T}}(z) = \mathcal{C}\mathcal{P}(z\vert\mathcal{D})~~\text{when}~~\mathrm{T}\to\infty \Rightarrow \lim_{\mathrm{T}\to\infty}\mathcal{Q}_{\mathrm{T}}(z)\propto\mathcal{P}(z\vert\mathcal{D}).
\end{equation}
Since $\lim_{\mathrm{T}\to\infty}\mathcal{Q}_{\mathrm{T}}(z)\propto\mathcal{P}(z\vert\mathcal{D})$, we arrive at the desired result defined by~\Cref{eq:convergenceTheorem}.
\end{proof}
Based on this theorem, the following remark can be given:
\begin{remark}
In contrast to conventional variational inference methods, which constrain the variational distribution to a restricted, finite-dimensional space and are thus subject to an approximation error as quantified by Lemma~\ref{thm:klDivergenceBound}, the proposed KProx algorithm offers the potential to mitigate this error. This reduction in approximation error is achievable through judicious selection of the proximal operator coefficient, allowing for a more accurate representation of the true posterior.
\end{remark}
\begin{remark}
\Cref{thm:ConvergenceControlEpsilon} provides a guideline for choosing $\varepsilon$. Specifically, to ensure the KL-divergence term decreases monotonically and to guarantee convergence, we recommend setting $\varepsilon = \frac{1}{\sqrt{\mathrm{T}}}$.
\end{remark}

\subsection{Parameter Learning for Networks }\label{subsec:ModelTrainingStrategy}
Even though~\Cref{subsec:proximalImplementationRKHS} provides the inference procedure for the distribution of latent variable, the parameter learning procedure for the generative network $\theta$ and inference network $\varphi$ has not been derived yet. Hence, the rest of this subsection will focus on deriving the learning procedure of these two parameters $\theta$ and $\varphi$.

\subsubsection{Generative Network Parameter Learning} To learn the generative network parameter, we define the following result based on $\{z_{i,\mathrm{T}}\}_{i=1}^{\ell}$ at time $\mathrm{T}$. Specifically, the objective function for generative network learning can be given as:
\begin{equation}\label{eq:objectiveForGenerativeNetwork}
    \mathop{\arg\max}_{\theta}~~\mathbb{E}_{\mathcal{Q}_{\mathrm{T}}(z)}[\log{p_\theta(\mathcal{D}|z)}].
\end{equation}
Notably, $\mathcal{Q}_{\mathrm{T}}(z)$ is represented by a group of particles $\{z_{i,\mathrm{T}}\}_{i=1}^{\ell}$. Thus, the learning objective of~\Cref{eq:objectiveForGenerativeNetwork} can be given as follows using the selective property of the Dirac delta measure:
\begin{equation}\label{eq:decoderParameterLearning}\mathop{\arg\max}_{\theta}~\mathbb{E}_{\mathcal{Q}_{\mathrm{T}}(z)}[\log{p_\theta(\mathcal{D}|z)}]\approx  \mathop{\arg\max}_{\theta}~ \frac{1}{\ell}\sum_{i=1}^{\ell}{\log{p_\theta}(\mathcal{D}|z_{i,\mathrm{T}})}.
\end{equation}
Applying gradient descent with learning rate $\varepsilon$ to the right-hand-side, the parameter learning procedure for $\theta$ can be therefore obtained as follows:
\begin{equation}
    \theta =\theta -\varepsilon  \frac{1}{\ell}\sum_{i=1}^{\ell}{\nabla\log{p_\theta}(\mathcal{D}|z_{i,\mathrm{T}})}
\end{equation}
\subsubsection{Inference Network Parameter Learning}
Notably, the parameter learning procedure for the inference network learning remains great challenges since we are merely available a group of particles $\{z_{i,\mathrm{T}}\}_{i=1}^{\ell}$ that represents the distribution $\mathcal{Q}_{\mathrm{T}}(z)$. Denote the predicted latent variable as $\{\widehat{z}_i\}_{i=1}^{\ell}$ Based on~\Cref{subsec:WassDistanceAndWassSpace}, we can introduce the 2-Wasserstein distance as the discrepancy metric to measure the differences between $q_\varphi(z|x)$ and $\mathcal{Q}_{\mathrm{T}}(z)$. On this basis, the loss function for inference network training can be given as follows:
\begin{equation}\label{eq:learningObjectiveInfNetwork}
\begin{aligned}
   \mathop{\arg\min}_{\varphi}~~ \mathcal{W}_2^2(q_\varphi(z|x),\mathcal{Q}_{\mathrm{T}}(z)) = \inf_{\pi \geq 0}~\sum_{i=1}^\ell \sum_{j=1}^\ell \pi_{ij} \|z_{i,\mathrm{T}} - \widehat{z}_j\|_2^2 \quad 
\mathrm{s.t.}  \sum_{j=1}^\ell \pi_{ij} = \frac{1}{\ell} ;\sum_{i=1}^\ell \pi_{ij} = \frac{1}{\ell}
.%\begin{cases} 
%
%\end{cases}
\end{aligned}
\end{equation}
In particular, applying the gradient descent-like neural network update procedure to~\Cref{eq:learningObjectiveInfNetwork} is difficult. Specifically, the optimization of inference network $q_\varphi(z|x)$ requires the gradient as follows:  
\begin{equation}\label{eq:wassDistanceGradResult}
    \nabla_\varphi \mathcal{W}_2^2(q_\varphi(z|x),\mathcal{Q}_{\mathrm{T}}(z)) =[\frac{ \partial  \mathcal{W}_2^2(q_\varphi(z|x),\mathcal{Q}_{\mathrm{T}}(z))}{\partial \widehat{z}}]^\top[ \frac{ \partial \widehat{z}}{\partial \varphi}],
\end{equation}
where $\frac{ \partial  \mathcal{W}_2^2(q_\varphi(z|x),\mathcal{Q}_{\mathrm{T}}(z))}{\partial \widehat{z}}$ is intractable due to the existence of the infimum operator ``\texttt{inf}''. To this end, we should solve the following optimal transportation problem to obtain the expression to facilitate the gradient backpropagation process. 
\begin{equation}
\begin{aligned}
     \mathscr{L}(\pi, \upmu, \upnu; \{z_{i,\mathrm{T}}\}, \{\widehat{z}_j\}) 
    % = &  \sum_{i=1}^\ell \sum_{j=1}^\ell \pi_{ij} \|z_{i,\mathrm{T}} - \widehat{z}_j\|_2^2 - \sum_{i=1}^\ell \upmu_i [\sum_{j=1}^\ell \pi_{ij} - \frac{1}{\ell}] \\
    % &- \sum_{j=1}^\ell \upnu_j [\sum_{i=1}^\ell \pi_{ij} - \frac{1}{\ell}] \\
=  \sum_{i=1}^\ell \sum_{j=1}^\ell \pi_{ij} (\|z_{i,\mathrm{T}} - \widehat{z}_j\|_2^2 - \upmu_i - \upnu_j) + \sum_{i=1}^\ell \frac{\upmu_i}{\ell} + \sum_{j=1}^\ell \frac{\upnu_j}{\ell},
\end{aligned}
\end{equation}
where $\{\upmu_i\}_{i=1}^{\ell}$ and $\{\upnu_j\}_{j=1}^{\ell}$ are Lagrange multipliers to handle the equality constraints.

According to the envelope theorem~\cite{border2015miscellaneous}, the gradient of the value function $\mathcal{W}_2^2(q_\varphi(z|x),\mathcal{Q}_{\mathrm{T}}(z))$ with respect to $\widehat{z}_j$ equals the partial derivative of the Lagrangian with respect to $\widehat{z}_j$, evaluated at the optimal solution:
\begin{equation}
    \frac{\partial \mathcal{W}_2^2(q_\varphi(z|x),\mathcal{Q}_{\mathrm{T}}(z)) }{\partial \widehat{z}_j} = \frac{\partial \mathscr{L}}{\partial \widehat{z}_j}\bigg|_{(\pi^*, \upmu^*, \upnu^*)}.
\end{equation}
Note that, the Lagrangian function, $\widehat{z}_j$ only appears in the cost terms $\|z_{i,\mathrm{T}} - \widehat{z}_j\|_2^2$. Thus, for fixed $j$, we have:
%, and only when the second index equals $j$. For fixed $j$, we have:
\begin{equation}\frac{\partial}{\partial \widehat{z}_j} \|z_{i,\mathrm{T}} - \widehat{z}_j\|_2^2 = \frac{\partial}{\partial \widehat{z}_j} \sum_{k=1}^{\mathrm{D}_{\rm{LV}}} (z_{i,\mathrm{T}}^{k} - \widehat{z}_j^{k})^2 = -2(z_{i,\mathrm{T}} - \widehat{z}_j)\end{equation}
In other words:
\begin{equation}
\frac{\partial \mathscr{L}}{\partial \widehat{z}_j} = \sum_{i=1}^\ell \pi_{ij} \cdot \frac{\partial}{\partial \widehat{z}_j} \|z_{i,\mathrm{T}} - \widehat{z}_j\|_2^2 = -2\sum_{i=1}^\ell \pi_{ij}(z_{i,\mathrm{T}} - \widehat{z}_j).
\end{equation}
Consequently, once we get the optimal transportation plan, we can directly obtain the gradient with-respect-to $\widehat{z}$ and conduct backpropagation easily.

Based on this, the key to obtaining the optimal transportation map is the key to conducting the training of the inference network. To this end, we consider using the Sinkhorn-Knopp iteration~\cite{cuturi2013sinkhorn}, where the entropy term about the transportation plan is selected as the proximal operator for the optimization problem. Specifically, we have the following optimization problem:
\begin{equation}
\begin{aligned}
  \mathcal{W}_{\text{Sink}}^2(z, \widehat{z}) = \min_{\pi \in \Pi(z, \widehat{z})} \sum_{i=1}^\ell \sum_{j=1}^\ell  \pi_{ij} \|z_{i,\mathrm{T}} - \widehat{z}_j\|_2^2 
  - \upepsilon  \sum_{i,j} \pi_{ij}(\log \pi_{ij} - 1),
\end{aligned}
\end{equation}
and the Lagrangian function can be obtained as follows:
\begin{equation}
\begin{aligned}
    \mathscr{L}_{\text{Sink}} =& \sum_{i=1}^\ell \sum_{j=1}^\ell \pi_{ij} (\|z_{i,\mathrm{T}} - \widehat{z}_j\|_2^2  
     + \upepsilon(\log \pi_{ij} - 1) - \upmu_i - \upnu_j) + \sum_{i=1}^\ell \frac{\upmu_i}{\ell} + \sum_{j=1}^\ell \frac{\upnu_j}{\ell}
\end{aligned}
\end{equation}
Taking the derivative with-respect-to $\pi_{ij}$ and setting the  derivative to zero, we get the following result:
\begin{equation}
  \frac{\partial  \mathscr{L}_{\text{Sink}}}{\partial \pi_{ij}} = \|z_{i,\mathrm{T}} - \widehat{z}_j\|_2^2 + \upepsilon \log \pi_{ij} - \upmu_i - \upnu_j = 0.
\end{equation}
Thus, the optimal coupling $\pi^*$ satisfies the following structure:
\begin{equation}\label{eq:sinkhornIterationResult}
\pi_{ij}^* = \exp(\frac{\upmu_i + \upnu_j - \|z_{i,\mathrm{T}} - \hat{z}_j\|_2^2}{\upepsilon}) = \tilde{\upmu}_i \tilde{\upnu}_j \mathscr{K}_{i,j}
\end{equation}
where we define $\tilde{\upmu}_i \coloneqq \exp(\dfrac{\upmu_i}{\upepsilon})$, $\tilde{\upnu}_j \coloneqq \exp(\dfrac{\upnu_j}{\upepsilon})$, and $\mathscr{K}_{i,j} \coloneqq \exp(-\dfrac{\|z_{i,\mathrm{T}} - \hat{z}_j\|_2^2}{\upepsilon})$. Consequently, based on~\Cref{eq:sinkhornIterationResult}, the overall iteration process for the Sinkhorn iteration to obtain the optimal coupling $\pi^*$ is summarized in Algorithm~\ref{algo:sinkhorn}.

\begin{algorithm}[htbp]
\caption{Sinkhorn-Knopp Algorithm}\label{algo:sinkhorn}
\begin{algorithmic}[1]
\State \textbf{Input:} 
Samples: $\{z_{i,\mathrm{T}}\}_{i=1}^\ell$ and $\{\hat{z}_j\}_{j=1}^\ell$, regularization parameter: $\upepsilon$, and iteration time: $\mathrm{T}$.

\State Initialize $\tilde{\upmu}_i^{(0)} \leftarrow  \frac{1}{\ell},\forall i = 1, \ldots, \ell$ %\Comment{Initialize row scaling}
\State Initialize $\tilde{\upnu}_j^{(0)} \leftarrow  \frac{1}{n},\forall j = 1, \ldots, \ell$ %\Comment{Initialize column scaling}
\State  $\mathscr{K}_{ij} \leftarrow \exp(-\frac{\|z_i - \hat{z}_j\|_2^2}{\upepsilon})\forall i=1,\ldots,\ell;j=1,\ldots,\ell.$

\For{$t = 0$ \textbf{to} $\mathrm{T}-1$}
  %  \State \textbf{// Update row scaling}
    \For{$i = 1$ \textbf{to} $\ell$}
        \State $\tilde{\upmu}_i^{(t+1)} \leftarrow \frac{1}{\ell\times \sum_{j=1}^\ell \tilde{\upnu}_j^{(t)} \mathscr{K}_{ij}}$
    \EndFor
    
   % \State \textbf{// Update column scaling}
    \For{$j = 1$ \textbf{to} $\ell$}
        \State $\tilde{\upnu}_j^{(t+1)} \leftarrow \frac{1}{\ell\times\sum_{i=1}^m \tilde{\upmu}_i^{(t+1)} \mathscr{K}_{ij}}$
    \EndFor
   % \State \textbf{// Check convergence}
    % \If{$\|\tilde{\upmu}^{(t+1)} - \tilde{\upmu}^{(t)}\|_\infty < \text{tol}$ and $\|\tilde{\upnu}^{(t+1)} - \tilde{\upnu}^{(t)}\|_\infty < \text{tol}$}
    %     \State \textbf{break}
   % \EndIf
\EndFor

\State Compute the optimal transport plan: $\pi_{ij}^* \leftarrow \tilde{\upmu}_i^{(\mathrm{T})} \tilde{\upnu}_j^{(\mathrm{T})} \mathscr{K}_{ij}$.

\State \textbf{Output:} Optimal transport matrix $\pi^* \in \mathbb{R}^{\ell \times \ell}$

\end{algorithmic}
\end{algorithm}

\subsection{Overall Workflow}

The overall workflow of the proposed model can be summarized in~\Cref{algo:algoTraining}, and the corresponding illustration is summarized in~\Cref{fig:modelIllustrationResult}. For simplicity, the prior distribution $\mathcal{P}(z)$ is set as the univariate Gaussian distribution $\mathcal{N}(0,I)$. Since the training process of the NPLVM relies heavily on the KProx algorithm, we name our NPLVM `KProxNPLVM'. From the upper part of~\Cref{fig:modelIllustrationResult}, it can be observed that the training of KProxNPLVM can be divided into two key steps: namely the training of the decoder $p_\theta(\mathcal{D}|z)$ and the training of the encoder $q_\varphi(z|x)$, as we demonstrated in Line 4 and Line 16 of~\Cref{algo:algoTraining}.

\begin{algorithm}[!h]
\caption{Training Algorithm of KProxNPLVM}\label{algo:algoTraining}
\begin{algorithmic}[1]
\State \textbf{Input:} 
Train, validate, and test data: $\{x_{m}, y_{m}\}_{m\in\mathcal{D}_{\text{train}}}$, $\{x_{\upnu}, y_{\upnu}\}_{v\in\mathcal{D}_{\text{valid}}}$, $\{x_{n}, y_{n}\}_{n\in\mathcal{D}_{\text{test}}}$, prior distribution $\mathcal{P}(z)\sim\mathcal{N}(0, I)$.

\State \textbf{Hyperparameters:} 
Batch size: $\mathcal{B}$, inference network learning rate: $\eta$, Epoch: $\mathcal{E}_{\text{generative}}$/$\mathcal{E}_{\text{inference}}$, particle number $\ell$, iteration time $\mathrm{T}$, proximal gradient descent coefficient: $\varepsilon$, and learning rate for inference network; $\eta$.

\State \textbf{Training:}
\State Initialize $z$ by $z\sim\mathcal{N}(0, I)$\Comment{\textbf{Step 1}: $p_\theta(\mathcal{D}\vert z)$ Training}
\For{epoch $= 1$ to $\mathcal{E}_{\text{generative}}$}
    % \State \textbf{// Inference of Latent Variable}
    %\State Set $\tau=1$
    % \For{$\tau = 1$ to $\mathrm{T}$}
       
    % \EndFor
     \State $\{z_{\mathrm{T},i}\}_{i=1}^{\ell} \leftarrow \text{Algorithm~\ref{algo:kProxVIProcedure}}$ \Comment{Latent Variable Inference}
    \State $\theta\leftarrow \text{Eq.}~\eqref{eq:decoderParameterLearning}$
\EndFor

% \State \textbf{// Obtain the Dataset for Inference Network $q_\varphi(z|x)$ Training}
\State $ \mathcal{D}_{\text{inference}} \leftarrow \varnothing$
\For{$(x_m,y_m)\in\mathcal{D}_{\text{train}}$}
%    \For{epoch $= 1$ to $\mathcal{E}_{\text{generative}}$}
          \State $\{z_{\mathrm{T},i}\}_{i=1}^{\ell} \leftarrow \text{Algorithm~\ref{algo:kProxVIProcedure}}$
        \State $\mathcal{D}_{\text{inference}}\leftarrow \mathcal{D}_{\text{inference}}\cup\{(z_m, x_m)\}$
   % \EndFor
\EndFor

% \State \textbf{// Training of $q_\varphi(z|x)$}
\State Set epoch $= 1$\Comment{\textbf{Step 2}: $q_\varphi(z|x)$ Training}
\For{epoch $= 1$ to $\mathcal{E}_{\text{inference}}$}
    \State Sample a minibatch $\mathcal{D}_{\text{minibatch}}\subset\mathcal{D}_{\text{inference}}$
    \State $\hat{z}\leftarrow q_\varphi(z|x)$
    % \State $ \mathcal{L}(z, \hat{z}) \leftarrow \dfrac{1}{\mathrm{N}_{\mathcal{B}}}\sum_{l=1}^{\mathrm{N}_{\mathcal{B}}}{\Vert \hat{z} - z_l\Vert^2}$
    \State $\pi_{i,j}^* \leftarrow \text{Eq.~\eqref{eq:sinkhornIterationResult}}$
   
    \State $\nabla_\varphi \mathcal{W}_2^2(q_\varphi(z|x),\mathcal{Q}_{\mathrm{T}}(z))\leftarrow\text{Algorithm.~\ref{algo:sinkhorn}}$
     \State $\varphi\leftarrow \varphi-\eta\times\nabla_\varphi \mathcal{W}_2^2(q_\varphi(z|x),\mathcal{Q}_{\mathrm{T}}(z))$
    \State Save the best estimated model $\hat{\varphi},\hat{\theta} = \varphi_{\text{best}},\theta_{\text{best}}$ on the $\mathcal{D}_{\text{valid}}$ with $\min \sum_{b=1}^{\mathrm{N}_{\text{valid}}}{(\hat{y}_b-y_b)^2}$
\EndFor
\end{algorithmic}
\end{algorithm}

Specifically, the training procedure of the decoder $p_\theta(\mathcal{D}|z)$ (Step 1 in \Cref{fig:modelIllustrationResult}) involves inferring the latent variable $z$ given the observed data $\mathcal{D}$ using the KProx algorithm (Algorithm~\ref{algo:kProxVIProcedure}). As depicted on the right side of \Cref{fig:modelIllustrationResult}, this step iteratively refines the variational distribution $\mathcal{Q}_t(z)$ to approximate the true posterior $p(z|\mathcal{D})$. The update rule for $\mathcal{Q}_t(z)$ is governed by the proximal gradient descent within the Wasserstein space, visualized as the "velocity field" guiding the distribution towards regions of higher posterior probability. The parameters $\theta$ of the decoder are then updated based on the inferred latent variables $z_{\mathrm{T},i}$, as shown in Eq.~\eqref{eq:decoderParameterLearning}.

Subsequently, the training of the encoder $q_\varphi(z|x)$ (Step 2 in \Cref{fig:modelIllustrationResult}) aims to learn a mapping from the observed data $x$ to the latent space $z$. As shown in the left side of \Cref{fig:modelIllustrationResult}, this is achieved by minimizing the Wasserstein-2 distance between the encoder's output $q_\varphi(z|x)$ and the approximated posterior distribution $\mathcal{Q}_{\mathrm{T}}(z)$ obtained from the decoder training step. The Sinkhorn algorithm (Algorithm.~\ref{algo:sinkhorn}) is employed to efficiently compute the gradient of the Wasserstein-2 distance, which is then used to update the encoder parameters $\varphi$. On this basis, during the model testing stage, the encoder and decoder are connected in series to predict the label $y$. Given a new input $x$, the encoder $q_\varphi(z|x)$ first infers the latent variable $z$, which is then fed into the decoder $p_\theta(\mathcal{D}|z)$ to generate the predicted output $\hat{y}$. 

\begin{figure}[!h]
    \centering
    \includegraphics[width=0.7\textwidth]{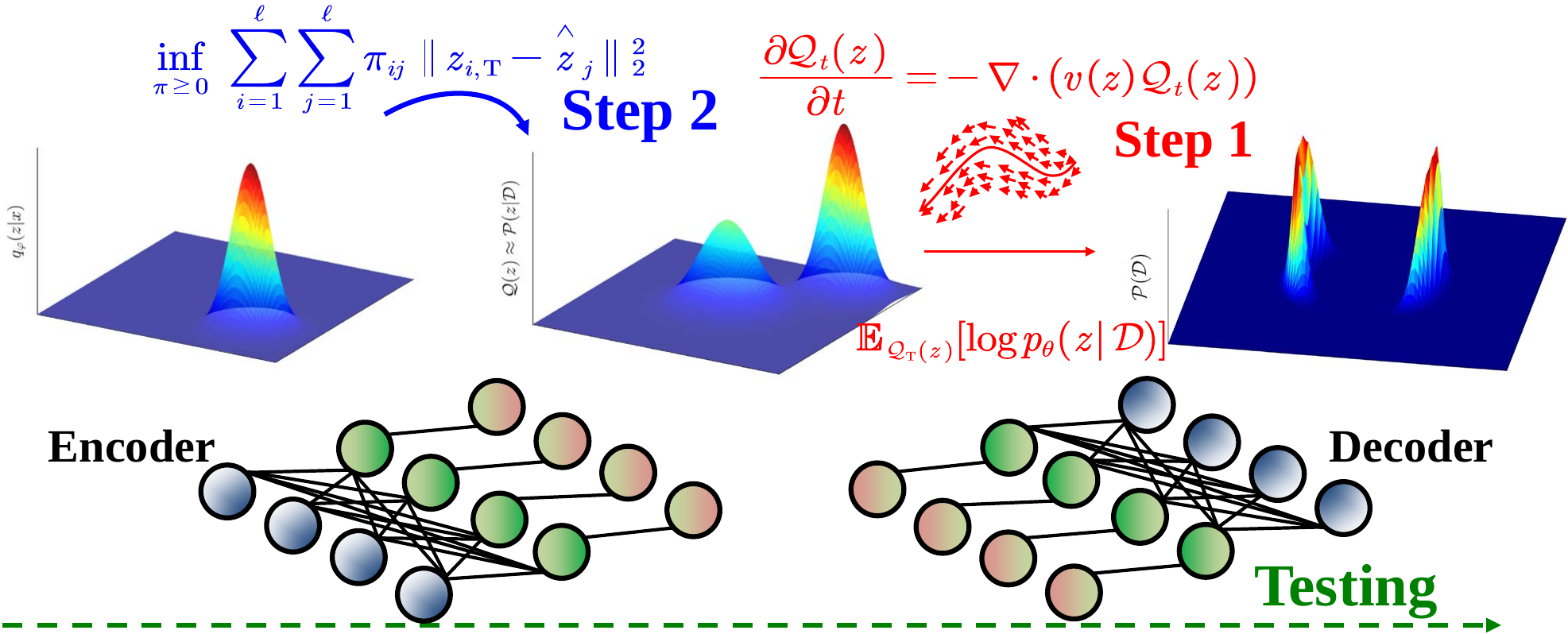}
   % \vspace{-0.5cm}
    \caption{The overall workflow of KProxNPLVM.}\label{fig:modelIllustrationResult}
   %\vspace{-0.3cm}
\end{figure}

\section{Experimental Results}\label{sec:experResults}
In this section, the following questions are investigated empirically to demonstrate the efficacy of the proposed KProx algorithm and KProxNPLVM:
\\\textbf{RQ1: }Can the KProx algorithm accurately approximate the posterior distribution?
\\\textbf{RQ2:} What is the performance of KProxNPLVM in the industrial soft sensor task? 
\\\textbf{RQ3:} How does the performance of KProxNPLVM vary with changes in hyperparameters?
\\\textbf{RQ4:} What factors contribute to the impressive performance of KProxNPLVM?
\\\textbf{RQ5:} Does the training process of KProxNPLVM converge?

% \\\textbf{RQ6:} What's the computational time of KProxNPLVM and other baseline models?

\subsection{Posterior Approximation Trajectory Visualization}
In this subsection, we address \textbf{RQ1}: ``Does the KProx algorithm take effect?" To answer this, we conduct a qualitative experiment visualizing the evolution trajectory of the probability density function (PDF). We initialize the approximate distribution as $\mathcal{Q}_0(z) = \mathcal{N}(0,1)$ (normal distribution) or $\mathcal{Q}_0(z) = \mathcal{U}(-0.5,0.5)$ (uniform distribution), and the target posterior distribution is defined as $\mathcal{P}(z|x) \propto \frac{1}{2}\mathcal{N}(-2, 0.5^2) + \frac{1}{2}\mathcal{N}(2, 0.5^2)$. The PDF evolution trajectories are illustrated in \Cref{subfig:densityEvolutionResults,subfig:unidensityEvolutionResults}. Based on this, the 2-Wasserstein distance  $\mathcal{W}_2(\mathcal{Q}_t(z), \mathcal{P}(z|\mathcal{D}))$ are illustrated in \Cref{subfig:WassDistanceDifferences,subfig:uniWassDistanceDifferences}.
\begin{figure}[htbp]
 %  \vspace{-0.3cm}
    \centering
    % figures\cut_sampled_results.pdf
     \subfigure[Trajectory of $\mathcal{Q}_t(z)$ along $t$, where $\mathcal{Q}_{0}(z)= \mathcal{N}(0,1)$.]{\includegraphics[width=0.24\linewidth]{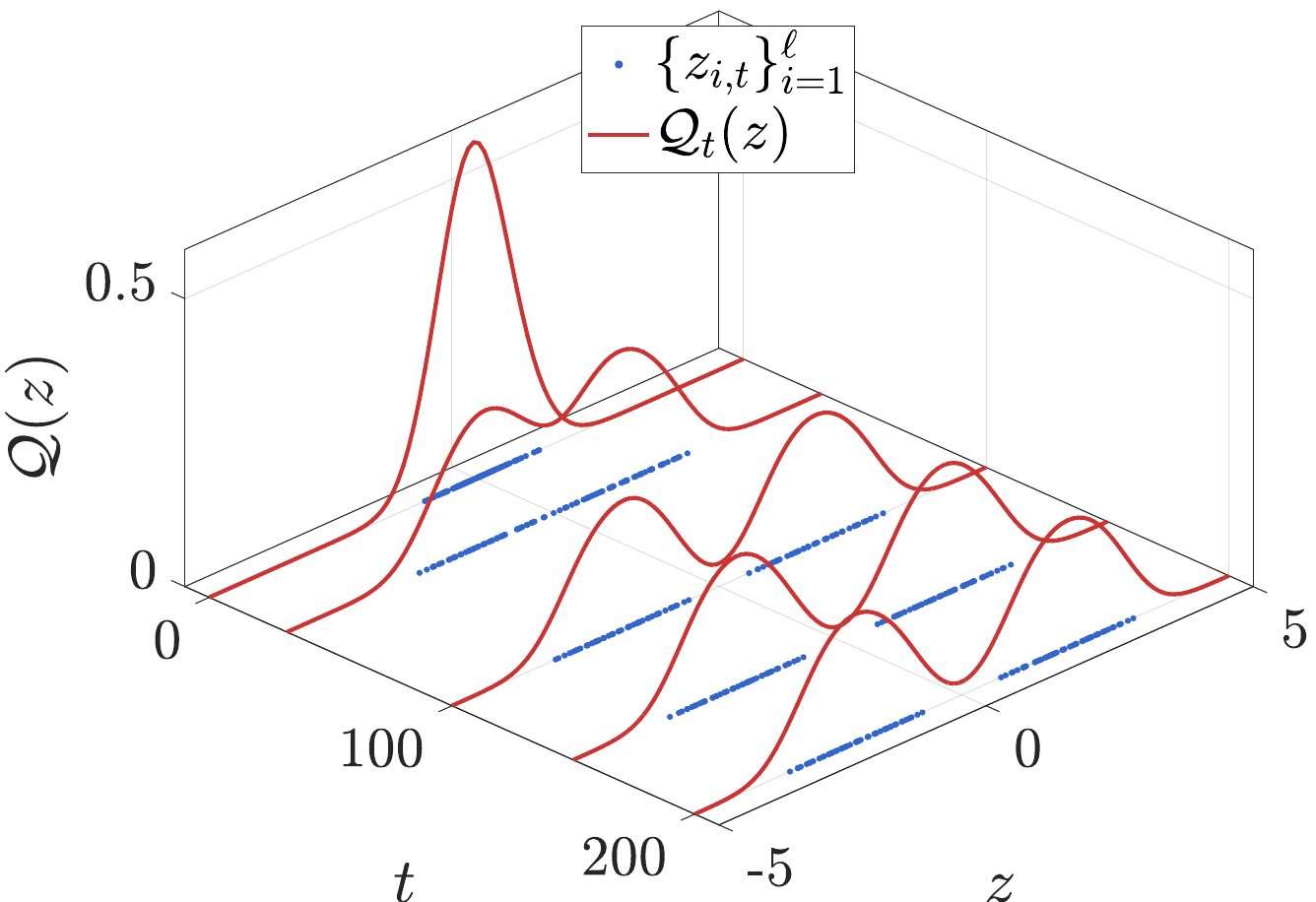}\label{subfig:densityEvolutionResults}}
        \subfigure[$\mathcal{W}_2(\mathcal{Q}_t(z), \mathcal{P}(z|\mathcal{D}))$ along $t$, where $\mathcal{Q}_{0}(z)= \mathcal{N}(0,1)$.]{\includegraphics[width=0.24\linewidth]{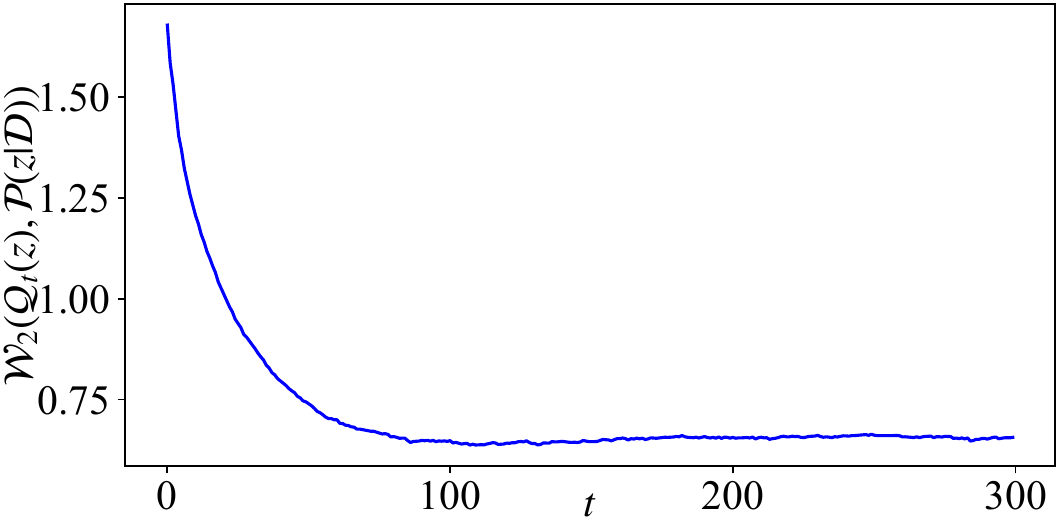}\label{subfig:WassDistanceDifferences}}
         \subfigure[Trajectory of $\mathcal{Q}_t(z)$  along $t$, where $\mathcal{Q}_{0}(z)= \mathcal{U}(-0.5,0.5)$.]{\includegraphics[width=0.24\linewidth]{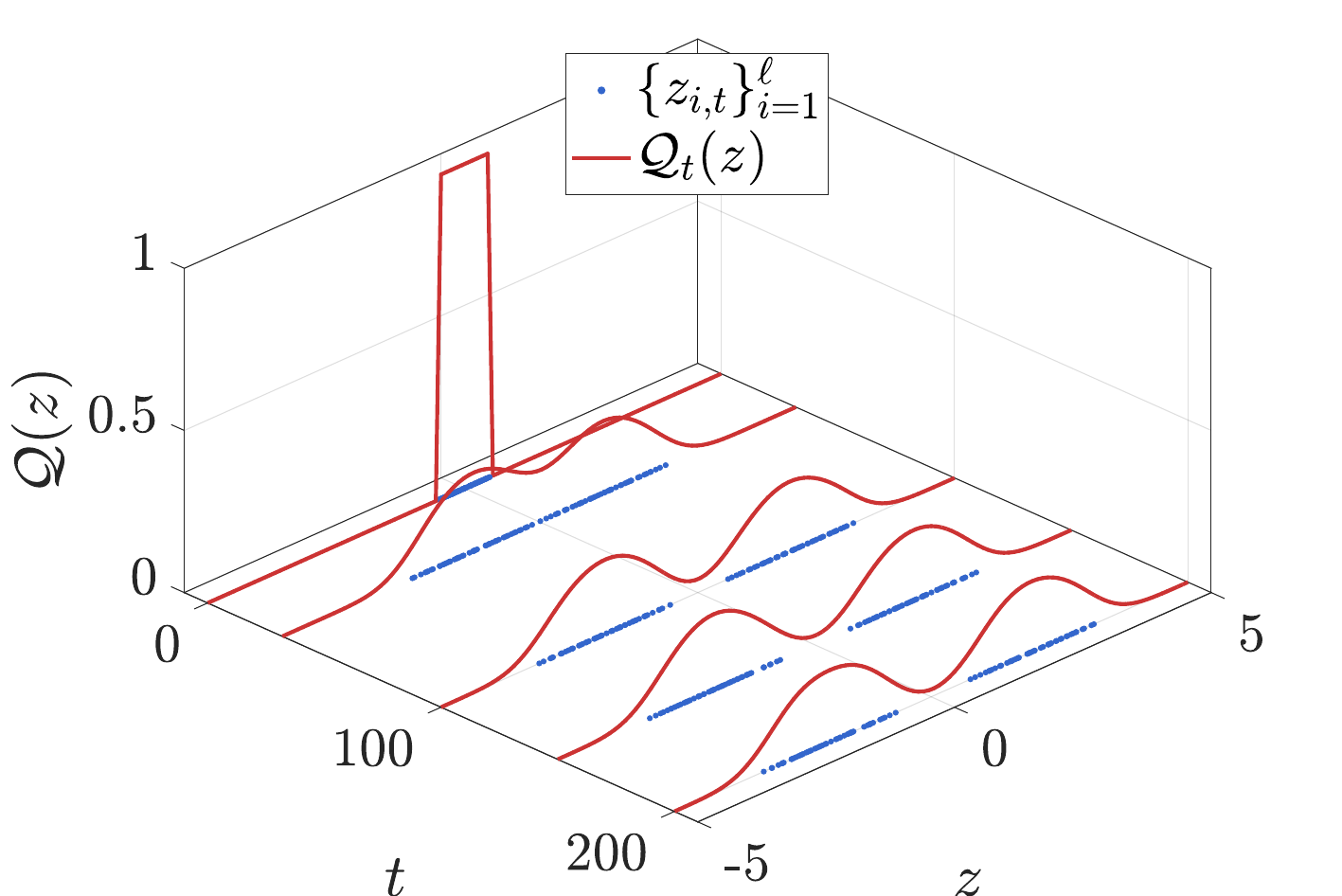}\label{subfig:unidensityEvolutionResults}}
    \subfigure[$\mathcal{W}_2(\mathcal{Q}_t(z), \mathcal{P}(z|\mathcal{D}))$ along $t$, where $\mathcal{Q}_{0}(z)= \mathcal{U}(-0.5,0.5)$.]{\includegraphics[width=0.24\linewidth]{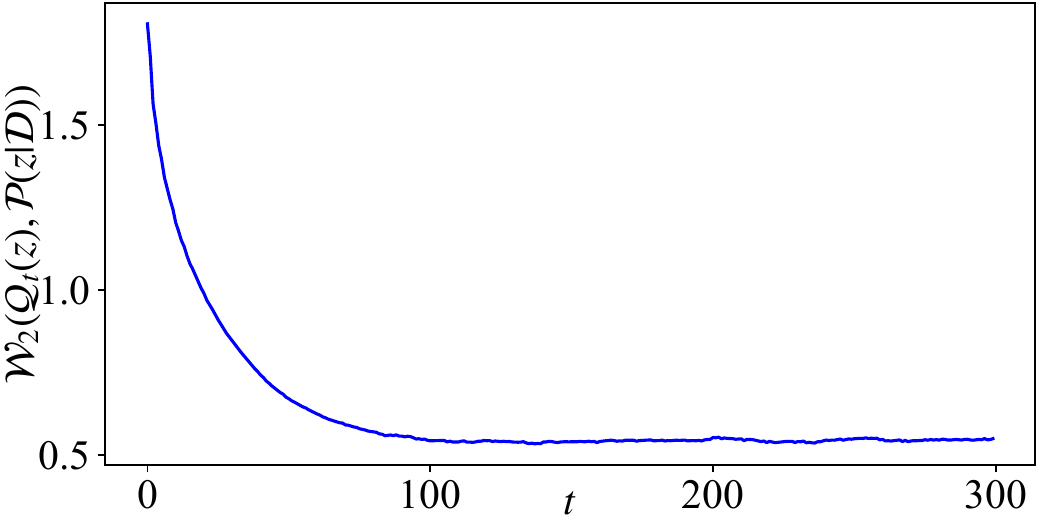}\label{subfig:uniWassDistanceDifferences}}
%   \vspace{-0.3cm}
    \caption{The evolution trajectory of the $\mathcal{Q}_t(z)$ estimated by the kernel density estimation (KDE), and the 2-Wasserstein distance $\mathcal{W}_2(\mathcal{Q}_t(z), \mathcal{P}(z|\mathcal{D}))$.% Comparison results between $\mathcal{P}(z|x)/p_{\theta,\text{NF}}(z|x)$ with $\mathcal{Q}(z)$ along iterative time $\tau$. %  $q(\vec{a})$.
    %. (b) Impact of entropic regularization strength ε. (c) Impact of PFOR strength γ (×103). (d) Impact of RMPR strength κ.
    }\label{fig:sensInfLR}
  % \vspace{-0.6cm}
\end{figure}

As observed in \Cref{subfig:densityEvolutionResults,subfig:unidensityEvolutionResults}, with the progression of $t$, the approximate distribution $\mathcal{Q}_t(z)$ evolves to exhibit two distinct modes. Furthermore, the high-probability density regions gradually extend into areas initially disjoint from the support of the initial Gaussian distribution. This behavior demonstrates that the proposed algorithm effectively adapts the shape of the variational distribution and successfully approximates posterior distributions, even when the initial distributions have minimal overlap. To further validate this, we analyze the 2-Wasserstein distance, $\mathcal{W}_2(\mathcal{Q}_t(z), \mathcal{P}(z|\mathcal{D}))$. As shown in \Cref{subfig:WassDistanceDifferences,subfig:uniWassDistanceDifferences}, the Wasserstein distance gradually decreases with increasing $t$, providing qualitative evidence of the KProx algorithm's efficacy. In addition, we observe that, when we change the initial guess of $\mathcal{Q}_t(z)$, both examples consistently converges to the target distribution. This observation demonstrates the robustness of the proposed approach to the initial guess, thereby demonstrating the superiority of the proposed KProx algorithm. In summary, these qualitative and quantitative observations empirically demonstrate the effectiveness of the KProx algorithm, providing a positive answer to \textbf{RQ1}.

\subsection{Soft Sensor Performance Analysis}\label{subsec:softSensPerformanceAnalysis}

In this subsection, we address the research problem \textbf{RQ2}: `What is the performance of KProxNPLVM in the industrial soft sensor task?'. To this end, we conduct experiments on three real industrial datasets including the separation and reaction unit operation in chemical process, namely, debutanizer column (DBC), carbon-dioxide absorber column (CAC), and catalysis shift conversion unit (CSC). The brief information of these datasets are summarized as follows:

\begin{itemize}[leftmargin=*]
    \item{\textbf{DBC:} The DBC benchmark dataset~\cite{fortuna2005soft}, collected from a refinery, describes the operation of a debutanizer whose objectives are to maximize the pentane content in the overhead distillate while minimizing the butane content in the bottom product. The target for DBC is real-time estimating the bottom butane concentration.}
 \item{\textbf{CAC:} The CAC dataset~\cite{shen2020nonlinear} comes from an ammonia synthesis process, where a caustic solvent is used to remove the carbon dioxide by-product from the hydrogen stream via absorption reactions. For downstream urea-quality assurance, the outlet-gas carbon dioxide concentration should be monitored in real time.}
  \item{\textbf{CSC:} The CSC dataset~\cite{8894374}, collected from an ammonia synthesis process, describes a series of fixed-bed reactors in which the water–gas shift reaction converts carbon monoxide and steam into hydrogen and carbon dioxide. To satisfy the required carbon–hydrogen ratio, the objective is to estimate the carbon monoxide concentration in real time.}
\end{itemize}
\noindent The detailed information of these datasets is provided in the supplementary material. To evaluate model performance, the metrics termed root mean squared error ($\textrm{RMSE}$), determination of coefficient ($\textrm{R}^2$), mean absolute error ($\textrm{MAE}$), and mean absolute percentage error ($\textrm{MAPE}$) are utilized, and the detailed expressions are provided in the supplementary material. 
The following class of models are considered as the baseline models. Due to page limit, the reasons for choosing these models and the experimental protocol including hyperparameters are listed in the supplement:
\begin{itemize}[leftmargin=*]
   \item{\textbf{NPLVMs:} Supervised Nonlinear Probabilistic Latent Variable Regression (NPLVR)~\cite{shen2020nonlinear}, Deep Bayesian Probabilistic Slow Feature Analysis (DBPSFA)~\cite{9625835}, Modified Unsupervised VAE-Supervised Deep VAE (MUDVAE-SDVAE)~\cite{xie2019supervised}, and Gaussian Mixture Model-Variational Autoencoder~\cite{guo2021just}. }
    \item{\textbf{Non-PLVMs:} Variable-wise weighted stacked autoencoder (VW-SAE)~\cite{8302941}, Gated-Stacked Target-Related Autoencoder (GSTAE)~\cite{9174659},  Deep Learning model with Dynamic Graph (DGDL)~\cite{9929274}, and   iTransformer~\cite{liu2024itransformer}. }
\end{itemize}

\begin{table*}[!h]
    \caption{Soft Sensor Accuracy Comparison}\label{tab:softSensorAccuracy}
     \vspace{-0.2cm}
   \resizebox{\linewidth}{!}{
        \begin{threeparttable}
\begin{tabular}{l|llll|llll|llll}\toprule \multirow{2}{*}{Model} & \multicolumn{4}{c|}{DBC}& \multicolumn{4}{c|}{CAC}& \multicolumn{4}{c}{CSC}\\ \cmidrule{2-13} & \multicolumn{1}{c}{$\text{R}^\text{2}$} & \multicolumn{1}{c}{RMSE} & \multicolumn{1}{c}{MAE} & \multicolumn{1}{c|}{MAPE} & \multicolumn{1}{c}{$\text{R}^\text{2}$} & \multicolumn{1}{c}{RMSE} & \multicolumn{1}{c}{MAE} & \multicolumn{1}{c|}{MAPE} & \multicolumn{1}{c}{$\text{R}^\text{2}$} & \multicolumn{1}{c}{RMSE} & \multicolumn{1}{c}{MAE} & \multicolumn{1}{c}{MAPE} \\ \midrule SNPLVR  & -1.02E-1$\dagger$ & 2.11E-1$\dagger$ & 1.71E-1$\dagger$ & 2.89E2$\dagger$ & -2.29E-1$\dagger$ & 7.90E-3$\dagger$ & 5.94E-3$\dagger$ & 2.03E0$\dagger$ & -3.13E-1$\dagger$ & 6.76E-1$\dagger$ & 5.42E-1$\dagger$ & 2.76E-1$\dagger$ \\  DBPSFA  & 2.51E-1$\dagger$ & 1.75E-1$\dagger$ & 1.43E-1$\dagger$ & 2.80E2$\dagger$ & -6.78E-2 & 7.36E-3 & 5.63E-3 & 1.93E0 & -3.73E4$\dagger$ & 1.15E2$\dagger$ & 1.15E2$\dagger$ & 5.85E1$\dagger$ \\  MUDVAE-SDVAE  & -1.04E-2$\dagger$ & 2.03E-1$\dagger$ & 1.62E-1$\dagger$ & 2.75E2$\dagger$ & -5.29E-3$\dagger$ & 7.14E-3$\dagger$ & 5.23E-3$\dagger$ & 1.78E0$\dagger$ & -1.64E-1$\dagger$ & 6.41E-1$\dagger$ & 5.13E-1$\dagger$ & 2.61E-1$\dagger$ \\  GMM-VAE  & 7.93E-1$\dagger$ & 8.15E-2$\dagger$ & 6.50E-2$\dagger$ & 8.99E1$\dagger$ & 2.89E-1$\dagger$ & 6.14E-3$\dagger$ & 4.66E-3$\dagger$ & 1.59E0$\dagger$ & 7.83E-1$\dagger$ & 2.75E-1$\dagger$ & 2.19E-1$\dagger$ & 1.11E-1$\dagger$ \\  GSTAE  & 9.70E-1$\dagger$ & 3.52E-2$\dagger$ & 1.48E-2$\dagger$ & 2.56E1$\dagger$ & \uwave{7.40E-1}$\dagger$ & \uwave{3.63E-3}$\dagger$ & \textbf{2.72E-3}$\dagger$ & \textbf{9.30E-1}$\dagger$ & 8.93E-1$\dagger$ & 1.93E-1$\dagger$ & 1.53E-1$\dagger$ & 7.77E-2$\dagger$ \\  VW-SAE  & 2.35E-1$\dagger$ & 1.76E-1$\dagger$ & 1.31E-1$\dagger$ & 2.44E2$\dagger$ & 4.39E-1$\dagger$ & 5.32E-3$\dagger$ & 3.74E-3$\dagger$ & 1.28E0$\dagger$ & 7.46E-1$\dagger$ & 2.95E-1$\dagger$ & 2.17E-1$\dagger$ & 1.11E-1$\dagger$ \\  DGDL  & 9.82E-1$\dagger$ & 2.59E-2$\dagger$ & 1.77E-2$\dagger$ & \uwave{1.53E1} & 7.35E-1 & 3.67E-3$\dagger$ & 2.83E-3 & 9.73E-1 & 9.31E-1$\dagger$ & 1.56E-1$\dagger$ & 1.23E-1$\dagger$ & 6.28E-2$\dagger$ \\  iTransformer  & \uwave{9.90E-1} & \uwave{1.77E-2} & \uwave{1.14E-2} & 3.87E1 & 6.97E-1$\dagger$ & 3.92E-3$\dagger$ & 3.06E-3$\dagger$ & 1.05E0$\dagger$ & \uwave{9.39E-1} & \uwave{1.46E-1} & \uwave{1.16E-1} & \uwave{5.91E-2} \\  KProxNPLVM  & \textbf{9.98E-1} & \textbf{9.84E-3} & \textbf{7.72E-3} & \textbf{9.25E0} & \textbf{7.52E-1} & \textbf{3.55E-3} & \uwave{2.79E-3} & \uwave{9.57E-1} & \textbf{9.41E-1} & \textbf{1.44E-1} & \textbf{1.15E-1} & \textbf{5.87E-2} \\ \midrule Win Counts & 8  & 8  & 8  & 8  & 8  & 8  & 7  & 7  & 8  & 8  & 8  & 8  \\ \bottomrule\end{tabular}
\begin{tablenotes}
    \item{$\dag$ marks variants that KProxNPLVM model significantly at $p$-value $<$ 0.05 over paired samples $t$-test. \textbf{Bolded} and \uwave{Wavy} results indicate first and second best in each metric. %\uwave{Wavy} results indicate the second best in each metric.
    }
    \end{tablenotes}
    \end{threeparttable}
} 
% \vspace{-0.3cm}
%
\end{table*}

\Cref{tab:softSensorAccuracy} presents the baseline comparison results, from which the following observations can be made:
\begin{enumerate}[leftmargin=*]
    \item{For all datasets, the performance of most NPLVMs does not surpass that of the majority of non-NPLVMs.
% For all datasets, the performance of most of the NPLVMs does not exceed that most of the non-NPLVMs.
    }
    \item{The GMM-VAE model outperforms most NPLVMs and demonstrates competitive performance compared to non-NPLVMs like GSTAE, DGDL and iTransformer.
% The GMM-VAE model outperforms most of the NPLVMs and even can beat with the non-NPLVMs 
    }
    % \item{All models perform better on the DBC dataset than on the CA dataset.}
    \item{KProx significantly outperforms the majority of baseline models, demonstrating not only superior predictive capabilities but also statistical significance in its results.
       %The KProxNPLVM
    % For the CA dataset, the S$^\text{2}$NPLVM surpasses most of the baseline models across several metrics, improving $\text{R}^\text{2}$ by 0.781\% to 47.682\%, reducing RMSE by 0.012\% to 0.288\%, lowering MAPE by 0.006\% to 0.293\%, and decreasing MAE by 0.006\% to 0.293\%.
        }
    %     \item{
    % For the DBC dataset, the KProxNPLVM surpasses most baseline models across several metrics, with improvements in $\text{R}^\text{2}$ ranging from 0.719\% to 3626\%, reductions in RMSE from 0.255\% to 0.905\%, decreases in MAPE from 0.034\% to 0.921\%, and reductions in MAE from 0.070\% to 0.908\%.
    % % For the DBC dataset, the S$^\text{2}$NPLVM surpasses most of the baseline models across several metrics, improving $\text{R}^\text{2}$ by 0.719\% to 3626\%, reducing RMSE by 0.255\% to 0.905\%, lowering MAPE by 0.034\% to 0.921\%, and decreasing MAE by 0.070\% to 0.908\%.
    %     }
\end{enumerate}
Observations 1) and 2) indicate that directly applying NPLVMs to soft sensor scenarios may be inadequate due to the limitations imposed by the parameterization of variational distributions within a uni-Gaussian prior. In support of this view, Observation 3) shows that when we replace the univariate Gaussian family with a more expressive Gaussian mixture model, the model's performance improves to some extent. Notably, GMM-VAE achieves substantially better performance than most NPLVM baselines. We attribute this gap primarily to posterior approximation. When the true posterior is highly complex (e.g., multimodal), a unimodal variational family can incur a large approximation error, leading to degraded predictive performance. In contrast, using a Gaussian mixture variational posterior provides greater flexibility and can better capture such complex structure, thereby reducing approximation error and improving performance. Finally, Observation 4 reveals that the proposed KProxNPLVM, a specific type of NPLVM, surpasses most baseline models and demonstrates the efficacy and superiority of the Wasserstein distance-based proximal operator regularization strategy for the NPLVM training, as described in \Cref{algo:kProxVIProcedure}. % In summary, the abovementioned experiments demonstrate the efficacy of th
% Observation 4) reveal that the proposed KProxNPLVM, as a specific type of NPLVM, not only surpasses most baseline models but also demonstrates the efficacy and superiority of the Wasserstein proximal operator strategy for NPLVMs learning in~\Cref{sec:proposedApproach}.
% Observation 1) indicates that direct application of NPLVMs to  sensor modeling may be detrimental to prediction accuracy, supporting the theoretical derivation in~\Cref{subsec:InfLVByOptLV}. 
% Observations 1) and 2) indicates that directly applying the NPLVMs to soft sensor scenario may not be suitable due to the restriction on the parameterization of variational distribution within uni-Gaussian prior. On this basis, observation 2) support this opinion since when we replac the univariate Gaussian family with more expressive Gaussian mixture family, the model performance increases to some extend. Finally, Observations 3) and 4) reveal that the proposed S$^\text{2}$NPLVM, as a special kind of NPLVM, can surpass most baseline models, demonstrating the efficacy and superiority of the proposed KProxNPLVM. %~\Cref{subsec:InfLVByOptLV,sec:InfNetTraining}.

% Observation 2) suggests that inferential sensor prediction accuracy heavily depends on the dataset. Finally, Observations 3) and 4) reveal that the proposed S$^\text{2}$NPLVM, as a special kind of NPLVM, can surpass most baseline models, demonstrating the effectiveness of the derivations in~\Cref{subsec:InfLVByOptLV,sec:InfNetTraining}.

  % \vspace{-0.5cm}
\subsection{Sensitivity Analysis Result}\label{subsec:sensResultAnalysis}
This subsection investigates research question \textbf{RQ3}: `How does the performance of KProxNPLVM vary with
changes in hyperparameters?'. To explore this, the hyperparameters—proximal operator coefficient $\varepsilon$, batch size $\mathcal{B}$, inference network learning rate $\eta$, and particle number $\ell$—are examined on the DBC dataset. The results are presented in~\Cref{fig:sensStepSize}.

From these figures, the following observations are made:
\begin{enumerate}[leftmargin=*]
    \item{As the proximal operator coefficient $\varepsilon$ increases, the performance of KProxNPLVM improves.}
    \item{As the batch size $\mathcal{B}$ increases, the performance of KProxNPLVM deteriorates.}
    \item{As the learning rate $\eta$ increases, the performance of KProxNPLVM first improves.}
    \item{As the number of particles $\ell$ increases, the performance of KProxNPLVM initially improves and then declines.}
    \end{enumerate}
Observation 1) indicates that when the proximal operator coefficient $\varepsilon$ is small, the latent variable inference procedure cannot approximate the posterior distribution close enough, resulting in degraded model performance. Observation 2) suggests that increasing the batch size $\mathcal{B}$ leads to decreased performance; when the batch size is too large, the model may become trapped in a local optimum, hindering generalization ability and thus degrading performance. Observation 3) implies that a very small learning rate $\eta$ causes the model to require more time to converge, possibly exceeding the predefined epochs. Finally, Observation 4) highlights that while a particle number greater than one enhances performance, an excessive number of particles can lead to overfitting and reduced performance. In conclusion, these findings underscore the importance of selecting a larger proximal operator coefficient $\varepsilon$, an appropriate batch size $\mathcal{B}$, a smaller learning rate $\eta$, and an optimal number of particles $\ell$ to ensure model performance.

\begin{figure}[htbp]
  \vspace{-0.3cm}
    \centering
    % figures\cut_sampled_results.pdf
    \subfigure[$\varepsilon$, DBC Dataset, $\text{R}^{\text{2}}$.]{\includegraphics[width=0.245\linewidth]{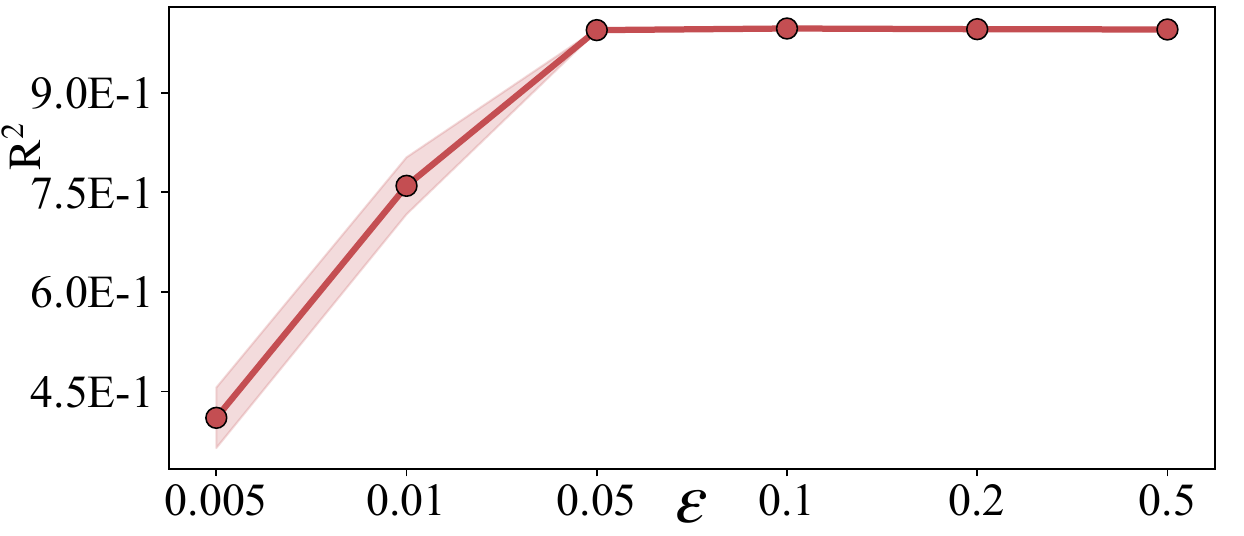}}%\label{subfig:a}
    \subfigure[$\varepsilon$, DBC Dataset, RMSE.]{\includegraphics[width=0.245\linewidth]{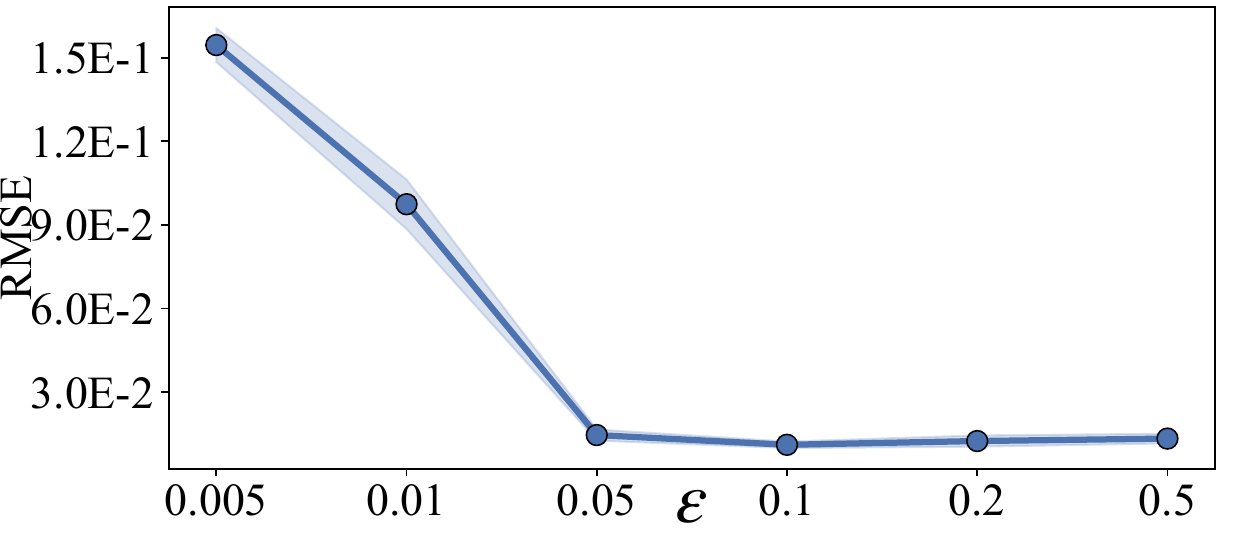}}
        \subfigure[$\mathcal{B}$, DBC Dataset, $\text{R}^{\text{2}}$.]{\includegraphics[width=0.245\linewidth]{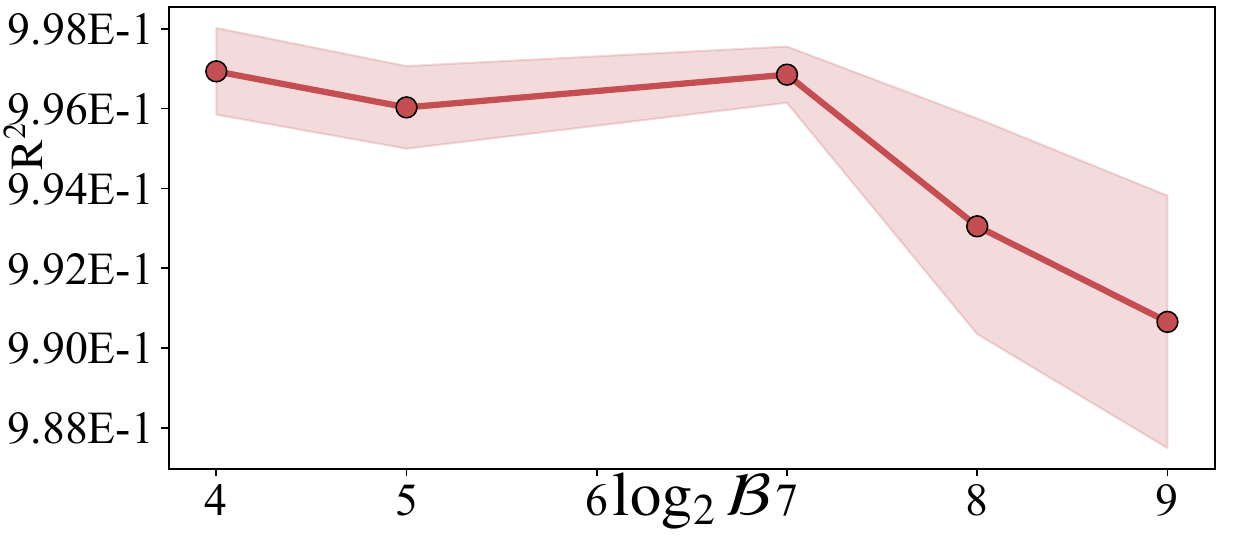}}%\label{subfig:a}
    \subfigure[$\mathcal{B}$, DBC Dataset, RMSE.]{\includegraphics[width=0.245\linewidth]{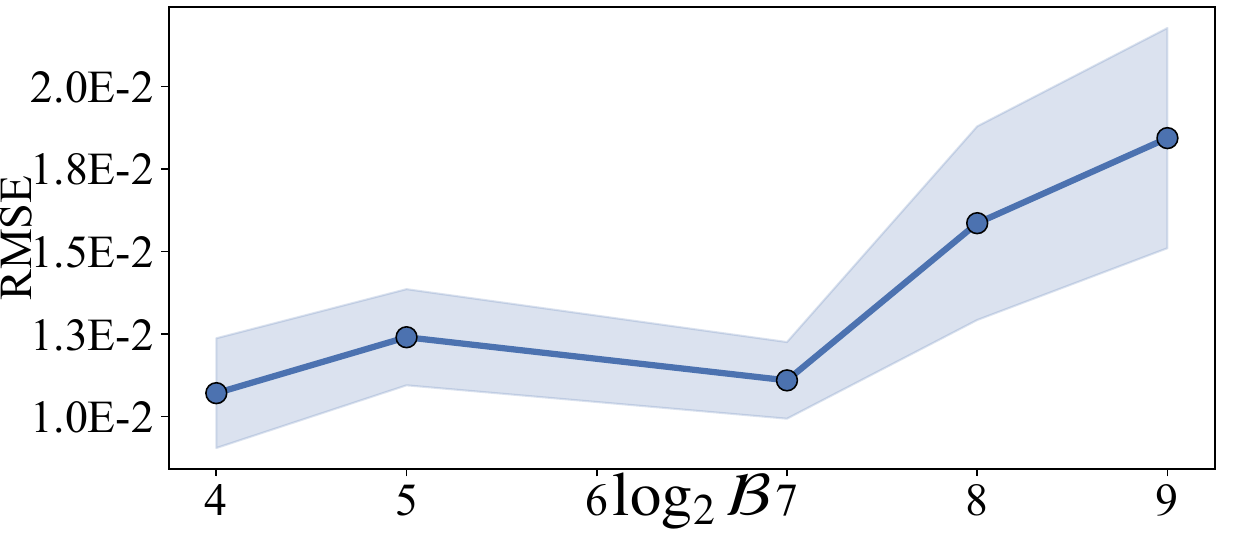}}
      \subfigure[$\eta$, DBC Dataset, $\text{R}^{\text{2}}$.]{\includegraphics[width=0.245\linewidth]{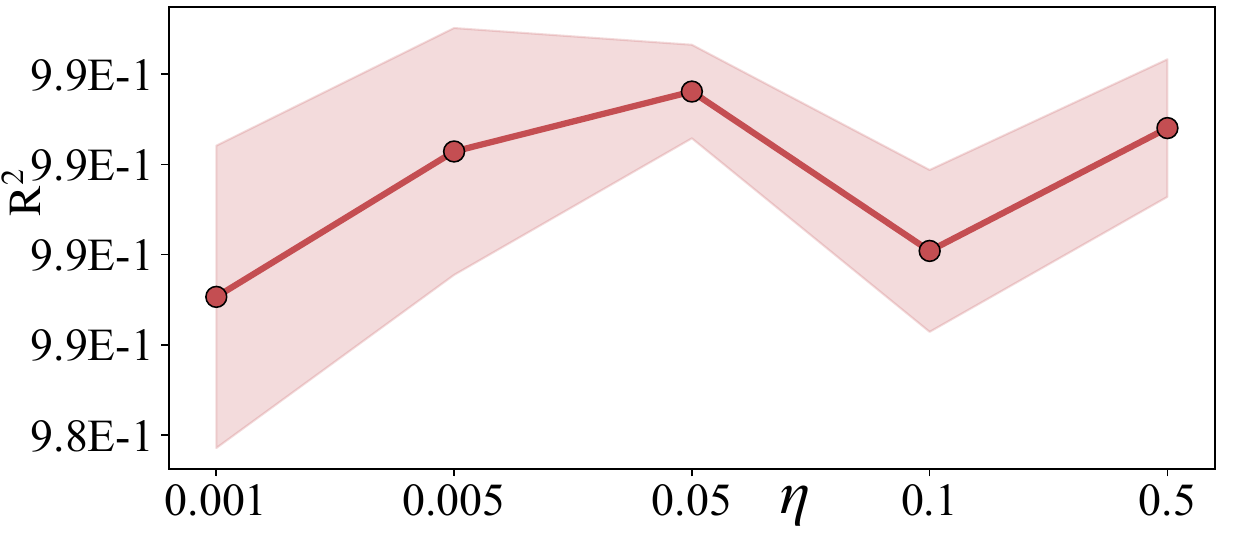}}%\label{subfig:a}
    \subfigure[$\eta$, DBC Dataset, RMSE.]{\includegraphics[width=0.245\linewidth]{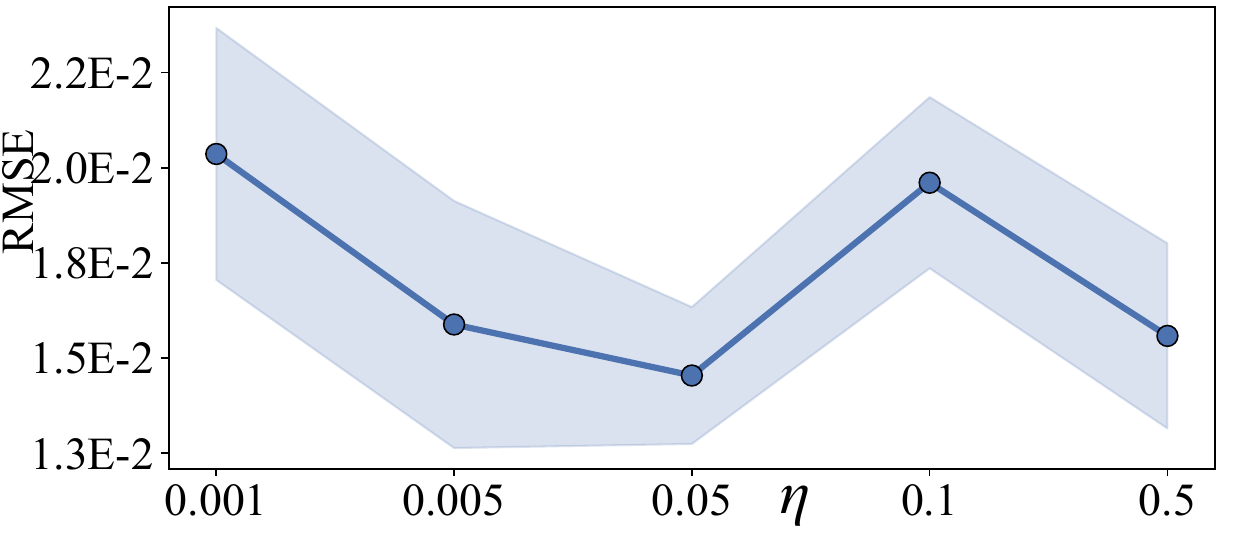}}
    \subfigure[$\ell$, DBC Dataset, $\text{R}^{\text{2}}$.]{\includegraphics[width=0.245\linewidth]{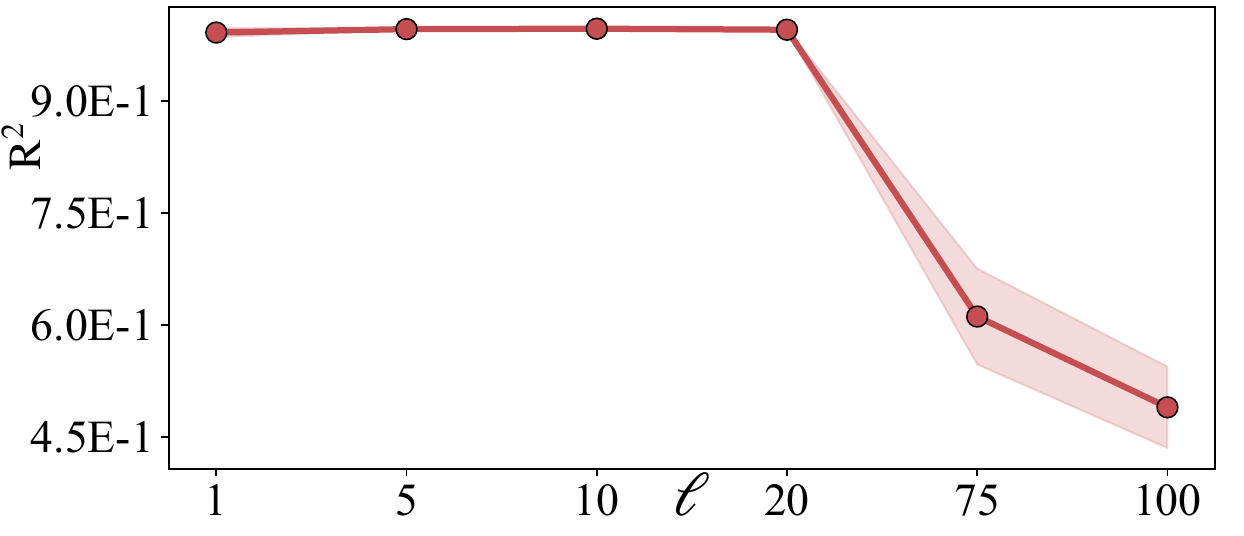}}%\label{subfig:a}
    \subfigure[$\ell$, DBC Dataset, RMSE.]{\includegraphics[width=0.245\linewidth]{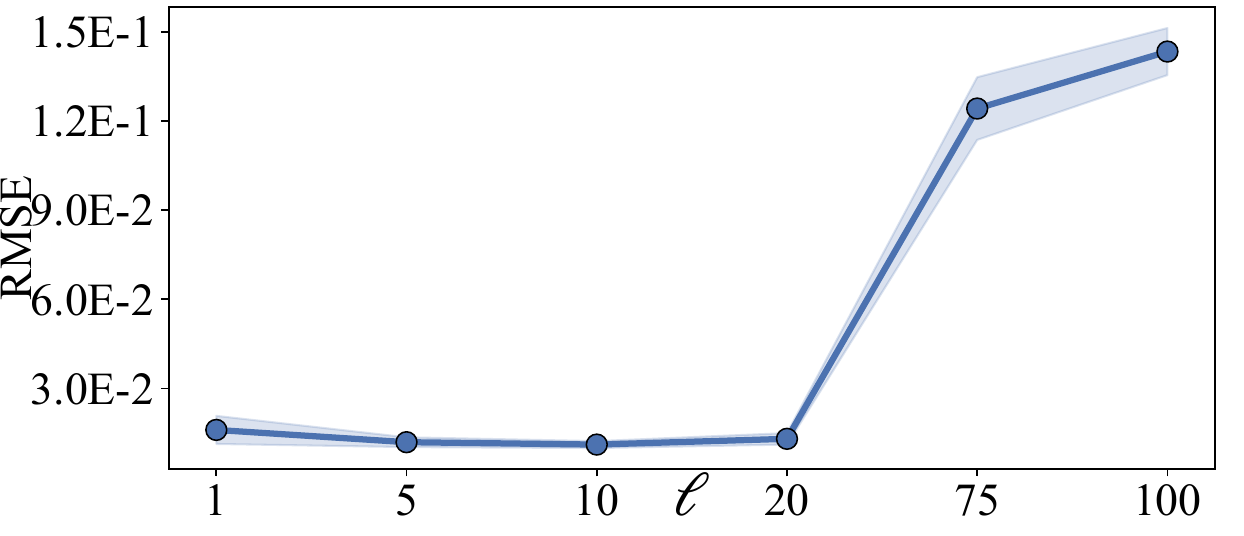}}
    % \subfigure[$\varepsilon$, CA Dataset, MAPE.]{\includegraphics[width=0.24\linewidth]{pictures/cut_sens_fig/stein_step/data_ca_metric_mape.pdf}}
    % \subfigure[$\varepsilon$, CA Dataset, MAE.]{\includegraphics[width=0.24\linewidth]{pictures/cut_sens_fig/stein_step/data_ca_metric_mae.pdf}}
    % \subfigure[$\varepsilon$, DBC Dataset, $\text{R}^{\text{2}}$.]{\includegraphics[width=0.24\linewidth]{pictures/cut_sens_fig/stein_step/data_dc_metric_r2.pdf}}%\label{subfig:a}
    % \subfigure[$\varepsilon$, DBC Dataset, RMSE.]{\includegraphics[width=0.24\linewidth]{pictures/cut_sens_fig/stein_step/data_dc_metric_rmse.pdf}}
    % \subfigure[$\varepsilon$, DBC Dataset, MAPE.]{\includegraphics[width=0.24\linewidth]{pictures/cut_sens_fig/stein_step/data_dc_metric_mape.pdf}}
    % \subfigure[$\varepsilon$, DBC Dataset, MAE.]{\includegraphics[width=0.24\linewidth]{pictures/cut_sens_fig/stein_step/data_dc_metric_mae.pdf}}
   \vspace{-0.3cm}
    \caption{The sensitivity analysis results vary proximal operator coefficient $\varepsilon$, batch size $\mathcal{B}$, inference network $\eta$, and particle number $\ell$ on the DBC dataset. The shaded area indicates the $\pm 0.5$ standard deviation error.% Comparison results between $\mathcal{P}(z|y)/p_{\theta,\text{NF}}(z|y)$ with $\mathcal{Q}(z)$ along iterative time $\tau$. %  $q(\vec{a})$.
    %. (b) Impact of entropic regularization strength ε. (c) Impact of PFOR strength γ (×103). (d) Impact of RMPR strength κ.
    }\label{fig:sensStepSize}
   %\vspace{-0.2cm}
\end{figure}

\subsection{Ablation Study}\label{subsec:ablationResultAnalysis}

This subsection addresses research question \textbf{RQ4}: `What factors contribute to the impressive performance of KProxNPLVM?'. To explore this question, we conducted an ablation study focusing on two critical components: the KProx algorithm (abbreviated as `KProx') and the Wasserstein distance-based inference network learning strategy (abbreviated as `Wass'). For `KProx', we replaced the generative network training with a conventional VAE model using the Gaussian distribution $\mathcal{N}\left(0, I\right)$ as the latent variable prior. For `Wass', we replaced the learning objective with $-\mathbb{E}_{\mathcal{Q}(z)}\left[\log{p_{\theta}(y|z)}\right] + \mathbb{D}_{\text{KL}}\left[q_\varphi(z) \Vert \mathcal{Q}(z)\right]$, where $\theta$ is frozen. 
\begin{table}[htbp]  \vspace{-0.0cm} \caption{Ablation Study Results}\label{tab:ablationStudyResult}  \vspace{-0.2cm} \centering \begin{tabular}{c|l|l|l|l|l|l} \toprule \multicolumn{1}{l|}{Dataset} & KProx & Wass & $\text{R}^{\text{2}}$ ($\downarrow$) & RMSE ($\uparrow$) & MAE ($\uparrow$) & MAPE ($\uparrow$) \\ \midrule \multirow{3}{*}{DBC}&  \XSolidBrush     &   \Checkmark& 2.73E3\%& 1.53E3\%& 1.59E3\%& 1.43E3\%\\&  \Checkmark      & \XSolidBrush& 4.20E0\%& 2.34E2\%& 2.50E2\%& 3.22E2\%\\&    \XSolidBrush     &  \XSolidBrush& 2.44E3\%& 1.53E3\%& 1.59E3\%& 1.42E3\%\\\midrule \multirow{3}{*}{CAC}&  \XSolidBrush     &   \Checkmark& 1.86E5\%& 1.00E2\%& 8.79E1\%& 8.64E1\%\\&  \Checkmark      & \XSolidBrush& 7.51E1\%& 5.14E1\%& 4.86E1\%& 4.82E1\%\\&    \XSolidBrush     &  \XSolidBrush& 7.87E4\%& 1.01E2\%& 8.85E1\%& 8.71E1\%\\\midrule \multirow{3}{*}{CSC}&  \XSolidBrush     &   \Checkmark& 5.52E2\%& 3.48E2\%& 3.48E2\%& 3.49E2\%\\&  \Checkmark      & \XSolidBrush& 7.59E0\%& 4.45E1\%& 4.21E1\%& 4.21E1\%\\&    \XSolidBrush     &  \XSolidBrush& 4.79E2\%& 3.56E2\%& 3.56E2\%& 3.56E2\%\\\bottomrule \end{tabular} \vspace{-0.0cm} \end{table}

The results are listed in~\Cref{tab:ablationStudyResult}, where the following observations can be made:
% From~\Cref{tab:abStudyExperiments}, the following observations can be obtained:
\begin{enumerate}[leftmargin=*]
    \item{
    Comparing the scenarios where `KProx' is ablated (Lines 1 and 3) to those where `KProx' is not ablated (Lines 2), the prediction accuracy of KProxNPLVM is greatly reduced.
    % Comparing to the scenario that the `Gen' is ablated (Lines 1 and 3) to that `Gen' is not ablated (Lines 2 and 4), the prediction accuracy of S$^\text{2}$NPLVM is greatly reduced.
    }
    \item{ When `Wass' is ablated (Line 2), the model performance also decreases (Line 4).
        % When the `Inf' is ablated, the model performance may also reduce. 
        }
\end{enumerate}
Observation 1) indicates that directly applying a network structure with $q_{\varphi}(z|x)$ that parameterizes the variational distribution may result in approximation error, thereby hindering the prediction accuracy of the soft sensor. This supports the justification provided in~\Cref{subsec:MotivationAnalysis,subsec:WassProximalOperator}. Observation 2) suggests that when the latent variable is sufficiently accurate, directly using it as the learning objective for inference network training suffices for model performance, underscoring the necessity of the learning objective designed in~\Cref{subsec:ModelTrainingStrategy}. In summary, the ablation study results strongly support the importance of integrating both the KProx in the latent variable distribution inference process and the Wasserstein distance for the inference network training into the KProxNPLVM training protocol, demonstrating their collective significance in ensuring optimal model performance.

% \textbf{We have finished the ablation study}

% \vspace{-0.5cm}
\subsection{Empirical Convergence Analysis}\label{subsec:convergenceAnalysisEmpirical}

In this section, we empirically address the \textbf{RQ5}: `Does the training process of KProxNPLVM converge?'~\Cref{fig:convergenceResult} showcases the progression of the expected log-likelihood, $\mathbb{E}_{\mathcal{Q}(z)}[\log p_{\theta}(z|\mathcal{D})]$, over training epochs on the CAC, DBC, and WGS datasets. A distinct pattern of rapid convergence emerges. Across all datasets, the learning objective quickly rises from a negative value and stabilizes near the optimal value of zero within five epochs. The minimal standard deviation, represented by a narrow shaded region around the mean curve, further highlights the algorithm's stable and consistent performance across multiple runs under different initialization parameters. These empirical findings align seamlessly with our theoretical analysis presented in~\Cref{thm:ConvergenceControlEpsilon}, providing strong evidence for the rapid and stable convergence of the KProx algorithm and answering \textbf{RQ5} empirically.

% In this section, we empirically validate the convergence properties of the InfO-EM algorithm, addressing \textbf{RQ5}: ``Does the ProxEM algorithm converge?"~\Cref{fig:convergenceResult} illustrates the evolution of the expected log-likelihood, $\mathbb{E}_{\mathscr{Q}(z)}[\log p_{\theta}(x|z)]$, as a function of training epochs on the CAC, DBC, and WGS datasets. A clear pattern of rapid convergence is observable. In all datasets, the learning objective swiftly rises from a negative value and plateaus near the optimal value of zero in under 5 epochs. The minimal standard deviation, visualized as a tight shaded region around the mean curve, further attests to the algorithm's stable and consistent performance across runs. These empirical results, in perfect alignment with our theoretical analysis in~\Cref{subsec:convergenceAnalysisResultsTheory}, provide compelling evidence for the rapid and stable convergence of the ProxEM algorithm, thereby answering \textbf{RQ5}.

\begin{figure}[!h]
    \centering
    % \includegraphics[width=0.40\textwidth]{picture/cut_gen_dc_precond.pdf}
   % \vspace{-0.45cm}
    \includegraphics[width=0.4\textwidth]{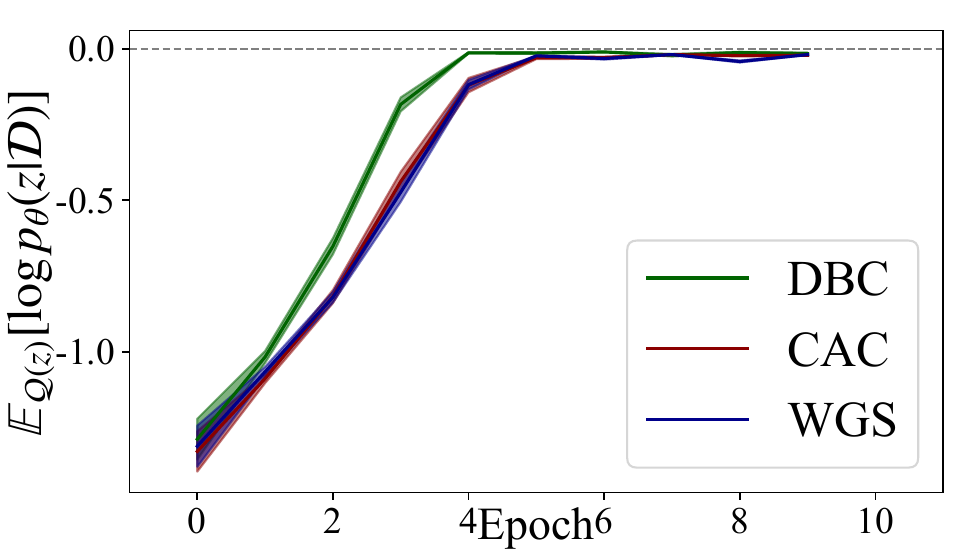}
      % \vspace{-0.45cm}
    \caption{The convergence results of the training process of KProx algorithm on three datasets. The shaded area indicates $\pm$0.25 times the standard deviation.}\label{fig:convergenceResult}
   % \vspace{-0.5cm}
   %  \vspace{-0.7cm}
\end{figure}

\section{Conclusions}\label{sec:Conclusions}
% \vspace{-0.2cm}
In this paper, we addressed the approximation accuracy challenges of NPLVMs arising from the parametric family of variational distributions by introducing the Wasserstein distance as a proximal operator and developing a novel latent variable inference strategy. Building on this foundation, we established a computationally implementable algorithm termed KProx for solving the regularized problem and proved its convergence properties within the RKHS. Furthermore, we developed a novel algorithm for training NPLVMs based on the KProx algorithm. Finally, experiments were conducted to demonstrate the efficacy of the KProx algorithm and KProxNPLVM.
% In this paper, to alleviate the approximation accuracy for NPLVMs stems from the parameteric family of variational distributions, we introduce the Wasserstein distance as proximal operator, and develop a novel latent variable inference strategy. Based on this, we develop an easy-to-implement strategy for the implementation of this regularized problem, and prove the convergence property of this implementation approach. Based on this, we developed a novel algorithm for NPLVM training. Finally, we conducted various experiments to demonstrate the efficacy of the proposed approach.

\noindent\textbf{Limitations \& Future Works:} In our manuscript, we utilized the RKHS to avoid the intractable computation of the $\nabla\log{\mathcal{Q}_t(z)}$. However, this approach may limit the expressiveness of the approximated velocity field and may falter in high-dimensional latent variable scenarios. Therefore, integrating neural networks and redesigning the velocity field approximation strategy is crucial for future research. Furthermore, while we employ Wasserstein distance as a proximal operator to regularize the optimization problem, exploring alternative discrepancy metrics, such as KL divergence, presents a valuable direction for further investigation.

% \appendix
\section*{Supplementary Material}
\addcontentsline{toc}{section}{Supplementary Material}

\setcounter{section}{0}
\setcounter{equation}{0}
\setcounter{figure}{0}
\setcounter{table}{0}

\renewcommand{\thesection}{S\arabic{section}}
\renewcommand{\theequation}{S\arabic{equation}}
\renewcommand{\thefigure}{S\arabic{figure}}
\renewcommand{\thetable}{S\arabic{table}}

\nomenclature{$\mathrm{D}$}{Dimension of process variable}
\nomenclature{KL divergence}{Kullback-Leibler divergence}
\nomenclature{$\mathrm{d}$}{Total differential operator}
\nomenclature{$\partial$}{Partial differential operator}
\nomenclature{$t$}{Time index for particle evolution/iteration}
\nomenclature{$\mathbb{F}$}{Normalized probability density function family}

\nomenclature{$\mathbb{D}_{\rm{KL}}[\mathcal{Q}(z)\Vert \mathcal{P}(z)]$}{Kullback-Leibler divergence of distribution $\mathcal{Q}(z)$ with-respect-to distribution $\mathcal{P}(z)$}
\nomenclature{$\mathbb{E}_{q(z)}[f(z)]$}{Expectation of function $f(z)$ with-respect-to distribution $q(z)$}
\nomenclature{$\rm{LV}$}{Latent variable}

\nomenclature{$\mathscr{P}_2(\mathrm{D})$}{Wasserstein space}
\nomenclature{$\mathbb{R}$}{Real number set}
\nomenclature{$\pi$}{Transportation plan}
\nomenclature{$\boldsymbol{T}$}{Transportation map}
\nomenclature{$\#$}{Pushforward measure}
\nomenclature{$v$}{Velocity field}
\nomenclature{$\mathscr{T}$}{Sufficient statistic}
\nomenclature{$\upeta$}{Natural parameter}
\nomenclature{$\mathbf{H.O.T}$}{Higher order term}
\nomenclature{$\varepsilon$}{Discretization step size}
\nomenclature{$\delta$}{First variation}
\nomenclature{$\mathscr{R}$}{Positive constant}
\nomenclature{$\vec{n}(z)$}{Outer normal vector}
\nomenclature{$\mathrm{d}S(z)$}{Surface element}
\nomenclature{$\mathcal{K}$}{Kernel function}

\nomenclature{$\mathscr{A}$}{Upper bound for the norm of velocity field}
\nomenclature{$\mathscr{B}$}{Upper bound for the norm of gradient for the velocity field}
\nomenclature{$\upepsilon$}{Entropy regularization strength}
\nomenclature{$\ell$}{Particle number}
\nomenclature{$\varphi$}{Parameter of the inference network}
\nomenclature{$\theta$}{Parameter of the generative network}
\nomenclature{$\upmu$}{Lagrange multiplier}
\nomenclature{$\upnu$}{Lagrange multiplier}

\nomenclature{$\mathscr{K}$}{Sinkhorn transition kernel}

\nomenclature{$\mathcal{W}_2$}{2-Wasserstein distance}

\nomenclature{$\mathcal{D}$}{Dataset}
\nomenclature{KProx}{Kernelized proximal gradient descent}
\nomenclature{$\mathcal{H}^{\mathrm{D}_{\rm{LV}}}$}{$\mathrm{D}_{\rm{LV}}$-dimensional RKHS}
\nomenclature{$\nabla$}{Nabla operator}
\nomenclature{$\mathscr{I}$}{Fisher information matrix}
\nomenclature{NPLVM}{Nonlinear probabilistic latent variable model}

\nomenclature{$p_\theta(\mathcal{D}\vert z)$}{Generative network with parameter $\theta$}
\nomenclature{$\mathcal{Q}_{\boldsymbol{T}}$}{Variational distribution after applying transportation direction $\boldsymbol{T}$.}
\nomenclature{$q_\varphi(z|x)$}{Inference network with parameter $\varphi$}
\nomenclature{$\mathcal{Q}_{\mathrm{T}}(z)$}{Variational distribution at terminal time $\mathrm{T}$}
\nomenclature{$\mathcal{Q}_t(z)$}{Variational distribution at time $t$}
\nomenclature{$\mathcal{Q}(z)$}{Variational distribution}

\nomenclature{$x$}{Process variable}
\nomenclature{$y$}{Quality variable}
\nomenclature{$\hat{y}$}{Predicted quality variable}

\nomenclature{$\exp$}{Expotential function.}

\printnomenclature

\section{Detailed Derivation of the Equations}
\subsection{Derivation of~\Cref{eq:HessianInequality}}
When $f:\mathbb{R}^{\mathrm{D}}\to\mathbb{R}^{\mathrm{D}}$ is twice continuously differentiable and the following conditions set up:
\begin{itemize}
    \item{\textbf{Condition 1:} The positive semidefinite Hessian matrix $\nabla^2f(x)$ is $M$-Lipschitz around $x$:
    \begin{equation}
         \|\nabla^2 f(y)-\nabla^2 f(x)\|\le M\|y-x\|.
    \end{equation}
    }
    \item{\textbf{Condition 2:} The positive semidefinite Hessian matrix $\nabla^2f(x)$ is local strong convex. In other words, for $w\coloneqq y-x$, we have:
    \begin{equation}
        w^\top \nabla^2f(x) w \ge \kappa \Vert w \Vert^2,
    \end{equation}
    where $\kappa$ is a positive constant.
    
    }
\end{itemize}
Then for $\Vert y - x \Vert \le \frac{3\kappa}{M}$, we get Eq. (12) in the main content:
\begin{equation*}
    f(y) \ge f(x) + [\nabla f(x)]^\top [y-x] + \frac{1}{2}[y-x]^\top \nabla^2f(x)[y-x]. 
\end{equation*}
To support this conclusion, we first apply the Taylor's expansion to $f(y)$ around $x$ as follows:
\begin{equation}
    f(y) = f(x) + [\nabla f(x)]^\top [y-x] +  \frac{1}{2}[y-x]^\top \nabla^2f(x)[y-x] + \mathrm{R}_3,
\end{equation}
where $\mathrm{R}_3$ is the 3-th order residual term. Thus, we should prove:
\begin{equation}
     \frac{1}{2}[y-x]^\top \nabla^2f(x)[y-x] + \mathrm{R}_3 \ge 0.
\end{equation}
According to ``Condition 1'', we observe that:
\begin{equation}\label{eq:R3UpperBoundResult}
    |\mathrm{R}_3| \le \frac{M}{6} \Vert w \Vert^3.
\end{equation}
Thus, we should prove the following inequality according to Eq.~\eqref{eq:R3UpperBoundResult} to support the conclusion:
\begin{equation}\label{eq:secondOrderSurrogateResult}
  \frac{1}{2}[y-x]^\top \nabla^2f(x)[y-x]  \ge  \frac{M}{6} \Vert w \Vert^3.
\end{equation}
According to ``Condition 2'', we know that:
\begin{equation}\label{eq:secondOrderlowerBoundResult}
    \frac{\kappa}{2}\Vert w \Vert^2 \le \frac{1}{2}[y-x]^\top \nabla^2f(x)[y-x].
\end{equation}
Comparing~\eqref{eq:secondOrderSurrogateResult} with~\eqref{eq:secondOrderlowerBoundResult}, we observe that:
\begin{equation}
      \frac{\kappa}{2}\Vert w \Vert^2 \ge  \frac{M}{6} \Vert w \Vert^3.
\end{equation}
Finally, we get the following result:
\begin{equation}\label{eq:inequalityCircle}
    \Vert y - x \Vert \le \frac{3\kappa}{M}.
\end{equation}
As a result, when the inequality defined by~\eqref{eq:inequalityCircle} is satisfied, we conclude Eq. (12) in the main content.

\iffalse
the positive semidefinite Hessian matrix $\nabla^2f(x)$ is $M$-Lipschitz around $x$, the following inequality holds: 
\begin{equation}
    \|\nabla^2 f(y)-\nabla^2 f(x)\|\le M\|y-x\|.
\end{equation}
Denote $w\coloneqq y-x$, we get the following inequality based on the second-order Taylor's expansion:
\begin{equation}
f(y)=f(x)+\nabla f(x)^\top w+\tfrac12 w^\top \nabla^2 f(x)\,w\;+\;R_3,
\end{equation}
where 
\begin{equation}
|R_3|\ \le\ \frac{M}{6}\,\|w\|^3
\end{equation}
when $y$ and $x$ are close enough. Since the on the direction $w$, we have the following inequality, we get:
\begin{equation}
    w^\top \nabla^2 f(x) w \ge \kappa \Vert  w\Vert^2.
\end{equation}
Thus, 
\begin{equation}
    \frac{M}{6}\Vert  w\Vert^3 \le \frac{1}{2}\kappa \Vert w \Vert^2 \iff \Vert w \Vert \le \frac{3\kappa}{M}.
\end{equation}
Finally, we get Eq. (12) in our main content:
\begin{equation}
    f(y) \ge f(x) + [\nabla f(x)]^\top [y-x] +\frac{1}{2}[y-x]^\top\nabla^2 f(x)[y-x]. 
\end{equation}
\fi
\subsection{Derivation of~\Cref{eq:klDivExpansionResult}}
Let us recall the expression for $\boldsymbol{T}(z)$:
\begin{equation}\label{eq:transportMap0}
      \boldsymbol{T}(z)=z+\varepsilon\, v(z),\quad \mathcal{Q}_{\boldsymbol{T}}:=\boldsymbol{T}_{\varepsilon\#}\mathcal{Q}.
\end{equation}
Notably, the $\varepsilon$ can be treated as time $t$. Based on this, we know that the the evolution of the probability density function satisfies the continuity equation:
\begin{equation}\label{eq:continutityEquation2333}
      \frac{\partial \mathcal{Q}_{\varepsilon}(z)}{\partial \varepsilon}
  =-\nabla\!\cdot\!\big(\mathcal{Q}_\varepsilon(z)\,v(z)\big),\qquad \mathcal{Q}_0(z)=\mathcal{Q}(z).
\end{equation}
Based on Eqs.~\eqref{eq:transportMap0} and~\eqref{eq:continutityEquation2333}, we know that:
\begin{equation}
    \mathcal{Q}_\varepsilon(z)=\mathcal{Q}(z)-\varepsilon\nabla\cdot(\mathcal{Q}(z)v(z)) + \mathcal{O}(\varepsilon^2).
\end{equation}
% \begin{equation}
%       \mathcal F(Q):=\mathrm{D}_{\mathrm{KL}}\!\big(Q\,\|\,\mathcal P\big)
%   =\int Q(z)\log\frac{Q(z)}{\mathcal P(z)}\,dz,
% \end{equation}
Taking the functional derivative of $\mathbb{D}_{\rm{KL}}[\mathcal{Q}_\varepsilon(z)\Vert \mathcal{P}(z\vert\mathcal{D})]$ with-respect-to $\mathcal{Q}_\varepsilon(z)$, we get:
\begin{equation}
\begin{aligned}
      & \frac{\mathrm{d}}{\mathrm{d}\varepsilon}\mathbb{D}_{\rm{KL}}[\mathcal{Q}_\varepsilon(z)\Vert \mathcal{P}(z\vert\mathcal{D})]\\
    = & \frac{\mathrm{d}}{\mathrm{d}\varepsilon}\int{\mathcal{Q}_\varepsilon(z)\log{\frac{\mathcal{Q}_\varepsilon(z)}{\mathcal{P}(z\vert\mathcal{D})}}\mathrm{d}z}\\
    =& \int{ \frac{\partial \mathcal{Q}_\varepsilon(z)}{\partial \varepsilon}\log{\frac{\mathcal{Q}_\varepsilon(z)}{\mathcal{P}(z\vert\mathcal{D})}} \mathrm{d}z } + \int{\mathcal{Q}_\varepsilon(z)\frac{\partial }{\partial \varepsilon}} \log{\frac{\mathcal{Q}_\varepsilon(z)}{\mathcal{P}(z\vert\mathcal{D})}} \mathrm{d}z\\
  = & \int{ \frac{\partial \mathcal{Q}_\varepsilon(z)}{\partial \varepsilon}\log{\frac{\mathcal{Q}_\varepsilon(z)}{\mathcal{P}(z\vert\mathcal{D})}} \mathrm{d}z } + \int{\cancel{\mathcal{Q}_\varepsilon(z) }\frac{\cancel{\mathcal{P}(z\vert\mathcal{D})}}{\cancel{\mathcal{Q}_\varepsilon(z)}}\frac{1}{\cancel{\mathcal{P}(z\vert\mathcal{D})}}\frac{\partial }{\partial \varepsilon}\mathcal{Q}_\varepsilon(z) \mathrm{d}z} \\
  = & \int{\frac{\partial \mathcal{Q}_\varepsilon(z)}{\partial\varepsilon}(\log{\frac{\mathcal{Q}_\varepsilon(z)}{\mathcal{P}(z\vert\mathcal{D})}}+1)\mathrm{d}z} \\
  = &\int{[-\nabla\cdot(\mathcal{Q}_\varepsilon(z)v(z))](\log{\frac{\mathcal{Q}_\varepsilon(z)}{\mathcal{P}(z\vert\mathcal{D})}}+1)\mathrm{d}z}  \\
\overset{\text{(i)}}{=} & \int{[(\mathcal{Q}_\varepsilon(z)v(z))]^\top\nabla(\log{\frac{\mathcal{Q}_\varepsilon(z)}{\mathcal{P}(z\vert\mathcal{D})}}+1)\mathrm{d}z} \\
= & \int{\mathcal{Q}_\varepsilon(z)v^\top(z) \nabla \frac{\delta \mathbb{D}_{\rm{KL}}[\mathcal{Q}_\varepsilon(z)\Vert \mathcal{P}(z\vert\mathcal{D})]}{\delta \mathcal{Q}_\varepsilon(z) } \mathrm{d}z} ,
 %  =& \int \Big\langle Q(z)\,v(z),\,\nabla\frac{\delta \mathcal F(Q)}{\delta Q(z)}\Big\rangle dz,
\end{aligned}
\end{equation}
where `(i)' is based on the integration-by-parts. 
Consequently, $\mathbb{D}_{\rm{KL}}[\mathcal{Q}_{\boldsymbol{T}}(z)\Vert\mathcal{P}(z\vert\mathcal{D})]$ can be expanded as follows when $\varepsilon\to0 $:
\begin{equation}\label{eq:klDivergenceResult2}
 \mathbb{D}_{\rm{KL}}[\mathcal{Q}_{\boldsymbol{T}}(z)\Vert\mathcal{P}(z\vert\mathcal{D})] =  \mathbb{D}_{\rm{KL}}[\mathcal{Q}(z)\Vert\mathcal{P}(z\vert\mathcal{D})] + \varepsilon \int{\mathcal{Q}_\varepsilon(z)v^\top(z) \nabla \frac{\delta \mathbb{D}_{\rm{KL}}[\mathcal{Q}_\varepsilon(z)\Vert \mathcal{P}(z\vert\mathcal{D})]}{\delta \mathcal{Q}_\varepsilon(z) } \mathrm{d}z} .
\end{equation}

For the 2-Wasserstein distance, we get:
\begin{equation}\label{eq:nonOptimalTransportResult}
    \mathcal{W}_2^2(\mathcal{Q}_{\boldsymbol{T}}(z),\mathcal{Q}(z)) = \int{\mathcal{Q}(z)\Vert z - \boldsymbol{T}^*(z)\Vert_2^2\mathrm{d}z} = \varepsilon^2 \int{\mathcal{Q}(z)\Vert  {v}^*(z)\Vert_2^2\mathrm{d}z} \le \varepsilon^2 \int{\mathcal{Q}(z)\Vert  {v}(z)\Vert_2^2\mathrm{d}z}  ,
\end{equation}
where $\boldsymbol{T}^*(z)$ and $v^*(z)$ are the optimal transportation map and optimal velocity field. Since $v(z)$ is not the optimal velocity filed, we obtain the last inequality. Based on Eqs.~\eqref{eq:nonOptimalTransportResult} and~\eqref{eq:klDivergenceResult2}, we finally reach the derivation of Eq. (19) in the main content:
\begin{equation}% \label{eq:klDivExpansionResult}
    \begin{aligned}
       & \mathbb{D}_{\rm{KL}}[\mathcal{Q}_{\boldsymbol{T}}(z)\Vert \mathcal{P}(z|\mathcal{D})] + \frac{1}{2\varepsilon} \mathcal{W}_2^2(\mathcal{Q}_{\boldsymbol{T}}(z),\mathcal{Q}(z))-  \mathbb{D}_{\rm{KL}}[\mathcal{Q}(z)\Vert \mathcal{P}(z|\mathcal{D})] \\
     \le  & \cancel{ \mathbb{D}_{\rm{KL}}[\mathcal{Q}(z)\Vert \mathcal{P}(z|\mathcal{D})] }+ \dfrac{\varepsilon}{2}\mathbb{E}_ \mathcal{Q}(z)[\Vert{v(z)}\Vert^2_2 ] + \varepsilon\int \cdot[{\mathcal{Q}(z) v^\top(z)}\nabla\dfrac{\delta \mathbb{D}_{\rm{KL}}[\mathcal{Q}(z)\Vert \mathcal{P}(z|\mathcal{D})]}{\delta \mathcal{Q}(z)}]\mathrm{d}z
     -  \cancel{\mathbb{D}_{\rm{KL}}[\mathcal{Q}(z)\Vert \mathcal{P}(z|\mathcal{D})]}
     \\
     \le & \dfrac{\varepsilon}{2}\underbrace{\mathbb{E}_{\mathcal{Q}(z)}[\Vert \nabla\dfrac{\delta \mathbb{D}_{\rm{KL}}[\mathcal{Q}(z)\Vert \mathcal{P}(z|\mathcal{D})]}{\delta \mathcal{Q}(z)}]
     \Vert^2_2 ]}_{\ge 0} +\dfrac{\varepsilon}{2}\mathbb{E}_{\mathcal{Q}(z)}[\Vert{v(z)}\Vert^2_2 ]  + \varepsilon\int [{\mathcal{Q}(z) v^\top(z)} \nabla\dfrac{\delta \mathbb{D}_{\rm{KL}}[\mathcal{Q}(z)\Vert \mathcal{P}(z|\mathcal{D})]}{\delta \mathcal{Q}(z)}]\mathrm{d}z\\
      = & \frac{\varepsilon}{2}\mathbb{E}_{\mathcal{Q}(z)}\{\Vert v(z) + \nabla\frac{\delta \mathbb{D}_{\rm{KL}}[\mathcal{Q}(z)\Vert\mathcal{P}(z\vert\mathcal{D})]}{\delta \mathcal{Q}(z)} \Vert_2^2 \}.
\end{aligned}
\end{equation}

Notably, similar results can be obtained from the ELBO. To better support this statement, we have the following relationship between the KL divergence and ELBO:
\begin{equation}
   \text{ELBO}[\mathcal{Q}(z)] = \log{\mathcal{P}(\mathcal{D})} - \mathbb{D}_{\rm{KL}}[\mathcal{Q}(z)\Vert \mathcal{P}(z|\mathcal{D})] .
\end{equation}
Thus, for the velocity field that maximizes the ELBO, we have the following result:
\begin{equation}
   \frac{\delta   \text{ELBO}[\mathcal{Q}(z)]}{\delta \mathcal{Q}(z)} = -  \frac{\delta \mathbb{D}_{\rm{KL}}[\mathcal{Q}(z)\Vert \mathcal{P}(z|\mathcal{D})]}{\delta \mathcal{Q}(z)}.
\end{equation}
On this basis,~\Cref{eq:klDivergenceResult2} can be reformulated as follows:
\begin{equation}
\begin{aligned}
   & \text{ELBO}[\mathcal{Q}_{\boldsymbol{T}}(z)] \\
    =   &\text{ELBO}[\mathcal{Q}(z)] + \varepsilon \int{\mathcal{Q}_\varepsilon(z)v^\top(z) \nabla \frac{\delta \text{ELBO}[\mathcal{Q}_{\varepsilon}(z)]}{\delta \mathcal{Q}_\varepsilon(z) }\mathrm{d}z} \\
    = &\text{ELBO}[\mathcal{Q}(z)] - \varepsilon \int{\mathcal{Q}_\varepsilon(z)v^\top(z) \nabla \frac{\delta \mathbb{D}_{\rm{KL}}[\mathcal{Q}_\varepsilon(z)\Vert \mathcal{P}(z\vert\mathcal{D})]}{\delta \mathcal{Q}_\varepsilon(z) }\mathrm{d}z} .
\end{aligned}
\end{equation}
Consequently, we arrive at a similar inequality governing the ascent direction:
\begin{equation}%\tag{S20}
    \begin{aligned}
        \text{ELBO}[\mathcal{Q}_{\boldsymbol{T}}(z)] - \frac{1}{2\varepsilon} \mathcal{W}_2^2(\mathcal{Q}_{\boldsymbol{T}}(z),\mathcal{Q}(z)) -  \text{ELBO}[\mathcal{Q}(z)]
   \ge - \frac{\varepsilon}{2}\mathbb{E}_{\mathcal{Q}(z)}\{\Vert v(z) + \nabla\frac{\delta \mathbb{D}_{\rm{KL}}[\mathcal{Q}(z)\Vert\mathcal{P}(z\vert\mathcal{D})]}{\delta \mathcal{Q}(z)} \Vert_2^2 \}.
    \end{aligned}
\end{equation}
% Consequently, whether we optimize the ELBO or equivalently minimize the KL divergence, we obtain the same velocity field.
This confirms that maximizing the ELBO with a Wasserstein penalty yields the same optimal velocity field $v^*(z) = \nabla \frac{\delta \text{ELBO}}{\delta \mathcal{Q}}$ (which equals $-\nabla \frac{\delta \text{KL}}{\delta \mathcal{Q}}$) as minimizing the KL divergence.
% As such, similar result can be obtained:
% \begin{equation}
%     \begin{aligned}
%         \text{ELBO}[\mathcal{Q}_{\boldsymbol{T}}(z)] + \frac{1}{2\varepsilon} \mathcal{W}_2^2(\mathcal{Q}_{\boldsymbol{T}}(z),\mathcal{Q}(z)) -  \text{ELBO}[\mathcal{Q}(z)]
%    \ge  \frac{\varepsilon}{2}\mathbb{E}_{\mathcal{Q}(z)}\{\Vert v(z) + \nabla\frac{\delta \mathbb{D}_{\rm{KL}}[\mathcal{Q}(z)\Vert\mathcal{P}(z\vert\mathcal{D})]}{\delta \mathcal{Q}(z)} \Vert_2^2 \}.
%     \end{aligned}
% \end{equation}
% Consequently, whether we optimize the ELBO or equivalently minimize the KL divergence, we obtain the same velocity field.

\color{black}
\subsection{Proof of~\Cref{thm:ConvergenceControlEpsilon}}
\begin{lemma*}[2] 
Suppose that $ \Vert v(z)\Vert\le \mathscr{A}  $ and $ \Vert \nabla v(z)\Vert\le \mathscr{B}  $. 
Let $\{ \mathcal{Q}_t(z)\}_{t=1}^{\mathrm{T}}$ denote the sequence of variational distributions generated by the KProx algorithm. Then, when $\varepsilon=\frac{1}{\sqrt{\mathrm{T}}} $, the following inequality holds:
\begin{equation}\tag{32}% \label{eq:convergenceTheorem}
\begin{aligned}
\lim_{\mathrm{T}\to\infty}\mathbb{D}_{\rm}[\mathcal{Q}_\mathrm{T}(z)\Vert\mathcal{P}(z\vert\mathcal{D})] = 0
 % & \dfrac{1}{{\mathrm{T}}}\sum_{t=1}^{\mathrm{T}}\mathbb{D}_{\rm{KL}}[\mathcal{Q}_t(z)\Vert \mathcal{P}(z|x)]
 % %   \le &  \frac{1}{\sqrt{\mathrm{T}}}\{ \mathbb{D}_{\rm{KL}}[\mathcal{Q}_{t+\varepsilon}(z)\Vert \mathcal{P}(z|x)]- \mathbb{D}_{\rm{KL}}[\mathcal{Q}_{t}(z)\Vert \mathcal{P}(z|x)]\} + \frac{2\mathcal{C}}{\sqrt{\mathrm{T}}}\\
 %    \leq
 %    \mathcal{O}(\frac{1}{\sqrt{\mathrm{T}}}).
\end{aligned}   
\end{equation}
\end{lemma*} 

\begin{proof}
The KL divergence at $t+1$ has the following relationship between the KL divergence at $t$:
\begin{equation}% \label{eq:expansionKPlus1Time}
\begin{aligned}
& \mathbb{D}_{\rm{KL}}[\mathcal{Q}_{t+1}(z)\Vert \mathcal{P}(z|\mathcal{D})]= \mathbb{D}_{\rm{KL}}[\mathcal{Q}_{t}(z)\Vert \mathcal{P}(z|\mathcal{D})] -\varepsilon \mathbb{E}_{\mathcal{Q}_t(z)}\{v^\top(z) \nabla\frac{\delta \mathbb{D}_{\rm{KL}}[\mathcal{Q}_t(z)\Vert \mathcal{P}(z\vert\mathcal{D})]}{\delta \mathcal{Q}_t(z)} \} + \mathcal{O}(\varepsilon^2).
%\\
%  \le & \mathbb{D}_{\rm{KL}}[\mathcal{Q}_{t}(z)\Vert \mathcal{P}(z|\mathcal{D})] \\
% & -\varepsilon \mathbb{E}_{\mathcal{Q}_t(z)}\{v^\top(z) \nabla\frac{\delta \mathbb{D}_{\rm{KL}}[\mathcal{Q}_t(z)\Vert \mathcal{P}(z\vert\mathcal{D})]}{\delta \mathcal{Q}_t(z)} \} + \mathcal{C}\varepsilon^2.
\end{aligned}
\end{equation}
The $\mathcal{O}(\varepsilon^2)$ contains the following terms:
\begin{equation}
\{ C_1\, \mathbb E_q\|v\|^2 + C_2\, \mathbb E_q\|\nabla v\|_F^2 + C_3\, [\sup(-\nabla^2\log{\mathcal{P}(z\vert\mathcal{D})})]\, \mathbb E_q\|v\|^2\}\cdot \varepsilon^2,
\end{equation}
where $C_1$, $C_2$, and $C_3$ are constant. Since we assume that $ \Vert v(z)\Vert\le \mathscr{A}  $ and $ \Vert \nabla v(z)\Vert\le \mathscr{B}  $, we know that the $\mathcal{O}(\varepsilon^2)$ is bounded by the following inequality: 
\begin{equation}
    \mathcal{O}(\varepsilon^2)\le \mathcal{C}\varepsilon^2,
\end{equation}
where $\mathcal{C}$ is a positive constant. In addition, since $h_{\text{RKHS}}(z)$ is the weak mode of convergence for $\nabla\log{\mathcal{Q}_t(z)}$, we get:
\begin{equation}% \label{eq:innerProductInequality}
\begin{aligned}
  &  %\int{
 \mathbb{E}_{\mathcal{Q}_t(z)}\{[ \nabla\frac{\delta \mathbb{D}_{\rm{KL}}[\mathcal{Q}_t(z)\Vert \mathcal{P}(z\vert\mathcal{D})]}{\delta \mathcal{Q}_t(z)}]^\top [ \nabla\frac{\delta \mathbb{D}_{\rm{KL}}[\mathcal{Q}_t(z)\Vert \mathcal{P}(z\vert\mathcal{D})]}{\delta \mathcal{Q}_t(z)}]\}
  %\mathrm{d}z} 
  \simeq   \mathbb{E}_{\mathcal{Q}_t(z)}\{[ \nabla\frac{\delta \mathbb{D}_{\rm{KL}}[\mathcal{Q}_t(z)\Vert \mathcal{P}(z\vert\mathcal{D})]}{\delta \mathcal{Q}_t(z)}]^\top [ \nabla\log{\mathcal{P}(z\vert\mathcal{D})} - h_{\text{RKHS}}(z)]\}
%\mathrm{d}z}
.
\end{aligned}
\end{equation}
% $v^\top\nabla\frac{\delta \mathbb{D}_{\rm{KL}}[\mathcal{Q}_t(z)\Vert \mathcal{P}(z\vert\mathcal{D})]}{\delta \mathcal{Q}_t(z)}$ 
% since we restrict the $h(z)$ within RKHS, we get the following inequality from the perspective of inner product similarity:
% \begin{equation}\label{eq:innerProductInequality}
% \begin{aligned}
%   &  [ \nabla\frac{\delta \mathbb{D}_{\rm{KL}}[\mathcal{Q}_t(z)\Vert \mathcal{P}(z\vert\mathcal{D})]}{\delta \mathcal{Q}_t(z)}]^\top [ \nabla\frac{\delta \mathbb{D}_{\rm{KL}}[\mathcal{Q}_t(z)\Vert \mathcal{P}(z\vert\mathcal{D})]}{\delta \mathcal{Q}_t(z)}] \\
% \ge  & [ \nabla\frac{\delta \mathbb{D}_{\rm{KL}}[\mathcal{Q}_t(z)\Vert \mathcal{P}(z\vert\mathcal{D})]}{\delta \mathcal{Q}_t(z)}]^\top [ \nabla\log{\mathcal{P}(z\vert\mathcal{D})} - h_{\text{RKHS}}(z)].
% \end{aligned}
% \end{equation}
Plugging~\Cref{eq:innerProductInequality} into~\Cref{eq:expansionKPlus1Time}, the following result can be obtained:
\begin{equation}% \label{eq:inequalityResult2}
\begin{aligned}
 &   \mathbb{D}_{\rm{KL}}[\mathcal{Q}_{t+1}(z)\Vert \mathcal{P}(z|\mathcal{D})] \le  \mathbb{D}_{\rm{KL}}[\mathcal{Q}_{t}(z)\Vert \mathcal{P}(z|\mathcal{D})]   
  %  &-\varepsilon \mathbb{E}_{\mathcal{Q}_t(z)}\{v^\top(z) \nabla\frac{\delta \mathbb{D}_{\rm{KL}}[\mathcal{Q}_t(z)\Vert \mathcal{P}(z\vert\mathcal{D})]}{\delta \mathcal{Q}_t(z)} \} + \mathcal{C}\varepsilon^2\\
 -\varepsilon \mathbb{E}_{\mathcal{Q}_t(z)}\{\Vert  \nabla\frac{\delta \mathbb{D}_{\rm{KL}}[\mathcal{Q}_t(z)\Vert \mathcal{P}(z\vert\mathcal{D})]}{\delta \mathcal{Q}_t(z)} \Vert_2^2 \} + \mathcal{C}\varepsilon^2\\
 \Rightarrow & \\
 &  \mathbb{D}_{\rm{KL}}[\mathcal{Q}_{t+1}(z)\Vert \mathcal{P}(z|\mathcal{D})] -\mathbb{D}_{\rm{KL}}[\mathcal{Q}_{t}(z)\Vert \mathcal{P}(z|\mathcal{D})]   \le 
  %  &-\varepsilon \mathbb{E}_{\mathcal{Q}_t(z)}\{v^\top(z) \nabla\frac{\delta \mathbb{D}_{\rm{KL}}[\mathcal{Q}_t(z)\Vert \mathcal{P}(z\vert\mathcal{D})]}{\delta \mathcal{Q}_t(z)} \} + \mathcal{C}\varepsilon^2\\
 -\varepsilon \mathbb{E}_{\mathcal{Q}_t(z)}\{\Vert  \nabla\frac{\delta \mathbb{D}_{\rm{KL}}[\mathcal{Q}_t(z)\Vert \mathcal{P}(z\vert\mathcal{D})]}{\delta \mathcal{Q}_t(z)} \Vert_2^2 \} + \mathcal{C}\varepsilon^2.
\end{aligned}
\end{equation}
Cascading~\Cref{eq:inequalityResult2} from $t=1$ to $\mathrm{T}$, we get:
\begin{equation}
\begin{aligned}
   & \mathbb{D}_{\rm{KL}}[\mathcal{Q}_{\mathrm{T}}(z)\Vert \mathcal{P}(z|\mathcal{D})]  -  \mathbb{D}_{\rm{KL}}[\mathcal{Q}_{0}(z)\Vert \mathcal{P}(z|\mathcal{D})]\le -\varepsilon \sum_{t=1}^{\mathrm{T}}{\mathbb{E}_{\mathcal{Q}_t(z)}\{\Vert  \nabla\frac{\delta \mathbb{D}_{\rm{KL}}[\mathcal{Q}_t(z)\Vert \mathcal{P}(z\vert\mathcal{D})]}{\delta \mathcal{Q}_t(z)} \Vert_2^2 \} } + 2\mathrm{T}\mathcal{C}\varepsilon^2\\
    \Rightarrow &  \varepsilon \sum_{t=1}^{\mathrm{T}}{\mathbb{E}_{\mathcal{Q}_t(z)}\{\Vert  \nabla\frac{\delta \mathbb{D}_{\rm{KL}}[\mathcal{Q}_t(z)\Vert \mathcal{P}(z\vert\mathcal{D})]}{\delta \mathcal{Q}_t(z)} \Vert_2^2 \} }\le 
    \mathbb{D}_{\rm{KL}}[\mathcal{Q}_{0}(z)\Vert \mathcal{P}(z|\mathcal{D})] -  \mathbb{D}_{\rm{KL}}[\mathcal{Q}_{\mathrm{T}}(z)\Vert \mathcal{P}(z|\mathcal{D})] 
    + 2\mathrm{T}\mathcal{C}\varepsilon^2
    \\
        \Rightarrow & \frac{1}{\mathrm{T}} \sum_{t=1}^{\mathrm{T}}{\mathbb{E}_{\mathcal{Q}_t(z)}\{\Vert  \nabla\frac{\delta \mathbb{D}_{\rm{KL}}[\mathcal{Q}_t(z)\Vert \mathcal{P}(z\vert\mathcal{D})]}{\delta \mathcal{Q}_t(z)} \Vert_2^2 \} }\le 
    \frac{\mathbb{D}_{\rm{KL}}[\mathcal{Q}_{0}(z)\Vert \mathcal{P}(z|\mathcal{D})] -  \mathbb{D}_{\rm{KL}}[\mathcal{Q}_{\mathrm{T}}(z)\Vert \mathcal{P}(z|\mathcal{D})] 
   }{\mathrm{T}\varepsilon} + 2\mathcal{C}\varepsilon
    \\
 \overset{\varepsilon = \frac{1}{\sqrt{\mathrm{T}}}}{\Rightarrow} & \frac{1}{\mathrm{T}}\sum_{t=1}^{\mathrm{T}}{\mathbb{E}_{\mathcal{Q}_t(z)}\{\Vert  \nabla\frac{\delta \mathbb{D}_{\rm{KL}}[\mathcal{Q}_t(z)\Vert \mathcal{P}(z\vert\mathcal{D})]}{\delta \mathcal{Q}_t(z)} \Vert_2^2 \} }   \le \underbrace{\frac{\mathbb{D}_{\rm{KL}}[\mathcal{Q}_0(z)\Vert \mathcal{P}(z\vert\mathcal{D})] - \mathbb{D}_{\rm{KL}}[\mathcal{Q}_{\mathrm{T}}(z)\Vert \mathcal{P}(z\vert\mathcal{D})]}{\sqrt{\mathrm{T}}} + \frac{2\mathcal{C}}{\sqrt{\mathrm{T}}}}_{\mathcal{O}(\frac{1}{\sqrt{\mathrm{T
    }}})}. %\\
    %& .
\end{aligned}
\end{equation}
Consequently, we get the following inequality:
\begin{equation}
    0\le\frac{1}{\mathrm{T}}\sum_{t=1}^{\mathrm{T}}{\mathbb{E}_{\mathcal{Q}_t(z)}\{\Vert  \nabla\frac{\delta \mathbb{D}_{\rm{KL}}[\mathcal{Q}_t(z)\Vert \mathcal{P}(z\vert\mathcal{D})]}{\delta \mathcal{Q}_t(z)} \Vert_2^2 \} }   \le {\frac{\mathbb{D}_{\rm{KL}}[\mathcal{Q}_0(z)\Vert \mathcal{P}(z\vert\mathcal{D})] - \mathbb{D}_{\rm{KL}}[\mathcal{Q}_{\mathrm{T}}(z)\Vert \mathcal{P}(z\vert\mathcal{D})]}{\sqrt{\mathrm{T}}} + \frac{2\mathcal{C}}{\sqrt{\mathrm{T}}}},%_{\mathcal{O}(\frac{1}{\sqrt{\mathrm{T
   % }}})} 
\end{equation}
which indicates that:
\begin{equation}
  \lim_{\mathrm{T}\to\infty}{\frac{1}{\mathrm{T}}\sum_{t=1}^{\mathrm{T}}{\mathbb{E}_{\mathcal{Q}_t(z)}\{\Vert  \nabla\frac{\delta \mathbb{D}_{\rm{KL}}[\mathcal{Q}_t(z)\Vert \mathcal{P}(z\vert\mathcal{D})]}{\delta \mathcal{Q}_t(z)} \Vert_2^2 \} }} = 0 \Rightarrow  \lim_{\mathrm{T}\to\infty}  \nabla \log{\mathcal{Q}(z)} -\nabla\log{\mathcal{P}(z\vert\mathcal{D})} = 0,
\end{equation}

Based on this, we define the error function $\mathscr{E}(z) \coloneqq \log\mathcal{Q}(z) - \log\mathcal{P}(z\vert \mathcal{D}
)$ and obtain the following result:
\begin{equation}
\nabla \mathscr{E}(z) = \nabla f(z) - \nabla g(z) = 0.
\end{equation}
So the gradient of $\mathscr{E}(z)$ is zero everywhere on $\mathbb{R}^{\mathrm{D}_{\rm{LV}}}$. This implies that along any smooth path $\upgamma(t)$ contained in $\mathbb{R}^{\mathrm{D}_{\rm{LV}}}$, we have the following result:
\begin{equation}
\frac{\mathrm{d}}{\mathrm{d}t} \mathscr{E}(\upgamma(t)) 
= [\nabla \mathscr{E}(\upgamma(t))]^\top [\frac{\mathrm{d}\upgamma(t)}{\mathrm{d}t}]
= 0.
\end{equation}

Therefore $ \mathscr{E}(\upgamma(t))$ is constant along the path. Since the region is connected, any two points can be joined by such a path, so $ \mathscr{E}(z)$ must be the same constant $\mathcal{C}$ throughout $\mathbb{R}^{\mathrm{D}_{\rm{LV}}}$:
\begin{equation}
\mathscr{E}(z) = \mathcal{C} \quad \Rightarrow \quad  \log\mathcal{Q}(z) -  \log\mathcal{P}(z\vert\mathcal{D})= \mathcal{C}.
\end{equation}
Hence, we have the following result:\color{black}
\begin{equation}
    \mathcal{Q}_{\mathrm{T}}(z) = \mathcal{C}\mathcal{P}(z\vert\mathcal{D})~~\text{when}~~\mathrm{T}\to\infty \Rightarrow \lim_{\mathrm{T}\to\infty}\mathcal{Q}_{\mathrm{T}}(z)\propto\mathcal{P}(z\vert\mathcal{D}).
\end{equation}
Since $\lim_{\mathrm{T}\to\infty}\mathcal{Q}_{\mathrm{T}}(z)\propto\mathcal{P}(z\vert\mathcal{D})$, we arrive at the desired result defined by~\Cref{eq:convergenceTheorem}.
\end{proof}

\section{Additional Theoretical and Empirical Discussions}
\subsection{Discussions of Wasserstein Proximal Recursion Scheme from Bias and Variance}
\noindent\textbf{Discussions on the ``bias'':}
 Let $\mathcal{P}(z\vert\mathcal{D})$ be a Boltzmann distribution form: $\mathcal{P}(z\vert\mathcal{D})\propto e^{-V(z)}$, and consider the energy functional as follows:
\begin{equation}
\mathbb{D}_{\rm{KL}}[\mathcal{Q}(z)\|\mathcal{P}(z\vert\mathcal{D})]=\int \mathcal{Q}(z)\log\frac{\mathcal{Q}(z)}{\mathcal{P}(z\vert\mathcal{D})}\mathrm{d}z
=\int\mathcal{Q}(z)\log\mathcal{Q}(z)\mathrm{d}z+\int V(z)\mathcal{Q}(z)\mathrm{d}z+\text{const}.
\end{equation}
The Wasserstein proximal recursion is
\begin{equation}\label{eq:proximalError}
\mathcal{Q}_{k+1}
=\mathop{\arg\min}_{\mathcal{Q}\in\mathcal P_2}\Big\{
\mathbb{D}_{\rm{KL}}[\mathcal{Q}\|\mathcal{P}]+\frac{1}{2\varepsilon}\mathcal{W}_2^2(\mathcal{Q},\mathcal{Q}_k)
\Big\}.
\end{equation}
The Wasserstein term constrains the next iterate to stay close to $\mathcal{Q}_k$ in transport distance, improving stability.

As $\varepsilon\to 0$, the discrete scheme~\Cref{eq:proximalError} converges to the exactly the following partial differential equation:
\begin{equation}
\frac{\partial \mathcal{Q}_t(z)}{\partial t}
=\nabla\cdot[\mathcal{Q}_t(z) \nabla V(z)]+\nabla\cdot\nabla\mathcal{Q}_t(z).
\end{equation}
Hence, for any fixed $\varepsilon>0$, we have:
\begin{enumerate}%[leftmargin=*]
    \item{Follows a time-discretized trajectory rather than the exact continuous partial differential equation.}
    \item{This is the deterministic discretization bias, typically shrinking as $\varepsilon\to 0$.}
\end{enumerate}
As such, we have the following conclusions:
\begin{itemize}%[leftmargin=*]
    \item{Larger $\varepsilon$: stronger proximal effect $\Rightarrow$ more stable iterates, but larger discretization bias.}
 \item{Smaller $\varepsilon$: smaller bias, but usually needs more outer steps and may be less stable in practice.}
\end{itemize}

Notably, the term ``bias'' is algorithmic approximation error to the continuous differential equation, not statistical bias.

\noindent\textbf{Discussions on the ``variance'':} The ideal update~\Cref{eq:proximalError} is deterministic. Variance comes from approximating $\mathcal{Q}_t(z)$. Specifically, $\mathcal{Q}_k$ is represented by $\ell$ particles, $\mathcal{Q}_t(z)\approx \frac{1}{\ell}\sum_{i=1}^\ell \delta_{z_{i,t}}$, then both the velocity filed computation and the OT computation involve finite-sample noise $\Rightarrow$ variance decreases with larger $\ell$.

\subsection{Key Differences and Advantages with Previous NPLVMs}

Comparing previous NPLVMs with the proposed KProxNPLVM, the main differences can be summarized as follows:
\begin{enumerate}% [leftmargin=*]

\item \textbf{Assumption on $\mathcal{Q}(z)$ (theoretical property):}
Prior works typically restrict the approximate posterior $\mathcal{Q}(z)$ to a predefined parametric family of normalized distributions. As shown in Lemma~1, this restriction can induce an intrinsic approximation error floor, which may limit the achievable performance in downstream soft sensor modeling.
In contrast, KProx optimizes $\mathcal{Q}(z)$ over the Wasserstein space $\mathscr{P}_2(\mathrm{D}_{\rm{LV}})$ (i.e., the set of distributions with finite second moment), substantially enlarging the feasible set for posterior approximation.
Under mild conditions, our analysis (Theorem~2) establishes the convergence of the resulting proximal updates, supporting improved approximation capability and more reliable latent inference.

\item \textbf{Training procedure (algorithmic formulation):}
Most existing methods optimize the inference network parameters $\varphi$ and the generative model parameters $\theta$ jointly by maximizing the standard ELBO.
KProxNPLVM instead adopts a sequential/proximal training strategy: it first refines the latent posterior via a Wasserstein proximal step and then updates $\theta$ and $\varphi$ accordingly.
This decoupling mitigates error propagation from inaccurate latent inference to network training, and empirically leads to more stable optimization and better predictive performance.

\item \textbf{Relevance to soft sensor modeling (application scenario):}
Soft sensor modeling often involves multi-modal latent factors. By allowing a more flexible posterior family and improving inference stability through the Wasserstein proximal regularization, KProxNPLVM yields more robust latent representations, which translate into improved prediction accuracy and stability in soft sensor tasks (see the added empirical comparisons).
\end{enumerate}

\subsection{Practical deployment in industrial environments}
% We thank the reviewer for encouraging a clearer discussion of deployment beyond offline evaluation. In the revised manuscript/supplementary material, we add the following practical notes on (i) online/near-real-time feasibility and (ii) retraining under non-stationarity.
In this subsection, we provide related discussions of deployment for the proposed KProxNPLVM beyond offline evaluation.
\begin{enumerate}
    \item{\textbf{Online feasibility:} After training, inference for a new sample only requires a forward pass through the encoder and the predictor; thus the per-sample latency is comparable to standard AVI-based NPLVM baselines. The Wasserstein-proximal updates are performed during training and do not introduce additional computation at test time. Therefore, the proposed method is feasible for online or near-real-time deployment, provided that the offline training/update budget is acceptable for the empirical real industrial application.}
    \item{\textbf{Retraining frequency under non-stationary processes:} In practice, we consider two commonly used modes:
\begin{itemize}% [leftmargin=*]
    \item \textbf{Periodic retraining:} For example, daily or weekly training when the process exhibits slow drift.
    \item \textbf{Triggered retraining:} For example, we require a triggered retraining when the drift is detected explictly, for example, a sustained increase in prediction residuals or a distribution-shift statistic.
\end{itemize}
As a rule of thumb, retraining can be initiated when a monitoring metric exceeds a pre-defined threshold for $\mathcal{K}$ consecutive windows, or when performance degrades by more than $\updelta\%$ relative to a recent baseline. The appropriate retraining policy depends on the drift rate, sensor maintenance schedule, and labeling availability in the plant.
    
    }
\end{enumerate}
%     \item{\textbf{}}
% \end{itemize}

\section{Detailed Experimental Information}\label{sec:DatasetDescription}
\subsection{Dataset Descriptions}
In this section, the background of the selected three datasets namely debutanizer column (DBC), Carbon-Dioxide Absorber Column (CAC), and Catalysis Shift Conversion (CSC) unit are delineated in detailed. 
% We select two datasets namely debutanizer column (DC) and catalytic shift conversion (CSC) as delineated in Fig~~\ref{TotalFlowsheet} (a) and (b), respectively. Note that, the DC dataset is a benchmark dataset. More details about the two datasets and the symbols in Fig.~\ref{TotalFlowsheet} are given in Supplementary Material. The basic flowsheet of our two datasets are given in Fig.~\ref{TotalFlowsheet} (a) and (b), respectively.

\subsubsection{DBC}
Fig. \ref{DCFlowsheet} presents the flowsheet of the DBC, a benchmark dataset~\cite{fortuna2005soft}.
The DBC is required to maximize the pentane (C5) content in the overheads distillate and minimize the butane (C4) content in the bottom flow simultaneously. To measure the butane concentration from the bottom flow in real-time for the sake of improving downstream product quality, seven covariates marked in red zone in Fig. \ref{DCFlowsheet} are chosen for soft sensor modeling. The detailed descriptions of the covariates are given in Table \ref{ProcessVariablesDebutanizer}. 
\begin{figure}[!h]
  \centering
  \includegraphics[scale=0.32]{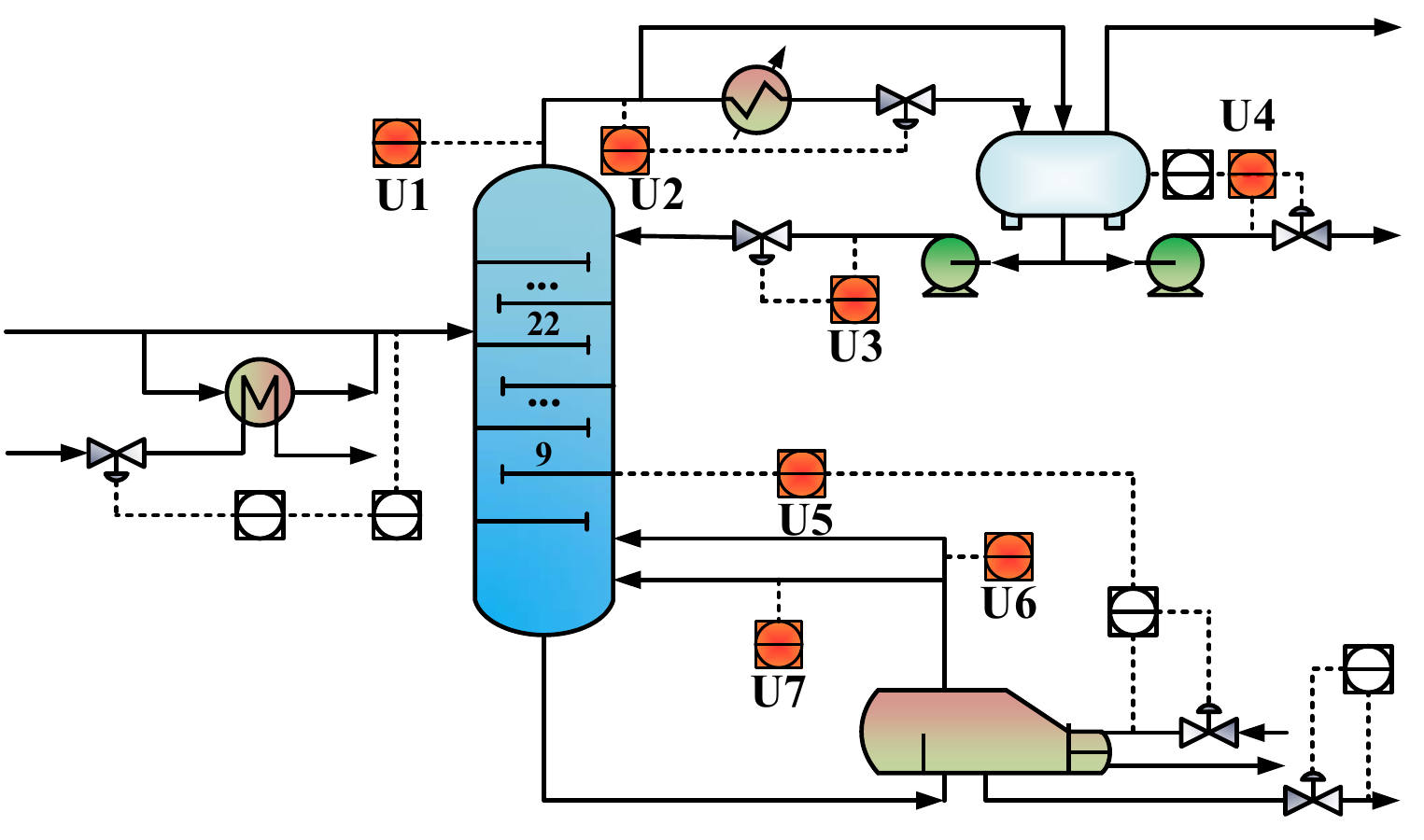}
  \caption{The flowsheet of DBC}
  \label{DCFlowsheet}
\end{figure}
\begin{table}[!h]
  % \vspace{-0.5cm}
  \caption{The process variables in the DBC}\label{ProcessVariablesDebutanizer}
  % \vspace{-0.2cm}
  \centering 
  \begin{tabular}{lllll}
\toprule
  Process variables    &      Unit          &  Description   &  \\ \midrule
  U1                   &   \textcelsius        &  Top temperature    &  \\
  U2              &    $\textrm{kg}\cdot\textrm{cm}^{-2}$        & Top pressure   &  \\
  U3           &   $\textrm{m}^3\cdot\textrm{h}^{-1}$               & Reflux flow rate   &  \\
U4             &  $\textrm{m}^3\cdot\textrm{h}^{-1}$            &  Top distillate rate   &  \\
  U5              & \textcelsius      & Temperature of the 9th tray &  \\
 U6              &  \textcelsius       &  Bottom temperature A &  \\
U7        &    \textcelsius       & Bottom temperature B   &  \\ \bottomrule
  \end{tabular}
  \vspace{-0.0cm}
\end{table}

Based on~\cite{fortuna2005soft}, we extend the process variables and quality variable into 13 dimensions based on the following equation with the consideration of process time-delay:
\begin{equation}\label{DCFortuna}
{\left[ \begin{aligned}
  &\text{U1}(t),\text{U2}(t),\text{U3}(t),\text{U4}(t),\text{U5}(t),\text{U5}(t - 1), \text{U5}(t - 2),\text{U5}(t - 3),{{\left( {\text{U1}(t) + \text{U2}(t)} \right)} \mathord{\left/
 {\vphantom {{\left( {\text{U1}(t) + \text{U2}(t)} \right)} 2}} \right.
 \kern-\nulldelimiterspace} 2}, \hfill  \text Y(t - 1),\text Y(t - 2),\text Y(t - 3),\text Y(t - 4) \hfill \\ 
\end{aligned}  \right]^\top}.
\end{equation}

\subsubsection{CAC}

\indent Fig. \ref{CO2Absorber} presents the flowsheet of the CAC~\cite{shen2020nonlinear}. 
The CAC is a vital equipment in ammonia synthesis process to handle the carbon-dioxide by-product 
in the hydrogen from upstream unit. The sodium hydroxide solvent is chosen to 
be absorption liquid and the corresponding chemical reaction can be given in \eqref{CO2Reaction}:
\begin{equation}
  % \left\{ \begin{array}{l}
  %   \text{CO}_2 + 2\text{KOH} = \text{K}_2\text{CO}_3 + \text{H}_2\text{O}\\
  %   \text{K}_2\text{CO}_3 + \text{H}_2\text{O} +\text{CO}_2 = 2\text{KHCO}_3
  %   \end{array} . \right.\
  \left\{ {\begin{array}{*{20}{c}}
    { \text{CO}_2 + 2\text{KOH} = \text{K}_2\text{CO}_3 + \text{H}_2\text{O}}\\
    { \text{K}_2\text{CO}_3 + \text{H}_2\text{O} +\text{CO}_2 = 2\text{KHCO}_3}
    \end{array}}. \right.
    \label{CO2Reaction}
\end{equation}
\indent To eliminate the carbon-dioxide concentration in the hydrogen stream for promoting the product quality 
in downstream urea synthesis process, the carbon-dioxide concentration in the outlet gas should be monitored in real time. 
However, the gas chromatography for carbon-dioxide concentration measurement has large time delay. 

\begin{figure}[htbp]
  \centering
  \includegraphics[width=0.3\textwidth]{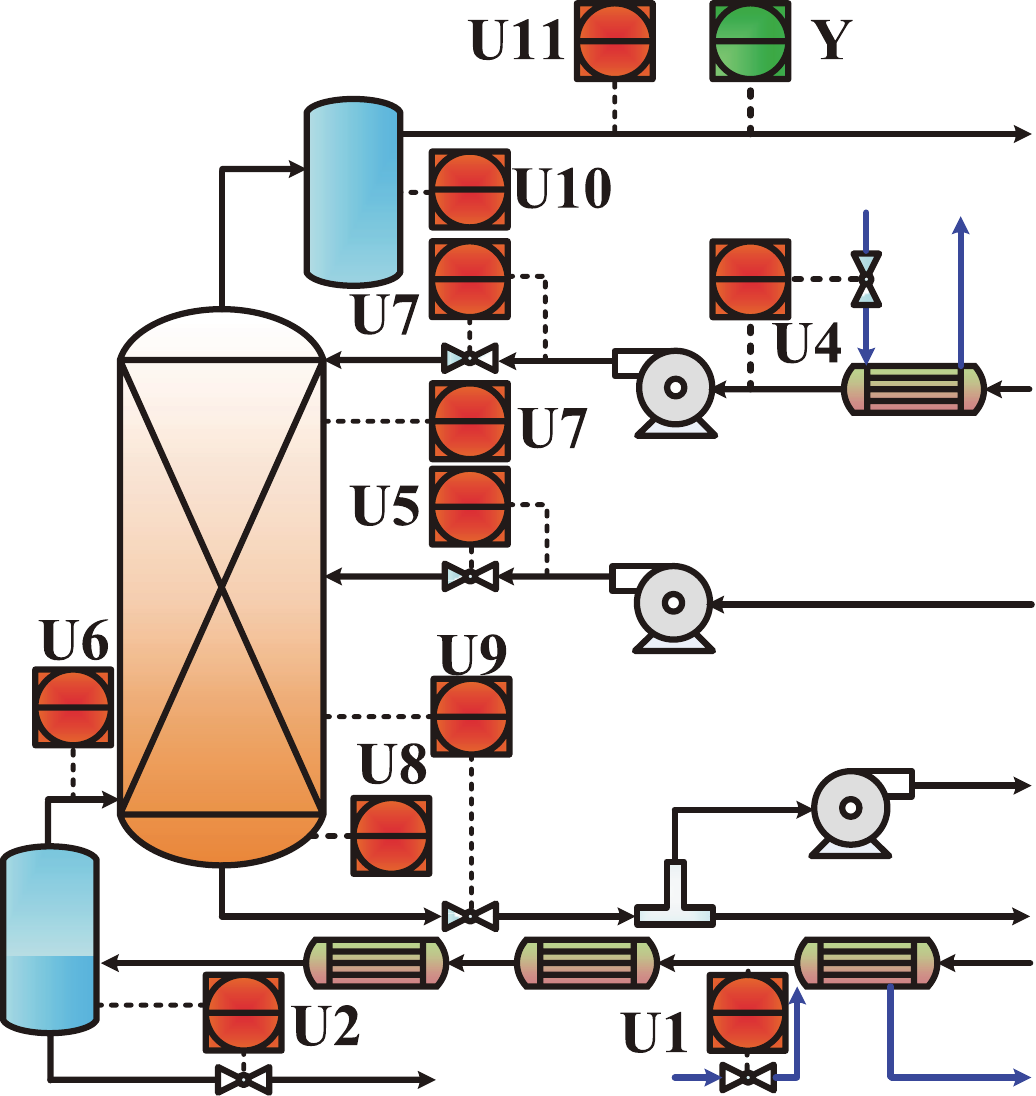}
  \caption{The flowsheet of CAC}
  \label{CO2Absorber}
\end{figure}

\indent To deal with measuring problem, and improve the control quality of the  carbon-dioxide absorber, 
inferential sensor models have been adopted to estimate the  carbon-dioxide concentration at the outlet stream of the absorber 
in real time. 
Even though the knowledge about the abosorber has been well studied from the perspective of unit operation, 
\textit{the rigorous process simulation is hard to construct due to the missing of thermodynamical electrolyte binary parameters}.
And thus, the data-driven model is the natural choice to construct the inferential sensor task.
For quality monitoring and control purposes, 
several hard sensors are installed on the plant to collect the samples for 
process variables, which can be used as secondary variables for inferential sensor. 
Eleven process variables marked in red zone are selected to construct the model. 
Table \ref{ProcessVariables} gives the detailed description of these variables. 
A total of 6,~000 samples are collected from the process. 

\begin{table}[htb]
  \centering
  \caption{Process variables for CAC}\label{ProcessVariables}
  \begin{tabular}{ll}
  \hline
  Input Variables & Description                          \\ \hline
  U1              & Pressure of inlet gas                \\
  U2              & Liquid level of buffer vessel        \\
  U3              & Temperature of inlet barren liquor   \\
  U4              & Flowrate of inlet lean solution      \\
  U5              & Flowrate of inlet semi-lean solution \\
  U6              & Temperature of inlet gas             \\
  U7              & Pressure drop of absorber            \\
  U8              & Temperature of rich solution         \\
  U9              & Liquid level of absorber             \\
  U10             & Liquid level of the separator        \\
  U11             & Pressure of outlet gas               \\
  Y               & $\mathrm{CO}_2$ concentration                    \\ \hline
  \end{tabular}
\end{table}

\subsubsection{CSC}
\begin{figure}[htbp]
  \centering
  \includegraphics[scale=0.32]{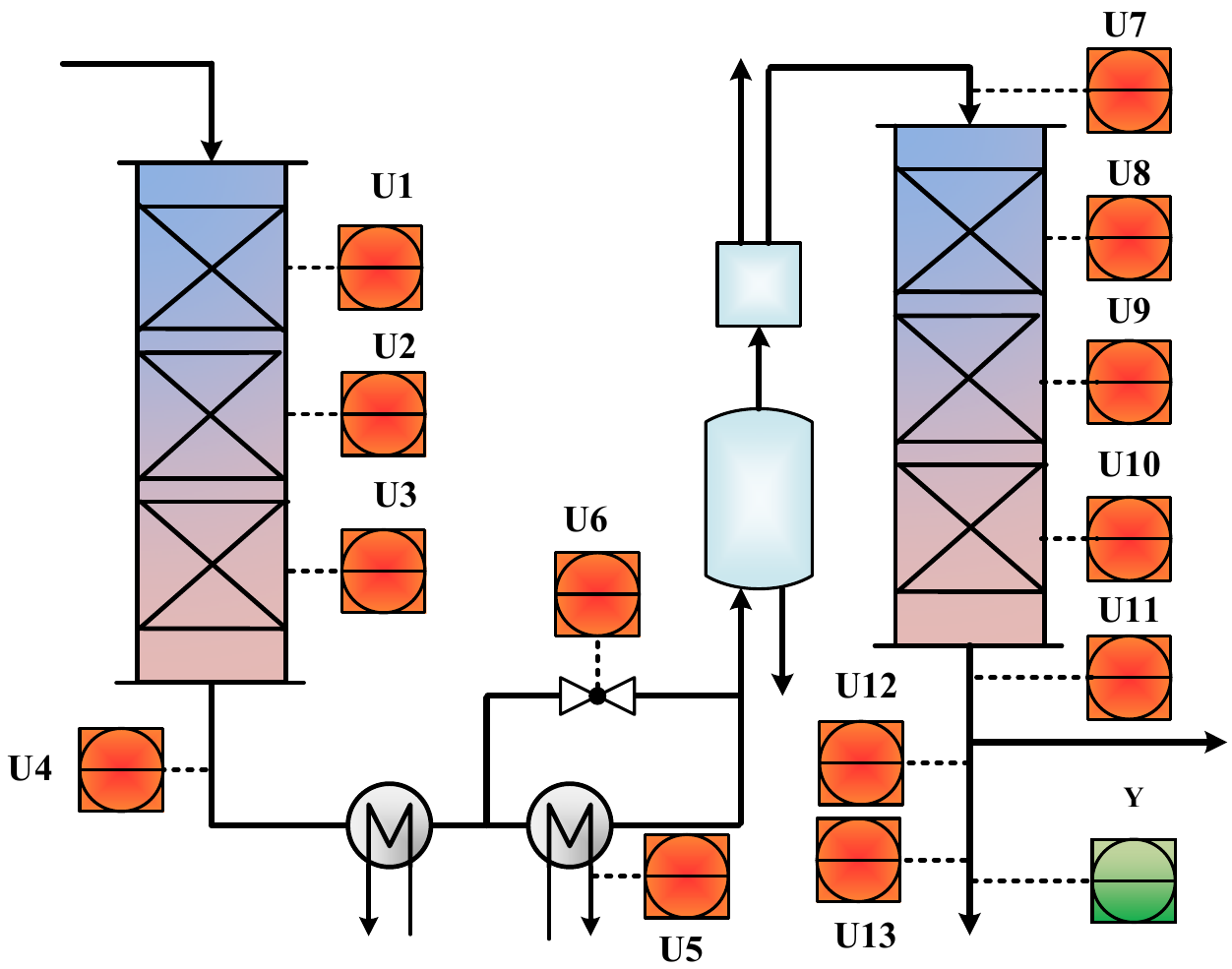}
  \caption{The flowsheet of CSC}
  \label{TotalFlowsheet}
\end{figure}

Fig. \ref{TotalFlowsheet} presents the flowsheet of CSC in an ammonia synthesis process~\cite{8894374}. The following heterogeneous catalytic reaction will take place in the fixed-bed reactors connected in series:
\begin{equation}\label{ReactionFormula}
    \textrm{CO} + {\textrm{H}_2}\textrm{O} \leftrightarrow \textrm{CO}_2 + {\textrm{H}_2}
    % \Delta \textrm{H}_{298\mathrm{K}}^0 = -41.4 \mathrm{kJ}\cdot\mathrm{{mol}}^{-1}
    .
\end{equation}

To estimate the carbon monoxide concentration marked in the green zone in real-time for the sake of meeting the technology requirement of carbon hydrogen ratio, 
  several hard sensors are installed on the section to collect process variables in real time.
  Thirteen covariates marked in the red zone in Fig. \ref{TotalFlowsheet} are chosen for inferential sensor modeling, and the detailed descriptions are given in Table \ref{CSCProcessVariables}. %(to align with Table \ref{ProcessVariablesDebutanizer}, we relabel the covariates with prefix ``U").

\begin{table}[htbp]
  % \vspace{-0.3cm}
  \caption{The process variables in the CSC process}
  % \vspace{-0.2cm}
  \centering 
  % \resizebox{\columnwidth}{!}{
\begin{threeparttable}
  \begin{tabular}{llll}
 \toprule
  Process variables & Description                  &  \\ \midrule
  U1                & High temperature bed temperature 1            &  \\
  U2                & High temperature bed temperature 2               &  \\
  U3                & High temperature bed temperature 3                  &  \\
  U4                & Outlet temperature of high temperature bed         &  \\
  U5                & Outlet temperature of cooling water   &  \\
  U6                & Split-gas temperature          &  \\
  U7                & Inlet temperature of low temperature bed         &  \\ 
  U8               & Low temperature bed temperature 1          &  \\ 
  U9                & Low temperature bed temperature 2         &  \\ 
  U10               & Low temperature bed temperature 3         &  \\ 
  U11               & Outlet temperature of low temperature bed         &  \\ 
  U12               & Outlet pressure of low temperature bed         &  \\ 
  U13               & Product gas pressure         &  \\ 
  Y                 & Carbon monoxide concentration         &  \\ \bottomrule
  \end{tabular}
\end{threeparttable}
% }
  \label{CSCProcessVariables}
  % \vspace{-0.3cm}
\end{table}  

% \newpage

\subsection{Detailed Information for Baseline Models}
\subsubsection{Baseline Models}
In this paper, the following baseline models are chosen to demonstrate the superiority of the proposed KProxNPLVM:
\begin{itemize}% [leftmargin=*]
    \item{\textbf{NPLVMs:} Supervised Nonlinear Probabilistic Latent Variable Regression (SNPLVR)~\cite{shen2020nonlinear}, Deep Bayesian Probabilistic Slow Feature Analysis (DBPSFA)~\cite{9625835}, Modified Unsupervised VAE-Supervised Deep VAE (MUDVAE-SDVAE)~\cite{xie2019supervised}, and Gaussian Mixture-Variational Autoencoder (GMM-VAE)~\cite{guo2021just}  .  }
    \item{\textbf{Non-NPLVMs:} Gated-Stacked Target-Related Autoencoder (GSTAE)~\cite{9174659}, Variable-wise weighted stacked autoencoder (VW-SAE)~\cite{8302941},  Deep Learning model with Dynamic Graph (DGDL)~\cite{9929274},   and iTransformer~\cite{liu2024itransformer}. }
  %  \item{\textbf{Non-PLVMs:} invert-Transformer (i-Transformer),  }
\end{itemize}

All experiments are conducted on a workstation equipped with an Intel Xeon E5 processor, 1 Nvidia GTX 3090 GPUs, and 64 GB of RAM. To maintain consistency and fairness across evaluations, the Adam optimizer, as detailed by~\cite{kingma2014adam}, is employed uniformly across all experimental runs. The model training and inference processes are carried out using Python 3.8 and PyTorch 1.13~\cite{TorchNips}. For all datasets, the data are sorted in ascending order by timestamp. On this basis, the first 60\% of the data is selected for training, the first 60\%$\sim$80\% of the data is selected for validation, and the rest of the data is selected for testing.

\subsubsection{Reasons for Baseline Models}\label{subsubsec:ReasonsForBaseline}
To demonstrate the effectiveness of the proposed KProxNPLVM, several NPLVMs designed for industrial inferential sensor modeling—specifically, SNPLVR, DBPSFA, MUDVAE-SDVAE, and GMM-VAE—are selected as baseline models. Additionally, to highlight the bottlenecks of current NPLVM-based inferential sensor modeling, other non-NPLVM models that have also been applied in industrial inferential sensor modeling—represented by DGDL, VW-SAE, GSTAE, and iTransformer—are adopted. Notably, \textit{iTransformer is the state-of-the-art model published in \emph{The Twelfth International Conference on Learning Representations} (ICLR-2024) for predictive-oriented application scenarios as exemplified by time-series forecasting.}

\subsubsection{Hyperparameter Settings}\label{subsubsec:HyperParamsSetting}
\begin{table}[htbp]
\caption{Hyper-parameter Settings for Baseline Comparison Experiment}\label{tab:hyperparams}
% \vspace{-0.3cm}
\centering
\begin{tabular}{l|cc|cc|cc}
\toprule
Dataset           & \multicolumn{2}{c|}{CAC} & \multicolumn{2}{c|}{DCB} & \multicolumn{2}{c}{CSC} \\ \midrule
Hyper-parameters  & $\mathcal{B}$  & $\eta$ & $\mathcal{B}$ & $\eta$ & $\mathcal{B}$ & $\eta$ \\ \midrule
SNPLVR            & 128             & 0.01   & 128             & 0.01      & 128             & 0.01    \\
DBPSFA            & 128             & 0.05  & 128             & 0.05   & 128             & 0.05     \\
MUDVAE-SDVAE      & 32             & 0.005   & 32             & 0.005    & 32             & 0.005     \\
%PDTM              & 128            & 0.01   & 128           & 0.01    &  &     \\

GMM-VAE           & 128             & 0.0001  & 128             & 0.0001  & 128             & 0.0001    \\
GSTAE             & 64             & 0.01   & 64            & 0.01    & 64            & 0.01     \\
VW-SAE            & 64             & 0.01   & 64            & 0.01    & 64            & 0.01     \\
DGDL              & 128            & 0.001  & 128            & 0.001   & 128            & 0.001   \\
iTransformer      & 64            & 0.001  & 64           & 0.001   & 64           & 0.001    \\
% UniFilter         & 32             & 0.01   & 32            & 0.01  &  &       \\
KProxNPLVM        & 128            & 0.01   & 128            & 0.01    & 128            & 0.01    \\ 
\bottomrule
\end{tabular}
\end{table}
For fairness, the multi-layer-perceptron for label prediction is set as $[10, 7, 5]$ based on references~\cite{shen2020nonlinear}. On this basis, the particle number $\ell$, discretization step-size $\varepsilon$, entropy regularization strength $\upepsilon$, and simulation time $\mathrm{T}$ for KProxNPLVM are set as 10, 0.1, 0.05, and 200, respectively. The component of Gaussian mixture model for GMM-VAE is set as 5. The embedding dimension for iTransformer and DGDL are set as 8 and 16, respectively. Other parameters like learning rate $\eta$ and batch size $\mathcal{B}$ for the baseline models and KProxNPLVM are listed in~\Cref{tab:hyperparams}. {All models are trained for 200 epochs, and we evaluate them on the validation set throughout training. For reporting results, we select the checkpoint that achieves the best validation performance. During optimization, we use Adam with the default settings provided by PyTorch backend~\cite{TorchNips}. We do not apply gradient clipping.}

\subsubsection{Evaluation Metrics}
To evaluate model performance, the metrics $\textrm{RMSE}$, $\textrm{R}^2$, $\textrm{MAE}$, and $\textrm{MAPE}$ are utilized as described in \Cref{eq:metrics}. 
% To evaluate the model performance, 
% the $\textrm{RMSE}$, $\textrm{R}^{\text{2}}$, $\textrm{MAE}$, and $\textrm{MAPE}$ in \Cref{RMSE,R2,MAPE,MAE} are adopted, where the $\mathrm{N}_{\text{test}}$ and $\bar{y}$ are the size of the testing dataset and average value of label $y$, respectively. 
% For $\textrm{RMSE}$, $\textrm{MAE}$, and $\textrm{MAPE}$, the closer value is to 0, the more accurate the prediction. 
% While for $\textrm{R}^{\text{2}}$, the closer $\textrm{R}^{\text{2}}$ is to 1, the better the prediction performance of model.

\begin{equation}\label{eq:metrics}
    \begin{aligned}
        \textrm{RMSE} &= \sqrt {\tfrac{\sum\limits_{i = 1}^{\mathrm{N}_{\text{test}}} {{{({y_i} - {{\hat{y}}_i})}^2}}}{\mathrm{N}_{\text{test}}} } ,\\
        \textrm{R}^{\text{2}} &= 1 - \tfrac{{\sum\limits_{l = 1}^{\mathrm{N}_{\text{test}}} {{{({y_l} - {{\hat{y}}_l})}^2}} }}{{\sum\limits_{l = 1}^{\mathrm{N}_{\text{test}}} {{{({y_l} - \bar y)}^2}} }},\\
        \textrm{MAPE} &= \frac{1}{\mathrm{N}_{\text{test}}}\sum\limits_{l = 1}^{\mathrm{N}_{\text{test}}} {{{\vert \frac {({y_l} - {{\hat{y}}_l})}{{y_l}}\vert}}} \times 100\%,\\
\textrm{MAE} &= \frac{1}{\mathrm{N}_{\text{test}}}\sum\limits_{l = 1}^{\mathrm{N}_{\text{test}}} {{{\vert{y_l} - {{\hat{y}}_l}\vert}}},
    \end{aligned}
\end{equation}
where $\mathrm{N}_{\text{test}}$ represents the size of the testing dataset, and $\bar{y}$ is the average value of the label $y$. For $\textrm{RMSE}$, $\textrm{MAE}$, and $\textrm{MAPE}$, values closer to 0 indicate more accurate predictions. Conversely, for $\textrm{R}^{\text{2}}$, values closer to 1 signify better predictive performance of the model.

\section{Additional Experimental Results}\label{sec:additionalExperimentalResults}
\subsection{Standard Deviation Results of Baseline Comparison}

In this subsection, we further report the standard deviations of the baseline comparison results. As shown in~\Cref{tab:standardDeviationResults}, among all competing models, our proposed KProxNPLVM achieves top-three (i.e., among the three lowest) standard deviations across the evaluated metrics. This observation indicates the robustness of the proposed approach and further suggests that the KProx algorithm is less sensitive to different initializations.

\begin{table}[!h]

    \centering
    \caption{Standard Deviation Error for Baseline Model Comparison Experiment}\label{tab:standardDeviationResults}
   \resizebox{\linewidth}{!}{
        \begin{threeparttable}
\begin{tabular}{l|llll|llll|llll}\toprule \multirow{2}{*}{Model} & \multicolumn{4}{c|}{DBC}& \multicolumn{4}{c|}{CAC}& \multicolumn{4}{c}{CSC}\\ \cmidrule{2-13} & \multicolumn{1}{c}{$\text{R}^\text{2}$} & \multicolumn{1}{c}{RMSE} & \multicolumn{1}{c}{MAE} & \multicolumn{1}{c|}{MAPE} & \multicolumn{1}{c}{$\text{R}^\text{2}$} & \multicolumn{1}{c}{RMSE} & \multicolumn{1}{c}{MAE} & \multicolumn{1}{c|}{MAPE} & \multicolumn{1}{c}{$\text{R}^\text{2}$} & \multicolumn{1}{c}{RMSE} & \multicolumn{1}{c}{MAE} & \multicolumn{1}{c}{MAPE} \\ \midrule SNPLVR  & 1.39E-1 & 1.34E-2 & 1.10E-2 & 4.92E1 & 2.20E-2 & 7.06E-5 & 9.47E-5 & 3.92E-2 & 3.24E-1 & 7.85E-2 & 6.54E-2 & 3.35E-2 \\  DBPSFA  & 9.13E-3 & \textbf{1.07E-3} & \textbf{7.68E-4} & \textbf{1.90E0} & \textbf{4.12E-4} & \textbf{1.42E-6} & \textbf{1.40E-6} & \textbf{5.16E-4} & 2.46E1 & 3.80E-2 & 3.80E-2 & 1.93E-2 \\  MUDVAE-SDVAE  & 3.21E-2 & 3.24E-3 & 2.27E-3 & 2.42E0 & 4.57E-3 & \uwave{1.62E-5} & \uwave{1.38E-5} & \uwave{5.00E-3} & 5.39E-2 & 1.48E-2 & 1.28E-2 & 6.55E-3 \\  GMM-VAE  & 2.02E-2 & 3.98E-3 & 3.05E-3 & 9.84E0 & 1.30E-2 & 5.61E-5 & 5.01E-5 & 1.66E-2 & 1.84E-2 & 1.16E-2 & 9.65E-3 & 4.90E-3 \\  GSTAE  & \uwave{3.42E-3} & 1.92E-3 & 1.67E-3 & 5.44E0 & \uwave{3.05E-3} & 2.13E-5 & 3.35E-5 & 1.14E-2 & 2.09E-2 & 1.90E-2 & 1.42E-2 & 7.25E-3 \\  VW-SAE  & 7.85E-2 & 8.95E-3 & 9.32E-3 & 7.43E0 & 8.36E-2 & 3.93E-4 & 2.75E-4 & 9.35E-2 & 8.71E-2 & 5.11E-2 & 3.60E-2 & 1.83E-2 \\  DGDL  & 1.04E-2 & 7.87E-3 & 5.56E-3 & 3.54E0 & 1.39E-2 & 9.46E-5 & 6.58E-5 & 2.30E-2 & 5.01E-3 & 5.61E-3 & 4.58E-3 & 2.34E-3 \\  iTransformer  & 1.18E-2 & 1.05E-2 & 6.65E-3 & 3.53E1 & 1.61E-2 & 1.04E-4 & 8.13E-5 & 2.82E-2 & \uwave{3.55E-3} & \uwave{4.32E-3} & \uwave{2.93E-3} & \uwave{1.49E-3} \\  KProxNPLVM  & \textbf{6.06E-4} & \uwave{1.32E-3} & \uwave{1.01E-3} & \uwave{2.02E0} & 4.94E-3 & 3.55E-5 & 3.26E-5 & 1.04E-2 & \textbf{2.25E-3} & \textbf{2.76E-3} & \textbf{2.05E-3} & \textbf{1.04E-3} \\ \midrule Win Counts & 8  & 7  & 7  & 7  & 5  & 5  & 6  & 6  & 8  & 8  & 8  & 8  \\ \bottomrule\end{tabular}
\begin{tablenotes}
    \item{\textbf{Bolded} and \uwave{Wavy} results indicate first and second best in each metric. %\uwave{Wavy} results indicate the second best in each metric.
    }
    \end{tablenotes}
    \end{threeparttable}
}  
\end{table}

\subsection{Space and Time Complexity Analysis}
% In this subsection, we conduct the empirical space and time complexity of our proposed KProxNPLVM. Specifically, when we conduct the training of KProxNPLVM, we have the following two major steps that requires the complexity analysis. 
In this subsection, we empirically analyze the time and space complexity of the proposed KProxNPLVM. Training KProxNPLVM mainly involves two components, whose complexities are summarized below.
For the time complexity, we have the following analysis result:

For the time complexity, our theoretical results can be summarized as follows:
\begin{itemize}
\item \textbf{KProx for $\mathcal{P}(z\vert x)$ inference:}
We parameterize $\log \mathcal{P}(z\vert x)$ using a multi-layer perceptron (MLP). Let $\mathrm{L}$ denote the network depth and let $\mathrm{D}_{\text{MLP}}$ be the (maximum) hidden width. With $\mathrm{N}$ particles and latent dimension $\mathrm{D}_{\rm LV}$, the per-iteration time complexity consists of (i) computing particle-wise scores via backpropagation through the MLP and (ii) evaluating pairwise RBF-kernel interactions:
\begin{equation}
\mathcal{O}(\mathrm{N}\mathrm{L}\mathrm{D}_{\text{MLP}}^2)+\mathcal{O}(\mathrm{N}^2 \mathrm{D}_{\rm LV}).
\end{equation}
Therefore, after $\mathrm{T}$ KProx iterations, the total time complexity is given as follows:
\begin{equation}
\mathcal{O}(\mathrm{T}\mathrm{N}\mathrm{L}\mathrm{D}_{\text{MLP}}^2)+\mathcal{O}(\mathrm{T}\mathrm{N}^2 \mathrm{D}_{\rm LV}).
\end{equation}

\item \textbf{Sinkhorn--Knopp algorithm for inference-network training:}
For a problem of size $\ell$, each Sinkhorn--Knopp iteration is dominated by matrix--vector multiplications and costs
\begin{equation}
\mathcal{O}(\ell^2).
\end{equation}
Thus, running $\mathrm{T}$ iterations yields a total time complexity as follows:
\begin{equation}
\mathcal{O}(\mathrm{T}\ell^2).
\end{equation}
\end{itemize}

% \begin{itemize}
%     \item{\textbf{KProx for $\mathcal{P}(z\vert x)$ Inference:} Since we are parameterizing the $\log\mathcal{P}(z\vert x)$ with a multi-layer perceptron, we denote the layer as $\mathrm{L}$, the maximum dimension as $\mathrm{D}_{\text{MLP}}$. Then the computation procedure can be given as follows for each iteration time:
%     \begin{equation}
%         \mathcal{O}(\mathrm{N}\mathrm{L}\mathrm{D}_{\text{MLP}}^2) +  \mathcal{O}(\mathrm{N}^2 \mathrm{D}_{\rm{LV}}).
%     \end{equation}
% Finally, the total computation complexity can be summarized as follows:
% \begin{equation}
%         \mathcal{O}(\mathrm{T}\mathrm{N}\mathrm{L}\mathrm{D}_{\text{MLP}}^2) +  \mathcal{O}(\mathrm{T}\mathrm{N}^2 \mathrm{D}_{\rm{LV}}).
% \end{equation}
%     }
%     \item{\textbf{Sinkhorn-Knopp Algorithm for Inference Network Training:} For Sinkhorn-Knopp algorithm, for each iteration time, we have the following computational complexity:
%     \begin{equation}
%         \mathcal{O}(\ell^2).
%     \end{equation}
% On this basis, for $\mathrm{T}$ iter time, we have:
% \begin{equation}
%       \mathcal{O}(\mathrm{T}\ell^2).
% \end{equation}
%     }
% \end{itemize}

For the space complexity, our theoretical results can be summarized as follows:

\begin{itemize}
\item \textbf{KProx for $\mathcal{P}(z\vert x)$ inference:}
We only need to store the set of latent particles $\{z_{i,t}\}_{i=1}^{\mathrm{N}}$ at time $t$. Hence, the space complexity is given as follows:
\begin{equation}
\mathcal{O}(\mathrm{N}\mathrm{D}_{\rm LV}).
\end{equation}

\item \textbf{Sinkhorn--Knopp algorithm for inference-network training:}
We store the kernel/cost matrix $\mathscr{K}\in\mathbb{R}^{\ell\times \ell}$, resulting in the following result:
\begin{equation}
\mathcal{O}(\ell^2).
\end{equation}
\end{itemize}

On this basis,~\Cref{tab:empiricalSpatialTemporalComplexity} reports the training/testing time and GPU memory usage. Overall, KProxNPLVM consistently achieves the best or second-best runtime while keeping memory comparable to baselines, indicating favorable practical efficiency. From~\Cref{tab:empiricalSpatialTemporalComplexity}, we observe that the KProxNPLVM attains the fastest training on DBC (22.06s) and WGS (41.71s), and the second-fastest on CAC (78.60s). In particular, compared with the strongest runtime baseline iTransformer, KProxNPLVM reduces training time by 14.1\% on DBC (25.69 $\rightarrow$ 22.06) and 40.7\% on WGS (70.34 $\rightarrow$ 41.71). KProxNPLVM achieves the fastest inference on WGS (0.010s) and the second-fastest on DBC (0.004s) and CAC (0.015s). On WGS, it yields a 69.7\% reduction in test time over iTransformer (0.033 $\rightarrow$ 0.010). KProxNPLVM uses 549 MB across all datasets, which is on the same order as most baselines (typically 532–537 MB) and much lower than DBPSFA on CAC (579 MB) and especially CSC (579 MB vs 549 MB).  

\begin{table}[!h]

    \centering
    \caption{Empirical Spatial Temporal Complexity}
    \label{tab:empiricalSpatialTemporalComplexity}
 \resizebox{\linewidth}{!}{    \begin{threeparttable} \begin{tabular}{l|lll|lll|lll} \toprule \multicolumn{1}{c|}{\multirow{2}{*}{Model}} & \multicolumn{3}{c|}{DBC}                       & \multicolumn{3}{c|}{CAC}                       & \multicolumn{3}{c}{CSC}   \\ \cmidrule{2-10} \multicolumn{1}{c|}{}                       & Training (s) & Testing (s) & \multicolumn{1}{l|}{Memory (MB)} & Training (s) & Testing (s) & \multicolumn{1}{l|}{Memory (MB)} & Training (s) & Testing (s) & Memory (MB) \\ \midrule SNPLVR & 60.835$_{\pm \text{3.420E0}}$ & 0.016$_{\pm \text{5.355E-4}}$ & 535 & 144.920$_{\pm \text{7.871E0}}$ & 0.036$_{\pm \text{4.506E-4}}$ & 535 & 159.527$_{\pm \text{6.549E0}}$ & 0.041$_{\pm \text{9.631E-4}}$ & 535 \\DBPSFA & 86.948$_{\pm \text{6.550E-1}}$ & \textbf{0.003}$_{\pm \text{8.059E-4}}$ & 579 & 182.223$_{\pm \text{6.572E1}}$ & \textbf{0.003}$_{\pm \text{2.037E-3}}$ & 579 & 347.915$_{\pm \text{8.032E-1}}$ & 0.045$_{\pm \text{5.891E-2}}$ & 579 \\MUDVAE-SDVAE & 75.004$_{\pm \text{4.467E-1}}$ & 0.016$_{\pm \text{5.822E-4}}$ & 535 & 160.677$_{\pm \text{9.567E0}}$ & 0.038$_{\pm \text{8.933E-4}}$ & 535 & 181.304$_{\pm \text{6.022E0}}$ & 0.042$_{\pm \text{5.934E-4}}$ & 535 \\GMM-VAE & 49.870$_{\pm \text{3.536E0}}$ & 0.026$_{\pm \text{6.005E-4}}$ & 534 & 126.370$_{\pm \text{3.685E0}}$ & 0.062$_{\pm \text{8.111E-4}}$ & 534 & 143.606$_{\pm \text{6.681E0}}$ & 0.073$_{\pm \text{1.717E-3}}$ & 534 \\GSTAE & 63.405$_{\pm \text{1.247E0}}$ & 0.012$_{\pm \text{1.631E-4}}$ & 535 & 151.798$_{\pm \text{6.928E0}}$ & 0.030$_{\pm \text{2.865E-4}}$ & 535 & 163.265$_{\pm \text{2.299E0}}$ & 0.035$_{\pm \text{2.552E-4}}$ & 535 \\VW-SAE & 66.652$_{\pm \text{8.477E0}}$ & 0.007$_{\pm \text{5.407E-5}}$ & 535 & 136.883$_{\pm \text{1.816E1}}$ & 0.019$_{\pm \text{1.414E-3}}$ & 535 & 147.260$_{\pm \text{1.825E1}}$ & \uwave{0.022}$_{\pm \text{4.003E-4}}$ & 535 \\DGDL & 60.005$_{\pm \text{2.004E0}}$ & 0.028$_{\pm \text{1.744E-3}}$ & 532 & 139.450$_{\pm \text{6.691E0}}$ & 0.068$_{\pm \text{9.109E-4}}$ & 532 & 160.682$_{\pm \text{1.122E1}}$ & 0.079$_{\pm \text{4.515E-4}}$ & 532 \\iTransformer & \uwave{25.687}$_{\pm \text{9.319E-1}}$ & 0.011$_{\pm \text{6.706E-5}}$ & 537 & \textbf{61.714}$_{\pm \text{4.930E0}}$ & 0.028$_{\pm \text{2.495E-4}}$ & 537 & \uwave{70.341}$_{\pm \text{4.714E0}}$ & 0.033$_{\pm \text{1.215E-3}}$ & 537 \\KProxNPLVM & \textbf{22.060}$_{\pm \text{8.838E-1}}$ & \uwave{0.004}$_{\pm \text{1.917E-4}}$ & 549 & \uwave{78.601}$_{\pm \text{6.664E1}}$ & \uwave{0.015}$_{\pm \text{6.473E-3}}$ & 549 & \textbf{41.707}$_{\pm \text{1.200E1}}$ & \textbf{0.010}$_{\pm \text{1.826E-4}}$ & 549 \\ \bottomrule \end{tabular} 
 \begin{tablenotes}
    \item{\textbf{Bolded} and \uwave{Wavy} results indicate first and second best in each metric. %\uwave{Wavy} results indicate the second best in each metric.
    }
    \end{tablenotes}
    \end{threeparttable}
    }

\end{table}

\subsection{Additional Convergence Analysis of Ablation Study}
% In this subsection, we further investigate the influence of the ablation study results on the convergence behavior. From~\Cref{fig:convGenResults}, we compare the influence of the ablation module with respect to the influence of generative network convergence behavior. We found that without introducing the KProx, the log-likelihood is better and will not decrease to zero, thereby resulting a worse performance. In addition, when we ablate the training of inference network with Wasserstein distance, we observe that the generative network ablation is greater, which demonstrate the necessicity of introducing the Wasserstein distance. 
In this subsection, we further examine how the ablation settings affect convergence. As shown in \Cref{fig:convGenResults}, we compare the convergence behavior of the generative network across different ablations. Without KProx, the log-likelihood fails to drop to zero and remains relatively high, which leads to inferior performance. Moreover, when we remove the Wasserstein-based training of the inference network, the generative network exhibits even more severe degradation, underscoring the necessity of introducing the Wasserstein distance.
\begin{figure}[!h]
   % \vspace{-0.3cm}
    \centering
    % figures\cut_sampled_results.pdf
     \subfigure[DBC Dataset.]{\includegraphics[width=0.315\linewidth]{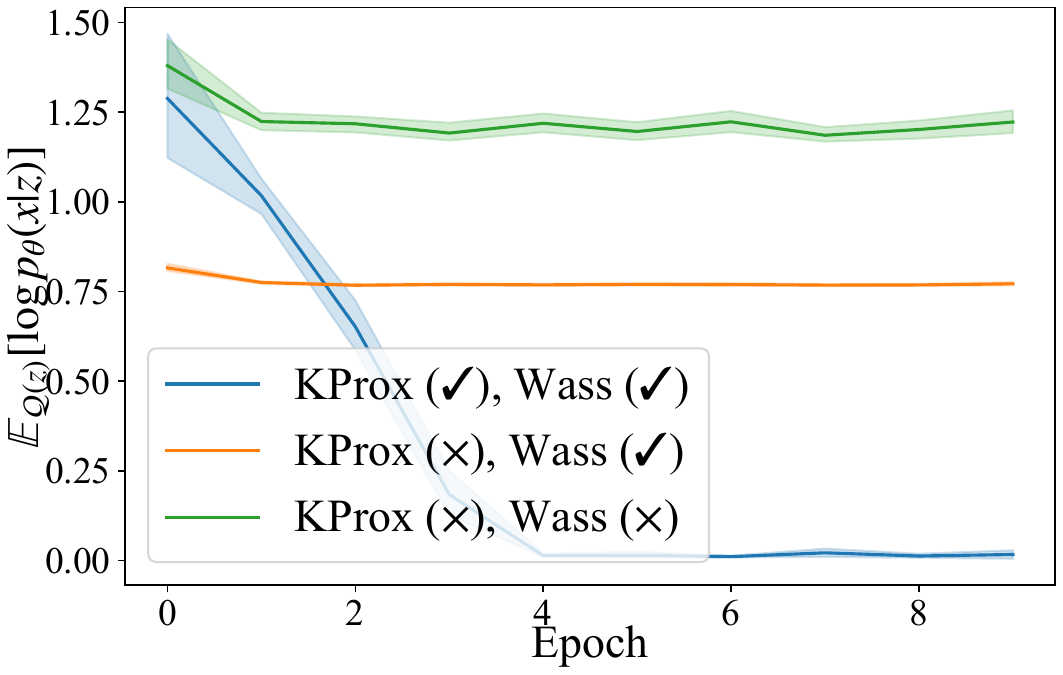}\label{subfig:ab_gen_dc}}
    \subfigure[CAC Dataset.]{\includegraphics[width=0.315\linewidth]{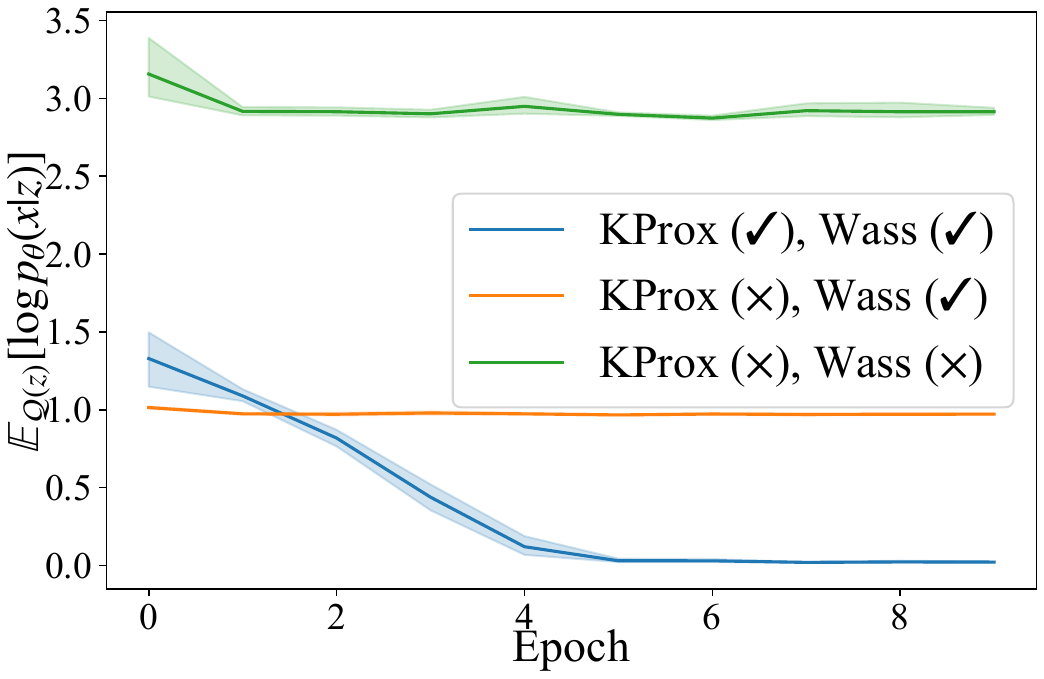}\label{subfig:ab_gen_ca}}
       \subfigure[CSC Dataset.]{\includegraphics[width=0.315\linewidth]{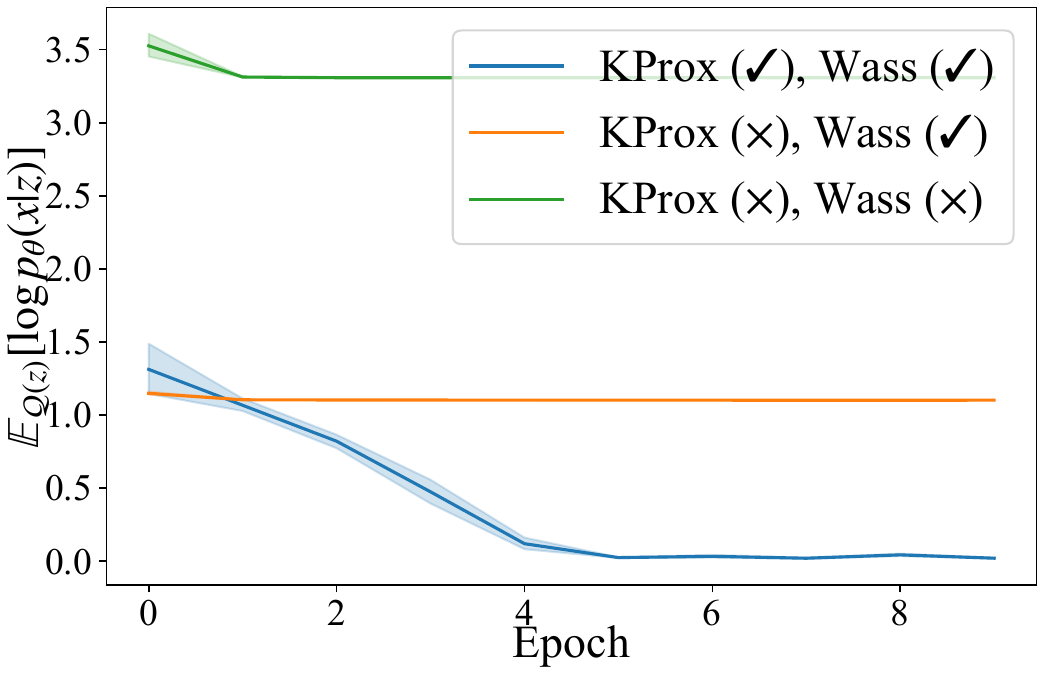}\label{subfig:ab_gen_wgs}}
    % \subfigure[$\mathcal{W}_2(\mathcal{Q}_t(z), \mathcal{P}(z|\mathcal{D}))$ along $t$, where $\mathcal{Q}_{0}(z)= \mathcal{U}(-0.5,0.5)$.]{\includegraphics[width=0.49\linewidth]{pictures/DensityEvolution/wass_evolution__gmm_uni(prox).pdf}\label{subfig:uniWassDistanceDifferences}}
   % \vspace{-0.3cm}
    \caption{Convergence analysis for generative network, .% Comparison results between $\mathcal{P}(z|x)/p_{\theta,\text{NF}}(z|x)$ with $\mathcal{Q}(z)$ along iterative time $\tau$. %  $q(\vec{a})$.
    %. (b) Impact of entropic regularization strength ε. (c) Impact of PFOR strength γ (×103). (d) Impact of RMPR strength κ.
    }\label{fig:convGenResults}
  % \vspace{-0.6cm}
\end{figure}

Building on this, we further compare the convergence behavior of the inference network under different ablation settings. We observe that when the Wasserstein distance is removed from the training objective, the inference network still converges, but it converges to a higher loss than the variant trained with the Wasserstein term. This result highlights the importance of incorporating the Wasserstein distance to promote more stable convergence.

\begin{figure}[!h]
  %  \vspace{-0.3cm}
    \centering
    % figures\cut_sampled_results.pdf
     \subfigure[DBC Dataset.]{\includegraphics[width=0.315\linewidth]{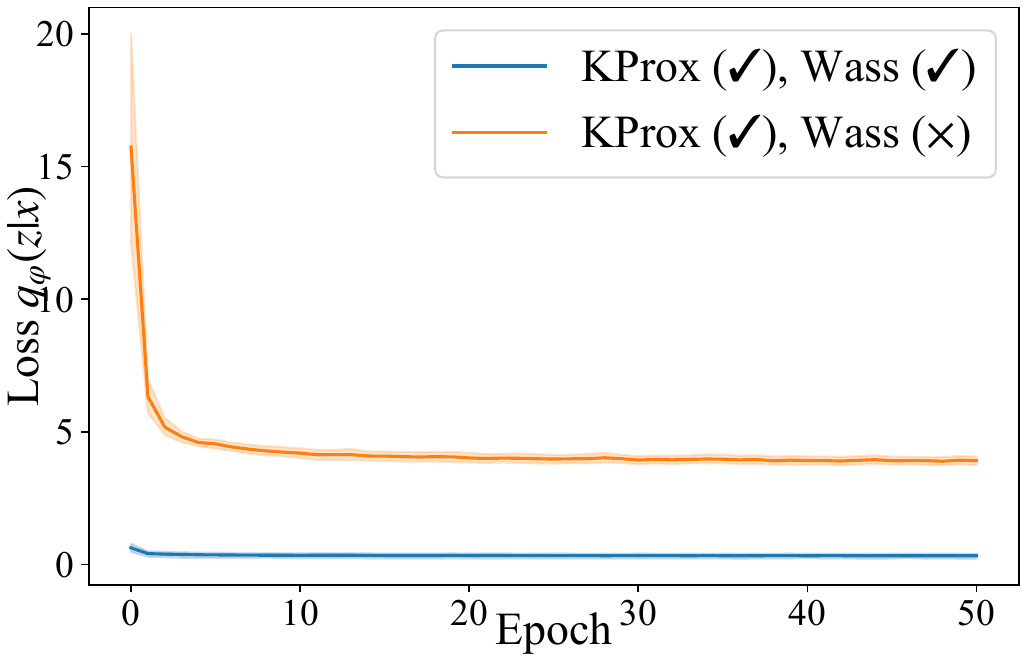}\label{subfig:ab_inf_dc}}
    \subfigure[CAC Dataset.]{\includegraphics[width=0.315\linewidth]{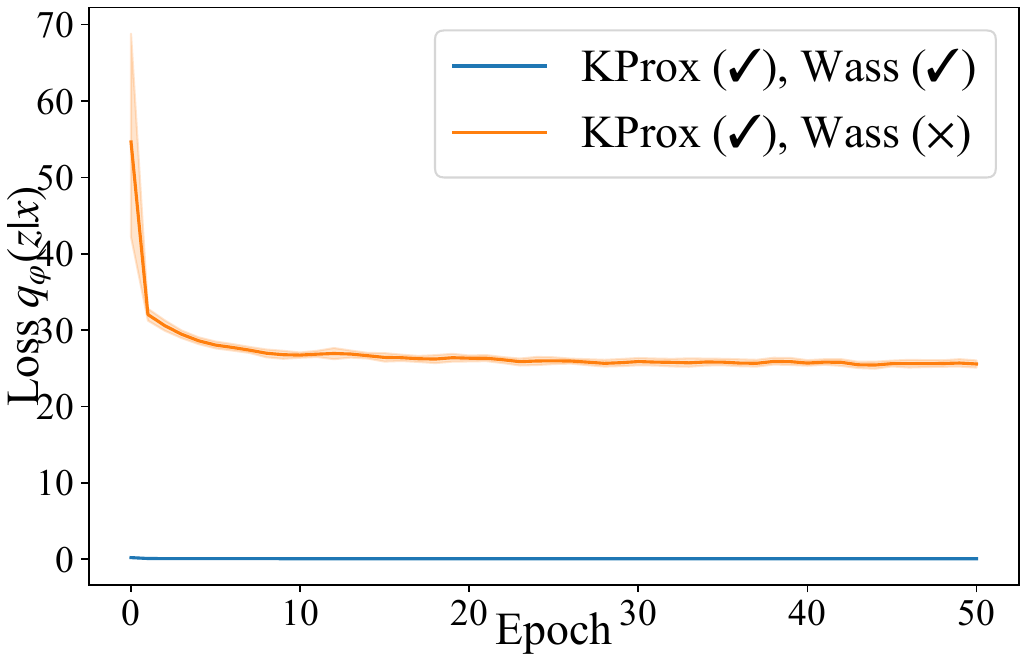}\label{subfig:ab_inf_ca}}
       \subfigure[CSC Dataset.]{\includegraphics[width=0.315\linewidth]{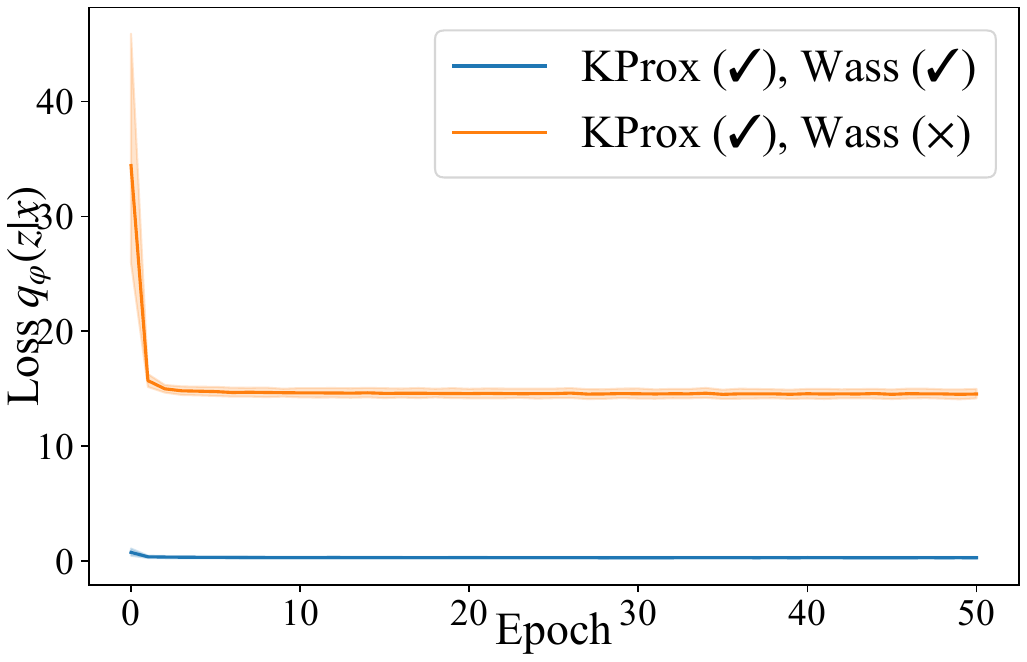}\label{subfig:ab_inf_wgs}}
    % \subfigure[$\mathcal{W}_2(\mathcal{Q}_t(z), \mathcal{P}(z|\mathcal{D}))$ along $t$, where $\mathcal{Q}_{0}(z)= \mathcal{U}(-0.5,0.5)$.]{\includegraphics[width=0.49\linewidth]{pictures/DensityEvolution/wass_evolution__gmm_uni(prox).pdf}\label{subfig:uniWassDistanceDifferences}}
   % \vspace{-0.3cm}
    \caption{Convergence analysis for inference network, .% Comparison results between $\mathcal{P}(z|x)/p_{\theta,\text{NF}}(z|x)$ with $\mathcal{Q}(z)$ along iterative time $\tau$. %  $q(\vec{a})$.
    %. (b) Impact of entropic regularization strength ε. (c) Impact of PFOR strength γ (×103). (d) Impact of RMPR strength κ.
    }\label{fig:convInfResults}
  % \vspace{-0.6cm}
\end{figure}

\clearpage
\section*{Acknowledgments}
The last author, Zhichao Chen, would like to express his sincere gratitude to Mr. Fangyikang Wang at Zhejiang University for valuable discussions regarding the gradient flow technique. He also wishes to thank Associate Professor Chang Liu at Zhongguancun Academy for pioneering works and enlightenment on particle-based variational inference. Furthermore, the authors would like to thank the anonymous reviewers for their valuable comments and suggestions, which have greatly improved the quality of this manuscript.

%\vspace{-0.2cm}
\footnotesize
\bibliographystyle{IEEEtran}
\bibliography{references}

\end{document}